\documentclass[nohyperref]{article}
\usepackage{microtype,natbib}
\usepackage{graphicx}
\usepackage{subcaption}
\usepackage{booktabs} % for professional tables
\usepackage{hyperref}
\usepackage{url}
\usepackage{xcolor}
\usepackage[accepted]{icml2022}
\usepackage{color}
\usepackage{soul}
\usepackage{amsmath,amssymb,amsthm, mathtools}
\usepackage{algorithm}
\usepackage{mathrsfs}
\usepackage{algorithm, algorithmic}
\usepackage{bm}

\renewcommand{\cal}{\mathcal}

\newcommand{\cC}{{\cal C}}
\newcommand{\cF}{{\cal F}}
\newcommand{\cY}{{\cal Y}}
\newcommand{\cD}{{\cal D}}

\newcommand{\cM}{{\cal M}}
\newcommand{\cN}{{\cal N}}

\newcommand{\cP}{{\cal P}}

\newcommand\cW{{\mathcal W}}
\newcommand{\cX}{{\mathcal X}}
\newcommand{\R}{\mathbb{R}}	%
\newcommand{\bE}{\mathbb{E}}

\newcommand{\bP}{\mathbb{P}}
\newcommand{\Prob}{\mathbb{P}}

\newcommand{\bR}{{\mathbb R}}

\newtheorem{theorem}{Theorem}[section]

\newtheorem{lemma}{Lemma}[section]

\newtheorem{assumption}{Assumption}[section]

\newtheorem{definition}{Definition}[section]

\usepackage{algorithm, algorithmic}

\DeclareMathOperator*{\E}{\mathbb{E}}

\DeclareMathOperator{\sgn}{sgn}

\DeclareMathOperator{\argmax}{argmax}
\DeclareMathOperator{\argmin}{argmin}

\def\tx{\tilde{x}}
\def\ty{\tilde{y}}
\def\tz{\tilde{z}}
\def\Lnmix{L^{\text{mix}}_n}
\def\Ln{L_n^{std}}
\def\Lnmixapp{\tilde{L}^{\text{mix}}_n}
\def\tD{\tilde{\cD}}

\def \ECE{\text{ECE}}
\usepackage{hyperref}

\usepackage{float}
\title{When and How Mixup Improves Calibration }

\begin{document}

%\twocolumn[

%%\aistatstitle{When and How Mixup Improves Calibration: A Theoretical View}

%\aistatsauthor{ Author 1 \And Author 2 \And  Author 3 }

%\aistatsaddress{ Institution 1 \And  Institution 2 \And Institution 3 } ]






\twocolumn[
\icmltitle{When and How Mixup Improves Calibration}



\icmlsetsymbol{equal}{*}

\begin{icmlauthorlist}
\icmlauthor{Linjun Zhang}{equal,a}
\icmlauthor{Zhun Deng}{equal,b}
\icmlauthor{Kenji Kawaguchi}{c}
\icmlauthor{James Zou}{d}

\end{icmlauthorlist}

\icmlaffiliation{a}{Rutgers University}
\icmlaffiliation{b}{Harvard University}
\icmlaffiliation{c}{National University of Singapore}
\icmlaffiliation{d}{Stanford University}
\icmlcorrespondingauthor{Linjun Zhang}{linjun.zhang@rutgers.edu}
\icmlcorrespondingauthor{Zhun Deng}{zhundeng@g.harvard.edu}

\icmlkeywords{Machine Learning, ICML}

\vskip 0.3in
]



\printAffiliationsAndNotice{\icmlEqualContribution} % otherwise use the standard text.

%======================================
\begin{abstract}
In many machine learning applications, it is important for the model to provide confidence scores that accurately capture its prediction uncertainty. Although modern learning methods have achieved great success in predictive accuracy, generating calibrated confidence scores remains a major challenge. Mixup, a popular yet simple data augmentation technique based on taking convex combinations of pairs of training examples, has been empirically found to significantly improve confidence calibration across diverse applications. However, when and how Mixup helps calibration is still a mystery. In this paper, we theoretically prove that Mixup improves calibration in  \textit{high-dimensional} settings by investigating natural statistical models. Interestingly, the calibration benefit of Mixup increases as the model capacity increases.  We support our theories with experiments on common architectures and datasets. In addition, we study how Mixup improves calibration in semi-supervised learning. While incorporating unlabeled data can sometimes make the model less calibrated, adding Mixup training mitigates this issue and provably improves calibration. Our analysis provides new insights and a framework to understand Mixup and calibration.
\end{abstract}

%======================================
\section{Introduction}
Modern machine learning methods have dramatically improved the predictive accuracy in many learning tasks \citep{simonyan2014very,srivastava2015highway,he2016deep}. The deployment of AI-based systems in high risk fields such as medical diagnosis \citep{jiang2012calibrating} requires a predictive model to be trustworthy, which makes the topic of accurately quantifying the predictive uncertainty an increasingly important problem \citep{thulasidasan2019mixup}. However, as pointed out by \citet{guo2017calibration}, many popular modern architectures such as neural networks are very poorly calibrated. 
 A variety of methods have been proposed for quantifying predictive uncertainty including training multiple probabilistic models with ensembling or bootstrap \citep{osband2016deep} and re-calibration of probabilities on a validation set through temperature scaling \citep{platt1999probabilistic}, which usually involves much more complicated procedures and extra computation. Meanwhile, recent work \citep{thulasidasan2019mixup} has shown that  models trained with Mixup \citep{zhang2017mixup}, a simple data augmentation technique based on taking convex combinations of pairs of examples and their labels, are significantly better calibrated. However, when and how Mixup helps calibration is still not well-understood, especially from a theoretical perspective. 

As our \textbf{first contribution}, we demonstrate that the calibration improvement brought by Mixup is more significant in the \textit{high-dimensional} settings, i.e. the number of parameters is comparable to the training sample size. Figure~\ref{fig:1} shows a motivating experiment on CIFAR-10.
%As a motivating example, \zhun{we implement the fully-connected neural networks with various values of the width (\emph{i.e.}, the number of neurons per hidden layer) and the depth (\emph{i.e.}, the number of hidden layers) on CIFAR-10. More results are shown in Section \ref{sec:supervised learning}}.
The Expected Calibration Error (ECE), %\footnote{\zhun{Previous works such as \citet{nixon2019measuring,kumar2019verified} suggest that many methods used to estimate ECE in practice is unavoidably biased. However, we directly compute ECE in theory, which directly measures the calibrated confidence. }} and Maximum Calibration Error (MCE),
which is a standard measure of how un-calibrated a model is, is smaller with Mixup augmentation compared to those without Mixup augmentation, especially when the model is wider or deeper. 
%when the width or depth of the neural networks are small, Mixup does not help and may even hurt calibration. However, as the width and depth of the neural networks grow, the calibration after applying Mixup is getting increasingly better than that without Mixup. 
We provide a theoretical explanation for this phenomenon under several natural statistical models. In particular, our theory holds when the data distribution can be described by a Gaussian generative model, which is very flexible and includes many generative adversarial networks (GANs). In a Gaussian generative model, a function is used to map a Gaussian random variable to an input vector of some models   such as neural networks. Because the function used to map a Gaussian random variable is arbitrary and can be nonlinear, our theory is applicable to a very broad class of data distributions.

%We further dive into the theory and confirm such phenomenon by studying two natural data models on classification and regression. 
%==================
\begin{figure}[!t]
	\centering
	\begin{subfigure}[b]{0.45\columnwidth}
		\includegraphics[width=\textwidth, height=0.7\textwidth]{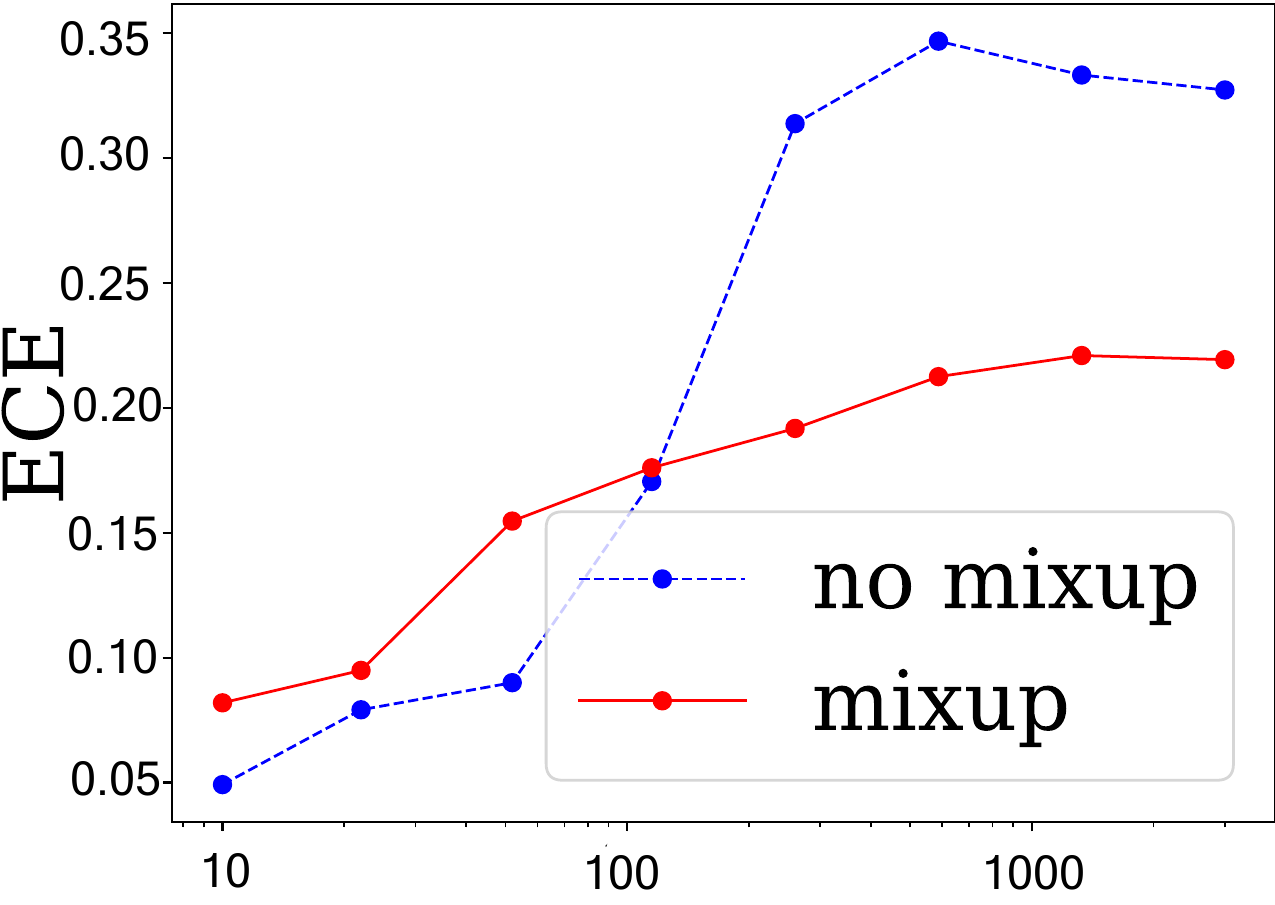}
		\caption{Network width}
	\end{subfigure}
\hskip 5mm
	\begin{subfigure}[b]{0.45\columnwidth}
		\includegraphics[width=\textwidth, height=0.7\textwidth]{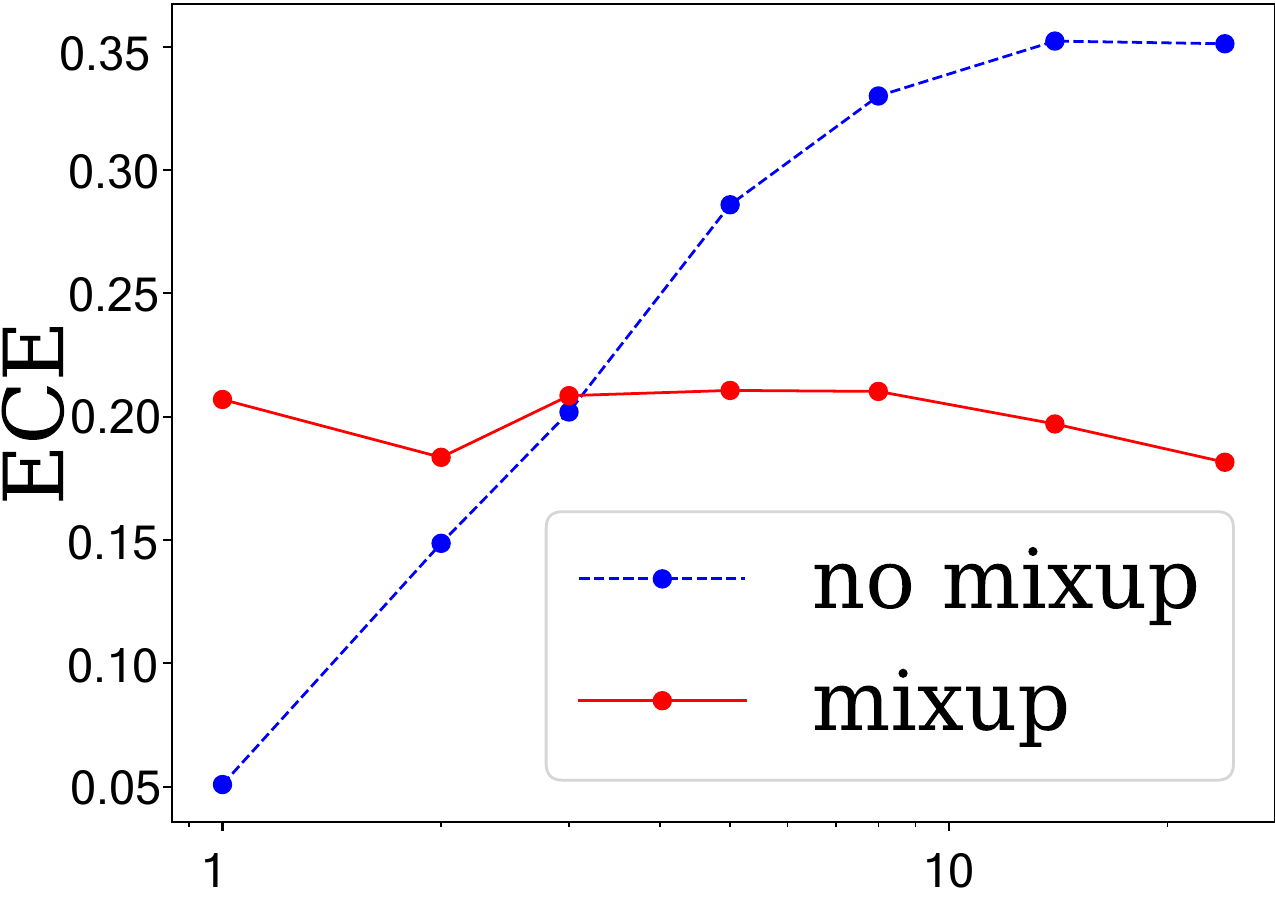}
		\caption{Network depth}
	\end{subfigure}
	\caption{Expected calibration error (ECE) calculated for a fully-connected neural network on CIFAR-10. In (a), we fix the depth and increase the width of the neural network; while in (b), we fix the width and increase the depth of the neural network. Mixup augmentation can reduce ECE especially for larger capacity models.} %\james{Good to make the font on the x, y axis much larger. Also for the later figures.}
	\label{fig:1} 
\end{figure}

As our \textbf{second contribution}, we investigate how Mixup helps calibration in semi-supervised learning, which is relatively under-explored. Labeled data are usually expensive to obtain, and training models by combining a small amount of labeled data with abundant unlabeled data plays an important role in AI \citep{chapelle2009semi}. In light of this, we investigate the effect of Mixup in semi-supervised learning, where we focus on the commonly used pseudo-labeling algorithm \citep{chapelle2009semi,carmon2019unlabeled}. We observe experimentally that the pseudo-labeling by itself can sometimes  hurt calibration. However, combining Mixup with pseudo-labeling consistently improves calibration. We provide theories to explain these findings. 

As our \textbf{third contribution}, we further extend our results to Maximum Calibration Error (MCE), which also demonstrates similar phenomena as those for ECE.

\textbf{Outline of the paper.}  Section \ref{sec:1} discusses related works and introduces the notations. In Section \ref{sec:supervised learning}, we present our main theoretical results for ECE by showing that Mixup improves calibration for classification problems in the high-dimensional regime. Section \ref{sec:2} investigates the semi-supervised learning setting and demonstrates the benefit of further applying Mixup to the pseudo-labeling algorithm. In Section \ref{sec:MCE}, we extend our studies of calibration to MCE. Section \ref{sec:3} concludes with a discussion of future work. Proofs are deferred to the Appendix.

%\paragraph{Related Works}

%should produce a predictive score that is close to actual likelihood of correctness and therefore enjoys the model interpretability. 

%\newpage
%======================================

\subsection{Related Work} \label{sec:1.1}

Mixup is a popular data augmentation scheme that has been shown to improve a model's prediction accuracy \citep{zhang2017mixup,thulasidasan2019mixup,guo2019mixup}. Recent theoretical analysis shows that Mixup has an implicit regularization effect that enables models to better generalize  \citep{zhang2020does}. The focus of our work is not on accuracy, but on calibration. 

Modern learning models such as neural networks have achieved remarkable performance nowadays in optimization \citep{deng2020representation,ji2021unconstrained,deng2021adversarial,ji2021power,kawaguchi2022understanding}. Even though the generalization and prediction \citep{deng2020towards,zhang2020does,deng2021shrinking} of neural networks are quite amazing, it has shown that neural networks tend to be over-confident. A well-calibrated predictive model is needed in many applications of machine learning, ranging from economics \citep{foster1997calibrated}, personalized medicine \citep{jiang2012calibrating}, to weather forecasting \citep{gneiting2005weather}, to fraud detection \citep{bahnsen2014improving}. The problem on producing a well-calibrated model has received increasing attention in recent years \citep{naeini2015obtaining,lakshminarayanan2016simple,guo2017calibration,zhao2020individual,foster2004variable,kuleshov2018accurate,wen2020combining,huang2020tutorial}. In real-world settings, the input distributions are sometimes shifted from the training distribution due to non-stationarity. The predictive uncertainty under such out-of-distribution condition was studied by \citet{ovadia2019can} and  \citet{chan2020unlabelled}.
Mixup has been empirically shown to improve the calibration for deep neural networks in both the same and out-of-distribution domains \citep{thulasidasan2019mixup, tomani2020towards}. Ours is the first work to provide theoretical explanation for this phenomenon. 

Semi-supervised learning is a broad field in machine learning concerned with learning from both labeled and unlabeled datasets \citep{chapelle2009semi}. Prior work mostly focuses on improving the prediction accuracy and robustness with unlabeled data \citep{zhu2003semi,zhu2009introduction,berthelot2019mixmatch,deng2020interpreting,deng2021improving,deng2020towards} and adversarial robustness \citep{carmon2019unlabeled,deng2020improving}. Recently, \citet{chan2020unlabelled} found that unlabeled data improves Bayesian uncertainty calibration in some experiments, but the relationship between using unlabeled data and calibration, especially from the theoretical perspective, is still largely unknown. All of the facts above motivate our theoretical exploration in this paper.

%===================================================
\section{Preliminaries} \label{sec:1}
In this section, We introduce the notations and briefly recap the mathematical formulation of Mixup and calibration measures considered in this paper.
%============================
\subsection{Notations} \label{sec:1.2}
We denote the training data set by  $S=\{(x_1,y_1),\cdots,(x_n,y_n)\},$ where $x_i\in\cX\subseteq\bR^d$ and $y_i\in\cY\subseteq\bR^m$ are drawn i.i.d. from a joint distribution $\cP_{x,y}$.  The general parameterized loss is denoted by $l(\theta,z)$, where $\theta\in\Theta\subseteq \bR^p$ and $z_i=(x_i,y_i)$ denotes the input and output pair. Let $L(\theta)=\bE_{z\sim\cP_{x,y}}l(\theta,z)$ denote the standard population loss and $\Ln(\theta,S)=\sum_{i=1}^nl(\theta,z_i)/n$ denote the standard empirical loss. In addition, we define $\tx_{i,j}(\lambda)=\lambda x_i+(1-\lambda)x_j$, $\ty_{i,j}(\lambda)=\lambda y_i+(1-\lambda)y_j$, and $\tz_{i,j}(\lambda)=(\tx_{i,j}(\lambda),\ty_{i,j}(\lambda))$ for $\lambda\in[0,1]$. We  use $t \cD_1+(1-t)\cD_2$ for $t\in(0,1)$ to denote the mixture distribution  such that a sample coming from that distribution is drawn with probabilities $t$ and $(1-t)$ from $\cD_1$ and $\cD_2$ respectively. In classification, the output $y_i$ is the embedding of the class of $x_i$; i.e., $y_i \in\{0, 1\}^m$ is the one-hot encoding of the class (with all entries equal to zero except for the one corresponding to the class of $x_i$), where $m$ is the total number of classes.

\subsection{Mixup}
Mixup is a data augmentation technique, which linearly interpolates the training sample pairs within the training data set to create a new data set $S^{mix}(\lambda)=\{(\tilde{z}_{i,j}(\lambda))\}_{i,j=1}^n$, with $\lambda$ following a distribution $\cD_\lambda$ supported on $[0,1]$. Throughout the paper, we consider the most commonly used $\cD_\lambda$ --- the Beta distribution $Beta(\alpha,\beta)$ for $\alpha,\beta>0$.

Typically, in a machine learning task, one wants to learn a function $f:\mathcal X\to\mathcal Y$ from a function class $\cF$ using the training data set $S\in(\cX\times\cY)^n$. Such a function is usually parametrized as $f_\theta$ with some parameter $\theta$. Let us denote the learned parameter by $\hat\theta=\mathcal M(S)$. In this paper, we consider learning the parameter by the Mixup training  $\mathcal M(S^{mix}(\lambda))$. Due to the randomness in $\lambda$, we consider taking the expectation over $\lambda$. % and the most commonly used $\cD_\lambda$ --- Beta distribution $Beta(\alpha,\beta)$ for $\alpha,\beta>0$. 
For example, a mapping could either be an estimator, such as the empirical mean of input: $\cM(S)=\sum_{i=1}^nx_i/n$,  or be the minimizer of a loss function: $\cM(S,\theta)=\argmin_{\theta}\sum_{i=1}^nl(\theta,z_i)/n$. The corresponding transformed mappings obtained via Mixup are then $\bE_{\lambda\sim\cD_\lambda}\cM(S^{mix}(\lambda))=\sum_{i,j=1}^n\bE_{\lambda\sim\cD_\lambda}x_{i,j}(\lambda)/n^2$ and $\bE_{\lambda\sim\cD_\lambda}\cM(S^{mix}(\lambda),\theta)=\sum_{i,j=1}^n \bE_{\lambda\sim \cD_\lambda}l(\theta,\tz_{i,j}(\lambda))/n^2$ respectively. 

%For example, in \citet{xx}, the mapping is a parameterized loss function, which we call Mixup loss, is defined in the following form:
%\begin{equation}\label{loss:Mixup}
%\Lnmix(\theta,S)=\frac{1}{n^2}\sum_{i,j=1}^n \bE_{\lambda\sim \cD_\lambda}l(\theta,\tz_{i,j}(\lambda)).
%\end{equation}

\subsection{Calibration for classification}
For a classification problem, if there are $K$ classes, typically, for an input $x$, a probability vector $\hat h(x)=(p_1(x),\cdots,p_K(x))^\top\in \bR^K$ is obtained from the trained model, where $p_i$ is the corresponding probability (or so-called confidence score) that $x$ belongs to the class $i$, and $\sum_{i=1}^K p_i=1$. Then, the output is $\hat{y}=\argmax_i p_i(x)$. The hope is that, for instance, given $1000$ samples, each with confidence $0.7$, around $700$ examples should be classified correctly. In other words, we expect for all $v\in[0,1]$, $\bP(\hat{y}=y|\hat{p}=v)\approx v,$
where $\hat{p}$ is the largest entry in $\hat h(x)$ and $y$ is the true class $x$ belongs to, which is termed as prediction confidence. 

\paragraph{Expected Calibration Error (ECE).} The most prevalent calibration metric is the Expected Calibration Error \citep{naeini2015obtaining}, which is defined as,
\begin{equation}\label{def:ECE}
ECE = \bE_{v\sim\cD_{\hat{p}}} \left[|\bP(\hat{y}=y|\hat{p}=v)- v|\right],   
\end{equation}
where %$\|\cdot\|_q^q$ is the $q$-th power of the $\ell_q$ norm and 
$\cD_{\hat{p}}$ is the distribution of $\hat{p}$.
While ECE is widely used, we note  that recents works \citep{nixon2019measuring,kumar2019verified} found that some methods of estimating ECE in practice (such as the binning method) is sometimes undesirable and can produce biased estimator under some specially constructed data distributions. Throughout this paper, in our theories, \textit{we mainly focus on the population version of calibration error as defined in \eqref{def:ECE}, which does not suffer from any such bias.}%Throughout the paper, all our theories would be built based on the ECE.}
%\zhun{We reiterate here that..}

\paragraph{Maximum Calibration Error (MCE).} Another widely used calibration metric is the Maximum Calibration Error \citep{naeini2015obtaining}, which is defined as 
$$MCE =\max_{v\in[0,1]} |\bP(\hat{y}=y|\hat{p}=v)- v|.$$

{Again, in our theory, we will only consider this population version of MCE. A predictor $\hat p$ with ECE/MCE equal to $0$ is said to be perfectly calibrated.}

%======================================
\section{Calibration in Supervised Learning}\label{sec:supervised learning}
Although Mixup has been shown to improve the test accuracy \citep{zhang2017mixup,guo2019mixup,zhang2020does}, there has been much less understanding of how it affects model calibration \footnote{Models with better test accuracy are not necessarily better calibrated.}. In this section, we focus on investigating when and how Mixup improves calibration. 
\subsection{Problem set-up}
As a confirmation of the phenomenon suggested in Figure \ref{fig:1} in the introduction, our theoretical results demonstrate that Mixup indeed improves calibration, and the improvement is especially significant in the \textbf{\textit{high-dimensional regime}}. Here, by high-dimensional regime, we mean when the number of parameters in the model, $p$, is comparable to the sample size $n$, 
i.e. $p/n>c$ for some constant $c>0$. In other words, the improvement in calibration by using Mixup is more significant in the over-parameterized case or when  the ratio between  $p$ and  $n$ is a constant asymptotically larger than $0$. Moreover, we also prove that Mixup helps calibration on out-of-domain data, which is critical for machine learning applications. %and extend our models to Gaussian generative models  %\cnote{Awkward phrasing in the last few sentences. Not sure if need to mention experiments here.} %we complement our theoretical results with further experiments showing the relation between the model capacity and the degree Mixup helping calibration in a variety of models and data sets. Lastly,  

In order to derive tractable analysis, we first study the concrete and natural Gaussian model. The Gaussian model is a popular setting for understanding phenomena happening in more complex models due to its tractability in theory and its ability to partially capture some essence of the phenomena. Indeed, the Gaussian model has been widely used in theoretical investigations of more complex machine learning models such as neural networks in adversarial learning \citep{schmidt2018adversarially,carmon2019unlabeled,dan2020sharp, deng2020improving}. %{\red KK: is there more references, especially ones published instead of arXiv? I think the justification of Guassian model will be one of the crucial point in review}\linjun{I've updated the reference, all three papers were accepted}
\textbf{\textit{We further extend our analysis to the very flexible Gaussian generative models}} in Section \ref{subsec:ggm}. 

% Its simplicity yet sharing similar phenomenon with more complex models has made this Gaussian model widely used in machine learning theoretical investigations such as in adversarial learning for neural networks \citep{carmon2019unlabeled,deng2020improving}. We further extend our analysis to the very flexible Gaussian generative model, as discussed in detail in Section \ref{subsec:ggm}.

\paragraph{The Gaussian model.}  We consider a common model used for theoretical machine learning analysis: a mixture of two spherical Gaussians with one component per class \citep{carmon2019unlabeled}: 

\begin{definition}[Gaussian model] \label{model:class}
For $\theta^*\in\bR^p$ and $\sigma>0$, the $(\theta^*,\sigma)$-Gaussian model is defined as the following distribution over $(x,y)\in \bR^p\times\{1,-1\}$: 
$$ x\mid y\sim \cN(y\cdot\theta^*,\sigma^2 I), \text{ for } i=1,2,...,n, %+\frac{1}{2}N(-\mu,I).
$$
and $y$ follows the Bernoulli distribution $\bP(y=1)=\bP(y=-1)=1/2$.
\end{definition}

For simplicity, we first consider the case where $\sigma$ is known, and the only unknown parameter is $\mu$. The case where $\sigma$ is unknown is a special example of the general Gaussian generative model that we will consider in Section~\ref{subsec:ggm}.%{\red need to discuss the unknown $\sigma$ case}

\paragraph{Algorithms.} In this section, we focus on studying the following linear classifier for the Gaussian classification. Specifically, the classifier follows the celebrated Fisher's rule  \citep{johnson2002applied}, or so-called linear discriminant analysis, which is also considered by \citet{carmon2019unlabeled} to study the adversarial robustness.  The classifier is constructed as
\begin{equation}\label{eq:C}
\hat{\mathcal{C}}(x)=\sgn(\hat\theta^\top x), 
\end{equation} 
 where $\hat\theta=\sum_{i=1}^n x_i y_i/n$. Given $\hat\theta$ and $x$, the output $y$ obtained via classifier $\hat{C}$ can be equivalently defined by the following process: we first obtain the confidence vector $h(x)=(p_{1}(x),p_{-1}(x))^\top,$ and then output
 $y=\hat{\mathcal{C}}(x)=\argmax_{k\in \{-1,1\}} p_k(x).$
Here, for $k\in \{-1,1\}$, the confidence score $p_k(x)$ represents an estimator of $\bP(y=k|x)$ and therefore takes the following form:
%$$p_j(x)=\text{\zhun{fill in later}}$$ 
\begin{equation}\label{eq:score}
p_k(x)=\frac{1}{e^{-2k\cdot\hat{\theta}^\top x_i/\sigma^2}+1}.
\end{equation}

In comparison, by applying Mixup to the above algorithm, we first obtain $\{\tilde x_{i,j}(\lambda),\tilde y_{i,j}(\lambda)\}_{i,j=1}^n$, which leads to another classifier
\begin{equation}\label{eq:Cmix}
\hat{\mathcal{C}}^{mix}(x)=\sgn(\hat{\theta}^{mix\top} x),  
\end{equation}
where $\hat{\theta}^{mix}=\bE_{\lambda\sim\cD_\lambda}\sum_{i,j=1}^n \tilde x_{i,j}(\lambda) \tilde{y}_{i,j}(\lambda)/n^2$. Here, given the randomness of $\lambda$, we take expectation with respect to $\lambda$ in the same way as in the previous study \citep{zhang2017mixup}, though this is unnecessary in our theoretical analysis. The confidence score obtained by $\hat{\mathcal{C}}^{mix}$ can be obtained similarly to that in Eq.~\eqref{eq:score}  with $\hat{\theta}$ being replaced by $\hat{\theta}^{mix}$.

%Again, we can view the output of this classifier is obtained by a similar two-step procedure, the only difference is that $\hat{\theta}$ is replaced by $\hat{\theta}^{mix}$. \james{last sentence not clear}

\subsection{Mixup helps calibration in classification} \label{subsec:mixuphelps}
We follow the convention in high-dimensional statistics, where the parameter dimension $p$ grows along with the sample size $n$, and state our theorem in the large $n,p$ regime where both $n$ and $p$ goes to infinity.  %\emph{i.e.} $p/n\rightarrow c$, where $c$ is a universal constant (not depending on $n$ and $p$). 

Throughout the paper, we use the term ``with high probability" to indicate that the event happens with probability at least $1-o(1)$, where $o(1)\rightarrow 0$ as $n\rightarrow \infty$ and the randomness is taken over the training data set. In the following, we show that the condition $p/n=\Omega(1)$ is necessary and the fact that Mixup improves calibration is a high-dimensional phenomenon.  

Let us denote the ECE calculated with respect to $\hat{\mathcal{C}}$ and $\hat{\mathcal{C}}^{mix}$ by $\ECE(\hat{\mathcal{C}})$ and  $\ECE(\hat{\mathcal{C}}^{mix})$ respectively. Our first theorem states that Mixup indeed improves calibration for the above algorithm under the Gaussian model. 

\begin{theorem}\label{thm:helpcalibration}
Under the settings described above, there exists $c_2>c_1>0$, when $p/n\in(c_1,c_2)$ and $\|\theta\|_2<C$ for some universal constants $ C>0$ (not depending on $n$ and $p$), then for sufficiently large $p$ and $n$, there exist $\alpha,\beta>0$, such that when the distribution $\cD_\lambda$ is chosen as $Beta(\alpha,\beta)$, with high probability,
 $$
ECE(\hat{\mathcal{C}}^{mix})<ECE(\hat{\mathcal{C}}).
$$
\end{theorem}
{The above theorem states that when $p$ is comparable to $n$ and $p/n$ is not too small, applying Mixup leads to a better calibration than without applying Mixup. In the very next theorem, we further demonstrate that the condition `` $p$ and $n$ are comparable" is necessary for Mixup to reach a smaller ECE. }

%\begin{theorem}\label{thm:helpcalibration}
%Under the settings described above, for sufficiently large $p$ and $n$, there exist $c_2>c_1>0$ and $\alpha,\beta>0$ such that, if $p/n\in(c_1,c_2)$, $\|\theta\|_2<C$ for some universal constants $ C>0$ (not depending on $n$ and $p$), and $\cD_\lambda$ is chosen as $Beta(\alpha,\beta)$, then with high probability,
% $$
%ECE(\hat C^{mix})<ECE(\hat C).
%$$
%\end{theorem}

%\begin{theorem}\label{thm:lowdim}
%There exists a threshold $\tau=o(1)$, when $p/n\le\tau$ and $\|\theta\|_2<C$ for some universal constant $C>0$, given any constants $\alpha,\beta>0$ (not depending on $n$ and $p$), when $n$ is sufficiently large, we have, with high probability, $$
%ECE(\hat C)<ECE(\hat C^{mix}).
%$$
%\end{theorem}
\begin{theorem}\label{thm:lowdim}
There exists a threshold $\tau=o(1)$ such that if $p/n\le\tau$ and $\|\theta\|_2<C$ for some universal constant $C>0$, given any constants $\alpha,\beta>0$ (not depending on $n$ and $p$), when $n$ is sufficiently large, we have, with high probability, $$
ECE(\hat{\mathcal{C}})<ECE(\hat{\mathcal{C}}^{mix}).
$$
\end{theorem}
%\james{Why need $c$ instead of just saying that $ECE(\hat C)< ECE(\hat C_{mix})$? }
In Theorem \ref{thm:lowdim}, we can see if $p$ is too small compared with $n$, then applying Mixup cannot have any gain and even hurts the calibration.

Usually, in the implementation of Mixup, we first fix $\alpha$ and $\beta$ before training, and the above theorem reveals the fact that in the low-dimensional regime, where $p/n$ is sufficiently close to $0$, the Mixup could not help calibration with high probability. Moreover, combined with Theorem \ref{thm:capacity} stated below, which characterizes the monotonic relationship between $p/n$ and the improvement brought by Mixup, we can see Mixup helps calibration more when the dimension is higher.

For the ease of presentation, for all $\beta>0$, let us define $Beta(0,\beta)$ as the degenerated distribution which takes the only value at $0$ with probability one. We also define $\hat{\mathcal{C}}^{mix}_{\alpha,\beta}$ as the classifier where we apply Mixup with distribution $\lambda\sim Beta(\alpha,\beta)$.

\begin{theorem}\label{thm:capacity}
For any constant $c_{\max}>0$, $p/n\to c_{ratio}\in(0,c^{\max})$, when  $\theta$ is sufficiently large (still of a constant level), we have for any $\beta>0$, with high probability, the change of ECE by using Mixup, characterized by
$$
\frac{d}{d\alpha}ECE(\hat{\mathcal{C}}^{mix}_{\alpha,\beta})\mid_{\alpha\to0+}
$$ is negative, and monotonically decreasing with respect to $c_{ratio}$. 
%\zhun{fill in later, describe monotonicity} {\red hard to describe if we didn't define use $\hat\theta_{mix}$ instead of $\hat\theta(\lambda)$}
\end{theorem}
In Theorem~\ref{thm:capacity}, the derivative with respect to $\alpha$ is interpreted as follows. Since for any $\beta>0$, $Beta(0,\beta)$ is the degenerated distribution at 0, $\hat\theta^{mix}(0,\beta)$ corresponds to the output without Mixup. Therefore, increasing $\alpha$ from 0 to some positive value implies applying Mixup. Thus, Theorem~\ref{thm:capacity} suggests that in high-dimensions, increasing the interpolation range in Mixup decreases ECE. 
%{\red need more explanation and interpretation: the gain of mixup increases when dimension is large}

\paragraph{Intuition behind our results.} In the high-dimensional regime, especially in the over-parameterized case ($p>n$), the models have more flexibility to set the confidence vectors. For instance, for trained neural networks, the entries of the confidence vectors for many data points are all close to zero except for one entry, whose value is close to $1$,  because the model is trained to memorize the training labels. Mixup mitigates this problem by using linear interpolation that creates one-hot encoding terms with entry values lying between $(0,1)$, which pushes the value of entries to diverge. This could be partially addressed in our analysis above, as the magnitude of the confidence is closely related to  $\|\hat{\theta}\|$, i.e. when $\|\hat{\theta}\|$ is large, the confidence scores are more likely to be close to $0$ or $1$. Mixup, as a form of regularization \citep{zhang2020does}, could shrink $\|\hat{\theta}\|$ and avoid too extreme confidence scores. %\zhun{Linjun, please check the notations are consistent here, may need change of words}

\paragraph{Additional supporting experiments.}
To complement our theory, we further provide more experimental evidence on popular image classification data sets with neural networks.
We used the fully-connected neural networks with various values of the width (i.e. the number of neurons per hidden layer) and the depth (i.e., the number of hidden layers). For the experiments on the effect of the width, we fixed the depth  to be 8 and varied the width from 10 to 3000. For the experiments on the effect of the depth, the depth was varied from 1 to 24 (i.e., from 3 to 26 layers including input/output layers) by fixing the width to be 400 with data-augmentation and 80 without data-augmentation. We used the following standard data-augmentation operations using \texttt{torchvision.transforms} for both data sets: random crop (via \texttt{RandomCrop(32, padding=4}) and random horizontal flip (via \texttt{RandomHorizontalFlip}) for each image.
%\james{describe a bit what data-augmentations are used here.} %\james{Could merge these three figs into one large 3-by-2 figure. It might be clearer to put CIFAR-100 without augmentation in the 2nd row since it's more similar to CIFAR-100 with augmentation.}. 
In this experiment, we used the standard data sets --- CIFAR-10 and CIFAR-100  \citep{krizhevsky2009learning}. We used stochastic gradient descent (SGD) with mini-batch size of 64.  We set the learning rate to be 0.01 and momentum coefficient to be 0.9.  We used the Beta distribution $Beta(\alpha,\alpha)$ with $\alpha=1.0$ for Mixup.  The results are reported in Figure \ref{fig:1} and \ref{fig:4} with a fully-connected neural network. Consistently across all the experiments, Mixup reduces ECE for larger capacity models and can hurt ECE for small models, which matches our theory. {For reasons of space, since our focus is mainly on the calibration, the empirical results regarding test accuracy for each figure are deferred to the Appendix.}

%It has been 
%{\cyan We would like to point out that for large capacity models, the Mixup not only improves the ECE, but also boosts the test accuracy. The detailed numerical results are presented in the Appendix.} %\james{Do we want to say a sentence about the model's accuracy? We want to make it clear to the reader that ECE reduction due to Mixup is orthogonal to changes in the model's accuracy. For example, is it true that Mixup improves accuracy for small models even when it hurts calibration?} 
%\linjun{Hi James, I just discussed with Kenji. He mentioned that in our experiment, we use fully-connected NN, so the test errors are not small enough relative to SOTA. It would be better we do not report the test accuracy of our experiment. Maybe we can add one sentence like ""although mixup has been shown to improve the test accuracy, there has been little research studying the calibration.." at the begining of Section 3? ok i'll take a sentence there. Thanks! }%\linjun{In experiments, we  found that Mixup also hurts accuracy for small models, so I write the sentence like above. Do you think it is fine? -Hi Kenji: please feel free to change my sentence. Thank you both!}

\begin{figure}[t!]
\centering
\begin{subfigure}{0.46\columnwidth}
  \includegraphics[width=\textwidth, height=0.7\textwidth]{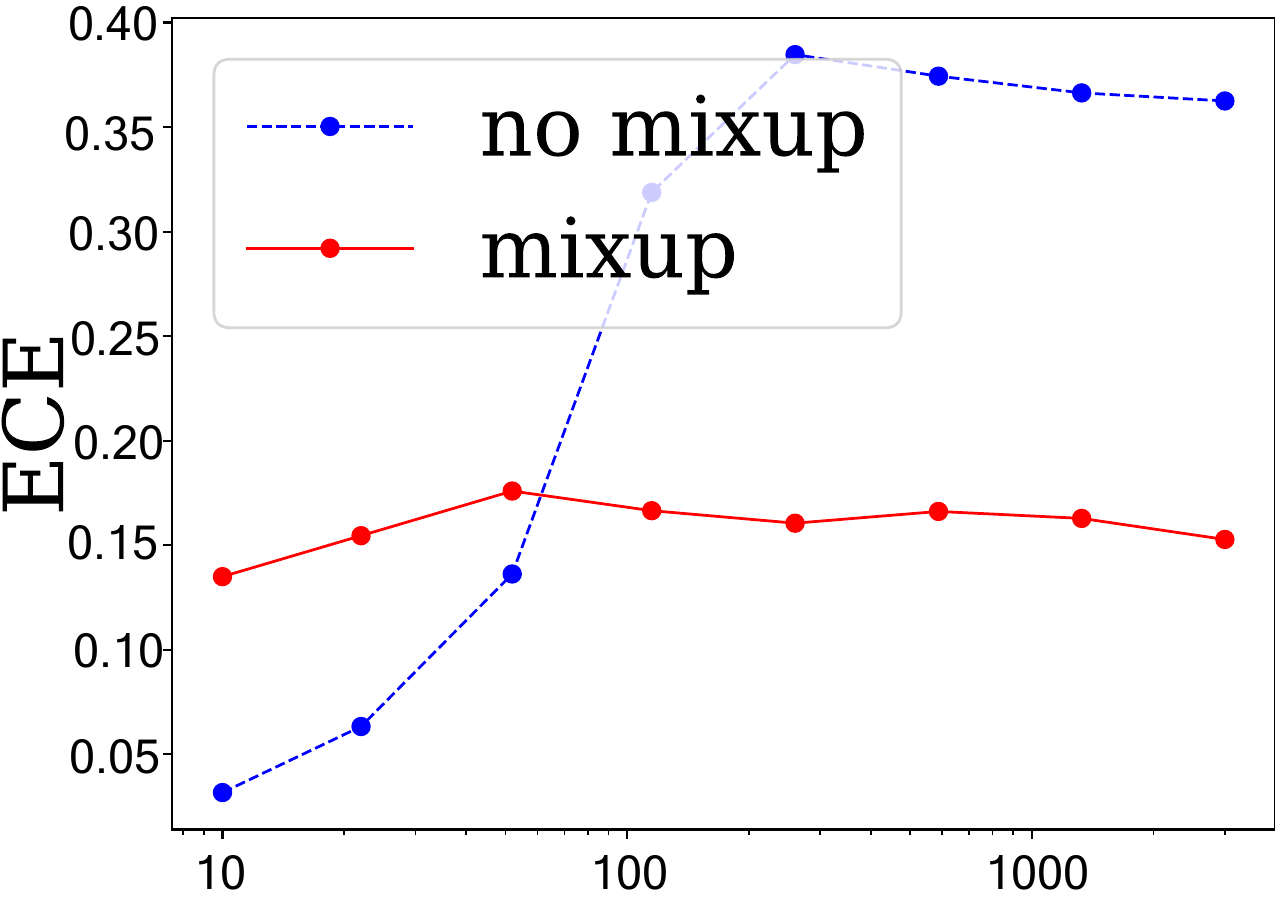}
 \caption{Width}
\end{subfigure}
\hspace{0.1in}
\begin{subfigure}{0.46\columnwidth}
  \includegraphics[width=\textwidth, height=0.7\textwidth]{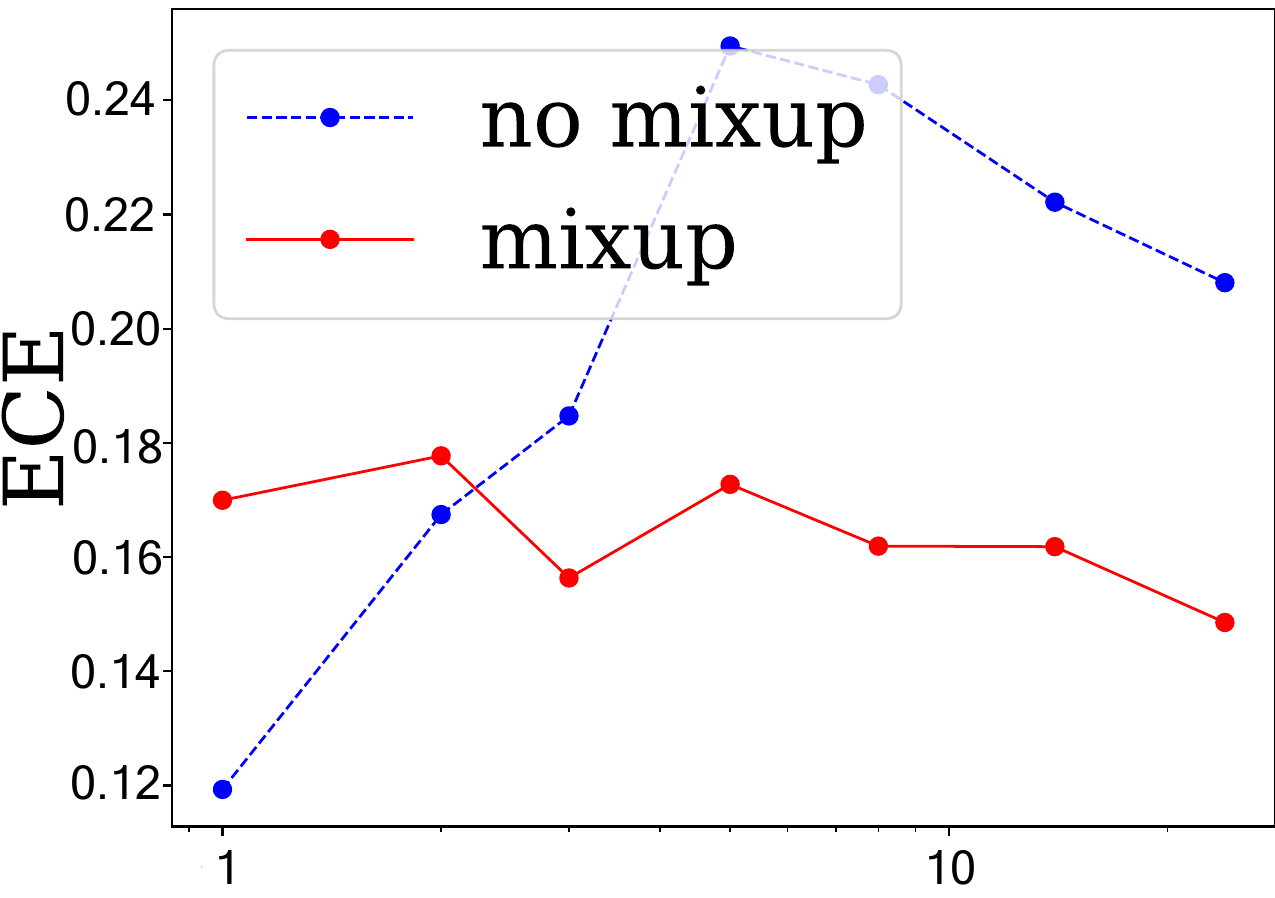}
 \caption{Depth}
\end{subfigure} 
\hspace{0.1in}
\begin{subfigure}{0.46\columnwidth}
  \includegraphics[width=\textwidth, height=0.7\textwidth]{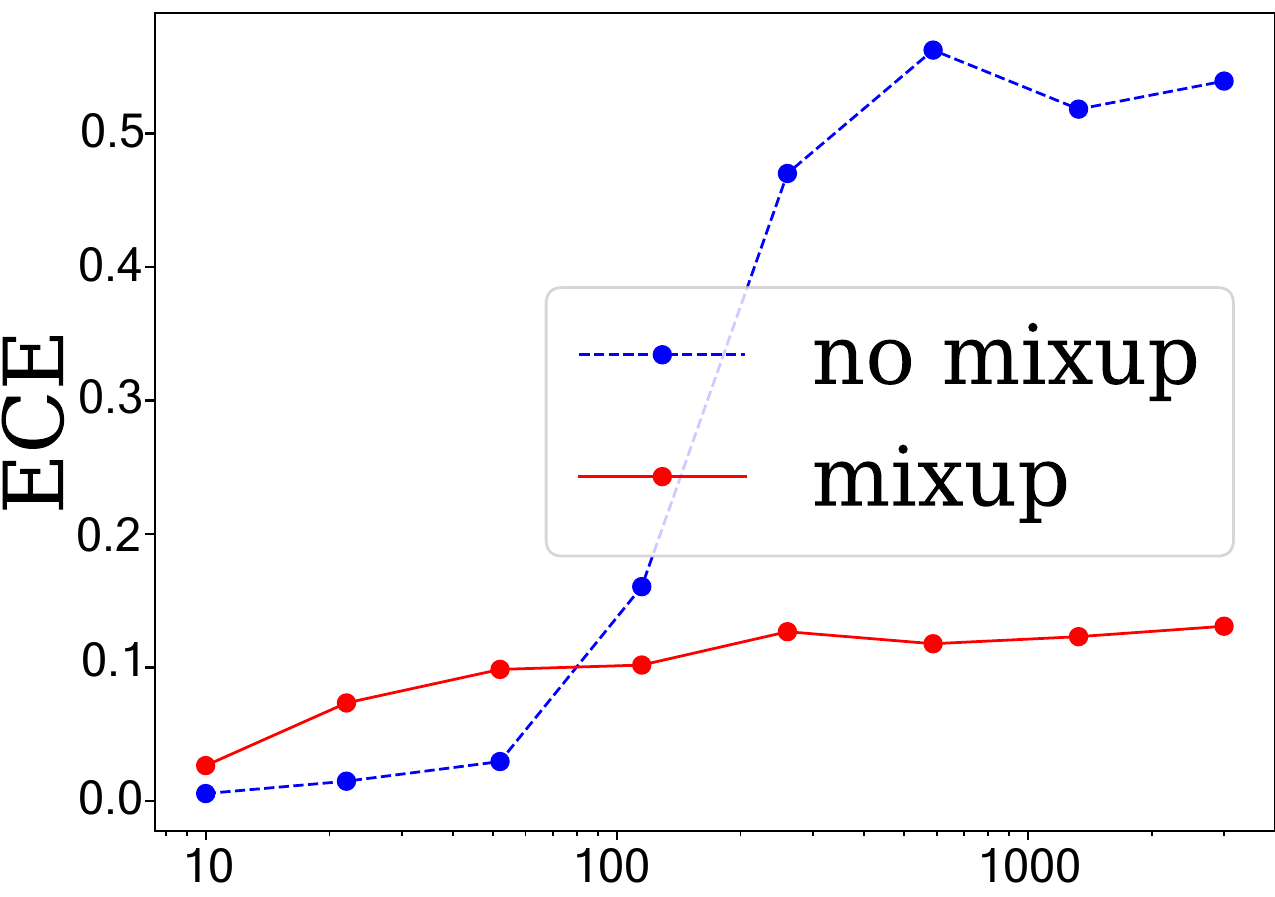}
 \caption{Width}
\end{subfigure}
\hspace{0.1in}
\begin{subfigure}{0.46\columnwidth}
  \includegraphics[width=\textwidth, height=0.7\textwidth]{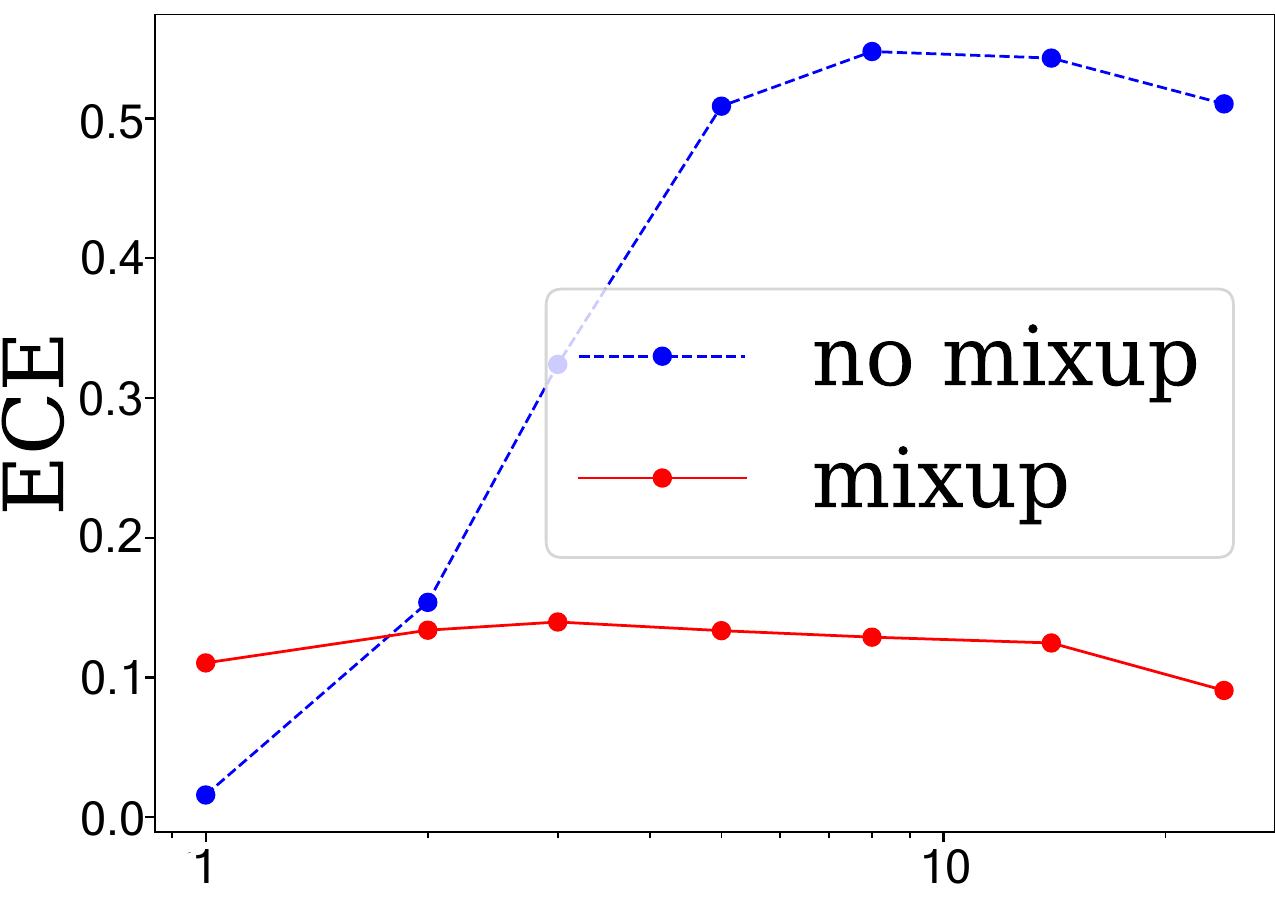}
 \caption{Depth}
\end{subfigure}
\hspace{0.1in}
\begin{subfigure}{0.46\columnwidth}
  \includegraphics[width=\textwidth, height=0.7\textwidth]{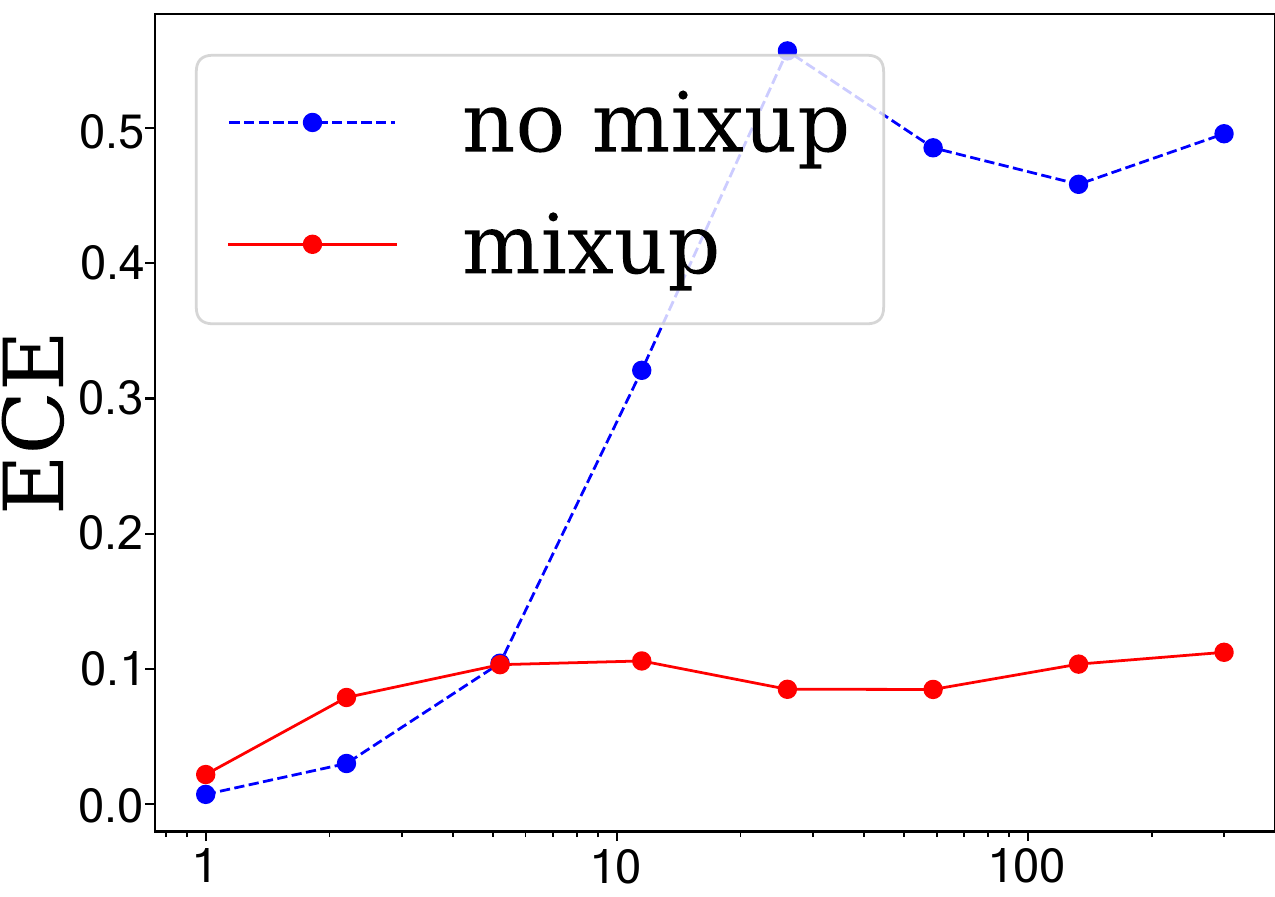}
 \caption{Width}
\end{subfigure}
\hspace{0.1in}
\begin{subfigure}{0.46\columnwidth}
  \includegraphics[width=\textwidth, height=0.7\textwidth]{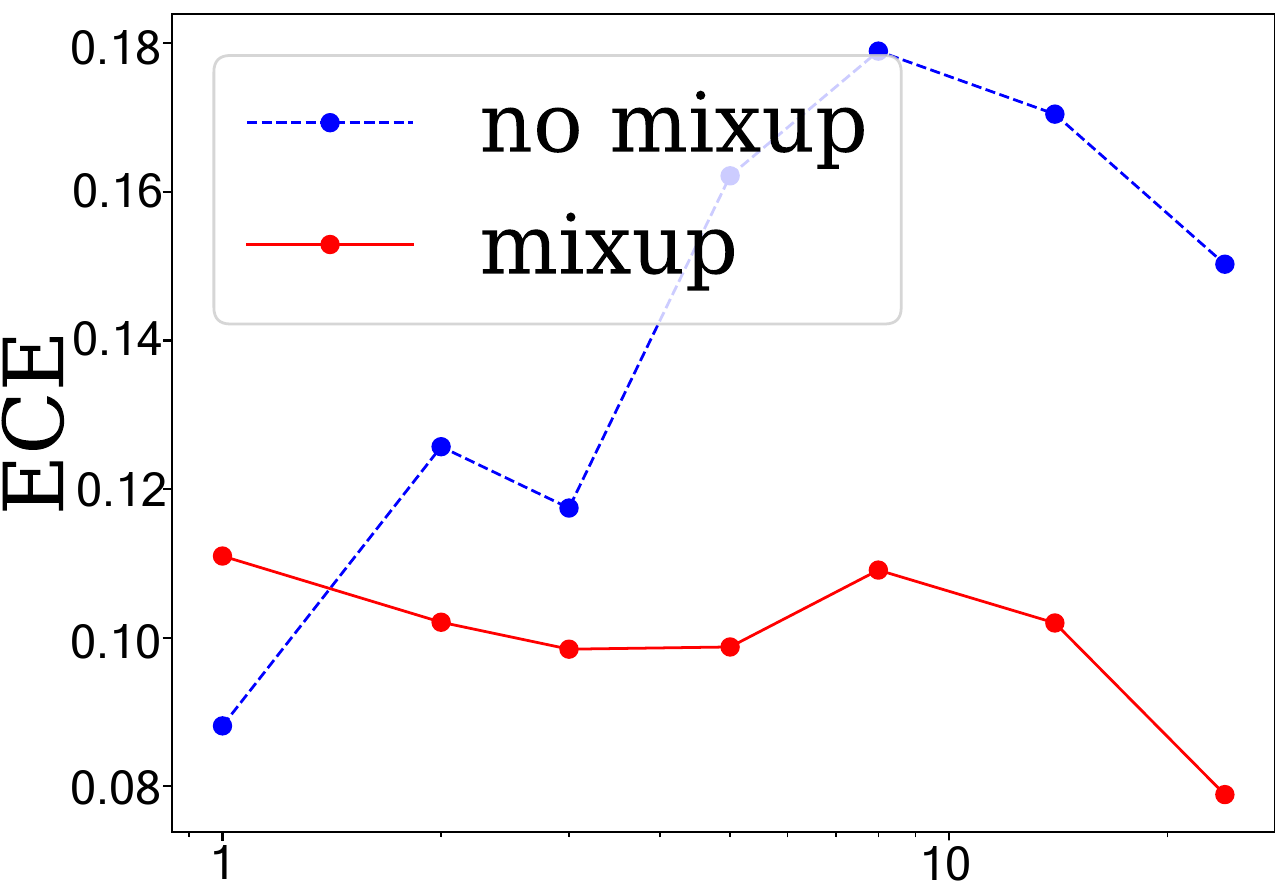}
 \caption{Depth}
\end{subfigure} 
\caption{Expected calibration error (ECE). (a), (b): CIFAR-10 without data augmentation; (c), (d): CIFAR-100 with data augmentation; (e), (f): CIFAR-100 without data augmentation.} 
\label{fig:4} 
\end{figure}

\subsection{Improvement for out-of-domain data }
In this section, we evaluate the quality of predictive uncertainty on out-of-domain inputs. It has been found empirically that in the out-of-domain setting, Mixup can also enhance the reliability of prediction and boost the performance in calibration comparing to the standard training (without Mixup) \citep{thulasidasan2019mixup, tomani2020towards}. To explain the above phenomenon, using the similar analysis as those for Theorem \ref{thm:helpcalibration}, we provide the following theorem.

%Specifically, in Section $7$ of \citet{thulasidasan2019mixup}, the authors train a VGG-$16$ network on in-distribution data (STL-10) and then predict on the data sampled from the ImageNet database that have not been encountered. They also test on newly generated random Gaussian samples with the same mean and variance as the training set. They used mixup to show that it's easier to detect out-of- distribution samples due to low model confidence on out-of-domain data, which. 

\begin{theorem}\label{thm:out}
Let us consider the ECE evaluated on the out-of-domain Gaussian model with mean parameter $\theta'$, that is, $\bP(y=1)=\bP(y=-1)=1/2$, and
$ x\mid y\sim \cN(y\cdot\theta',\sigma^2 I), \text{ for } i=1,2,...,n. %+\frac{1}{2}N(-\mu,I).
$ If we have $(\theta'-\theta^*)^\top\theta^*\le p/(2n)$, then when $p$ and $n$ are sufficiently large, with high probability, 
$$ECE(\hat{\mathcal{C}}^{mix};\theta',\sigma)<ECE(\hat{\mathcal{C}};\theta',\sigma),$$
%\zhun{fill in later}
where $\ECE(\cdot;\theta',\sigma)$ denotes the expected calibration error calculated with respect to the out-of-domain distribution -- the Gaussian model with parameters $\theta'$ and $\sigma$, while $\hat{\mathcal{C}}^{mix}$ and $\hat{\mathcal{C}}$ are still obtained via (\ref{eq:C}) and (\ref{eq:Cmix}) via the in-domain training data.
\end{theorem}

The above theorem states the continuity of the boosting effect of Mixup over the domain shift. As long as the domain shift is not too large, Mixup still helps calibration.

\subsection{Gaussian generative model}\label{subsec:ggm}
Now let us consider a more general class of distributions, the Gaussian generative model, which is a flexible distribution and has been commonly considered in the machine learning literature. For example, many common deep generative models such as 
Generative Adversarial Nets (GANs) \citep{goodfellow2014generative} are Gaussian generative models, where the input is a Gaussian sample. 
\begin{definition}[Gaussian generative model] \label{model:class2}
For $\theta^*\in\bR^p$ and $g:\R^p\to\R^d$ ($d\ge p$), the $(\theta^*,g)$-Gaussian model is defined as the following distribution over $(x,y)\in \bR^d\times\{1,-1\}$, $x=g(z)$, where: 
$$ z\mid y\sim \cN(y\cdot\theta^*, I), \text{ for } i=1,2,...,n, %+\frac{1}{2}N(-\mu,I).
$$
and $y$ follows the Bernoulli distribution $\bP(y=1)=\bP(y=-1)=1/2$.
\end{definition}

Now suppose we can learn an $h\in\{h: h\circ g \text{ is an identity mapping in }\R^p\}$ approximately such that the estimator $\hat h$ satisfies the following condition.
\begin{assumption}\label{assump}
For any given $v\in\R^p$, $k\in\{-1,1\}$, there exists a $\theta^*\in\R^p$, such that given $y=k$, the probability density function of $R_1=v^\top \hat h(x)$ and $R_2=v^\top h(x)=v^\top z\sim N(k\cdot b^\top\theta^*, \|v\|^2)$ satisfies that $p_{R_1}(u)=p_{R_2}(u)\cdot (1+\delta_u)$ for all $u\in\R$ where $\delta_u$ satisfies $\E_{R_1}[|\delta_u|]=o(1)$  when $n\to\infty$. 
\end{assumption}

As a special case of Definition~\ref{model:class2}, we consider the Gaussian model %(Definition~\ref{model:class}) 
with unknown $\sigma$. Estimating $\sigma$ by 
$$\hat\sigma=\sqrt{\|\sum_{i=1}^n (x_i-y_i\hat\theta)\|^2/pn}$$ 
will satisfy Assumption~\ref{assump} when $\|\theta^*\|<C$ for some universal constant $C$. In practice, for more general cases, we can learn such $h$ following the framework of GANs. For example, 
$$\hat h=\argmin_{h}\max_{k\in\{-1,1\}} \cW(h(x),z\mid y),$$ 
where $z$ is the Gaussian mixture defined in Definition~\ref{model:class2} with $\theta^*=1_p/\sqrt{p}$ and $\sigma=1$, and $\cW(\cdot,\cdot)$ denotes the Wasserstein distance. Due to the flexibility of $h$, the choice of $\theta^*$ and $\sigma$ will not impact the training process.

Now we consider the following two classifiers:  
$$\hat{\mathcal{C}}(x)=\sgn(\hat\theta^\top \hat h(x)),$$ 
where $\hat\theta=\sum_{i=1}^n \hat h(x_i) y_i/n$, and 
$$\hat{\mathcal{C}}^{mix}(x)=\sgn(\hat{\theta}^{mix\top} \hat h(x))$$
where $$\hat{\theta}^{mix}=\sum_{i,j=1}^n\bE_{\lambda\sim\cD_\lambda} (\lambda \hat h(x_i)+(1-\lambda)\hat h(x_j))\cdot \tilde{y}_{i,j}(\lambda)/n^2.$$ Similarly, for a generic $\hat\theta$, the confidence scores are given by %\zhun{ remove 2 here?} 
$$
p_k(x)=1/(e^{-2k\cdot\hat\theta^\top \hat h(x)}+1).$$ We then have the following result showing that under the more general Gaussian generative model, the Mixup method could still provably lead to an improvement on the calibration. 
\begin{theorem}\label{thm:helpcalibration2}
Under the settings described above with Assumption~\ref{assump}, there exists $c_2>c_1>0$, when $p/n\in(c_1,c_2)$, $\hat h$ is $L$-Lipschitz, and $\|\theta\|_2<C$ for some universal constants $L, C>0$ (not depending on $n$ and $p$), then for sufficiently large $p$ and $n$, there exist $\alpha,\beta>0$ for the Mixup distribution $\cD_\lambda = Beta(\alpha,\beta)$, such that, with high probability,
 $$
\ECE(\hat{\mathcal{C}}^{mix})<\ECE(\hat{\mathcal{C}}).
$$
\end{theorem}

\begin{figure*}[t!]
\centering
\begin{subfigure}[b]{0.7\columnwidth}
  \includegraphics[width=\textwidth, height=0.7\textwidth]{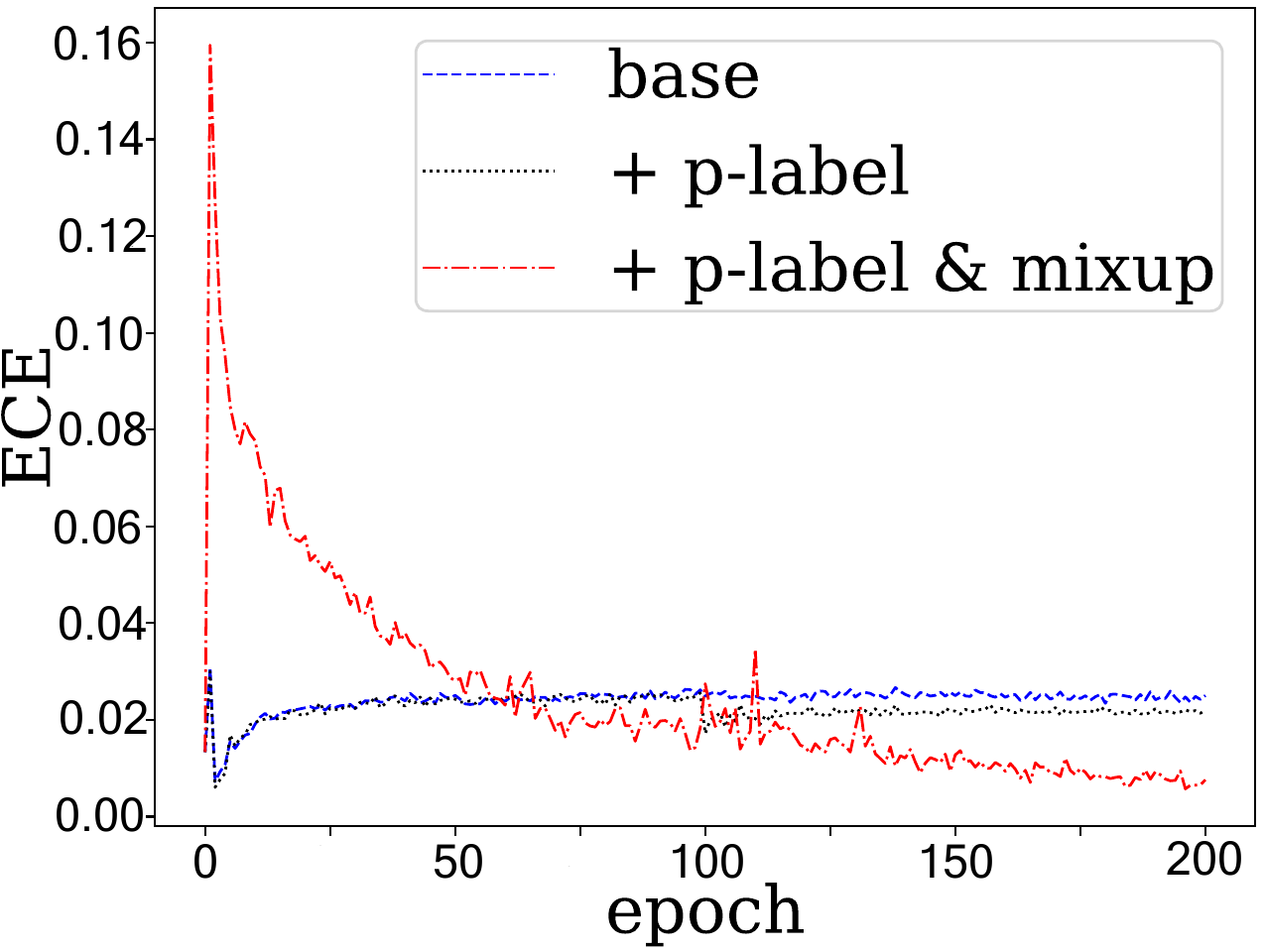}
 \caption{Kuzushiji-MNIST}
 \label{subf:kmnist}
\end{subfigure}
\hskip 15mm
\begin{subfigure}[b]{0.7\columnwidth}
  \includegraphics[width=\textwidth, height=0.7\textwidth]{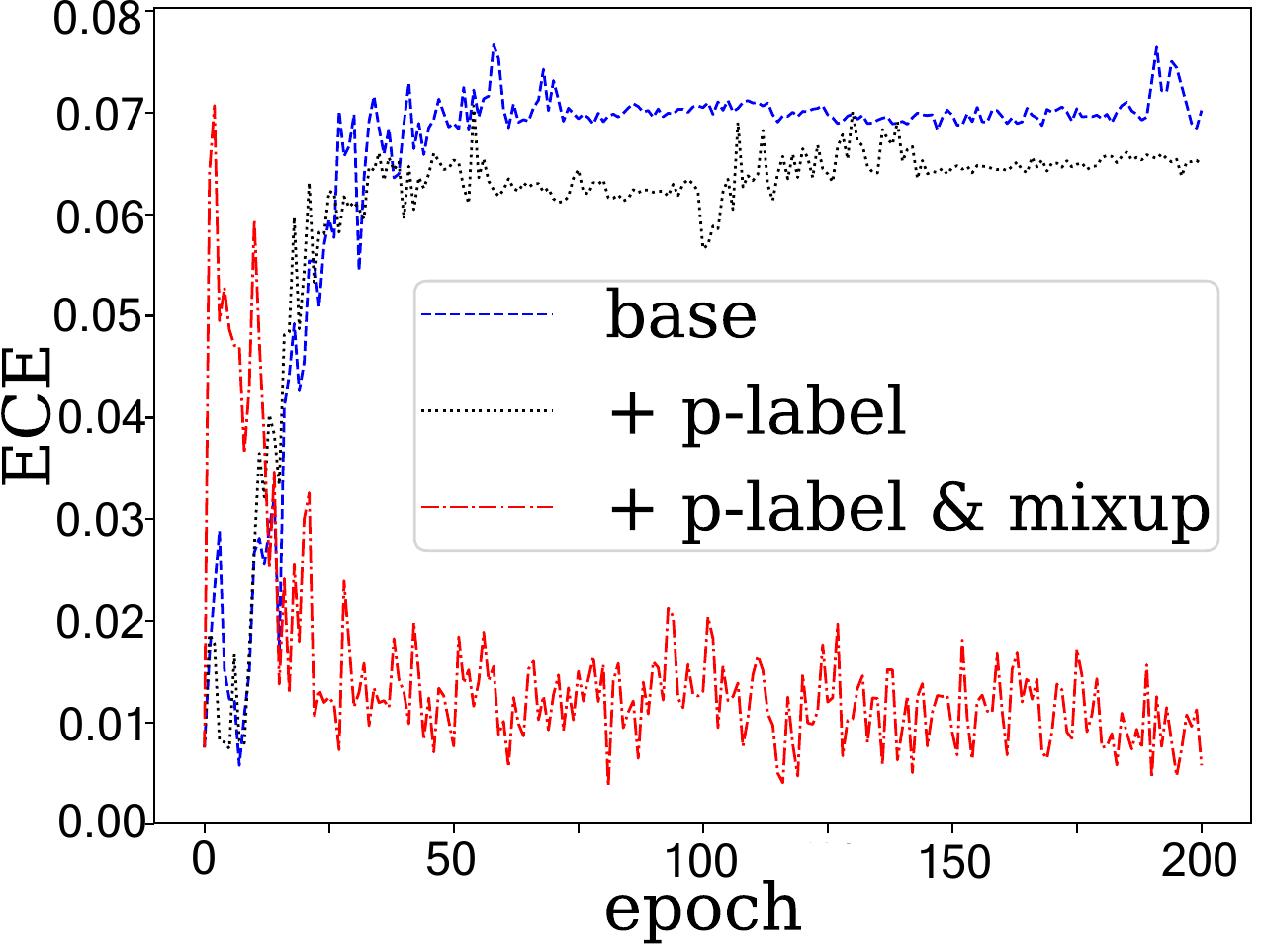}
 \caption{Fashion-MNIST}
 \label{subf:fmnist}
\end{subfigure}
\hskip 15mm
\begin{subfigure}[b]{0.7\columnwidth}
  \includegraphics[width=\textwidth, height=0.7\textwidth]{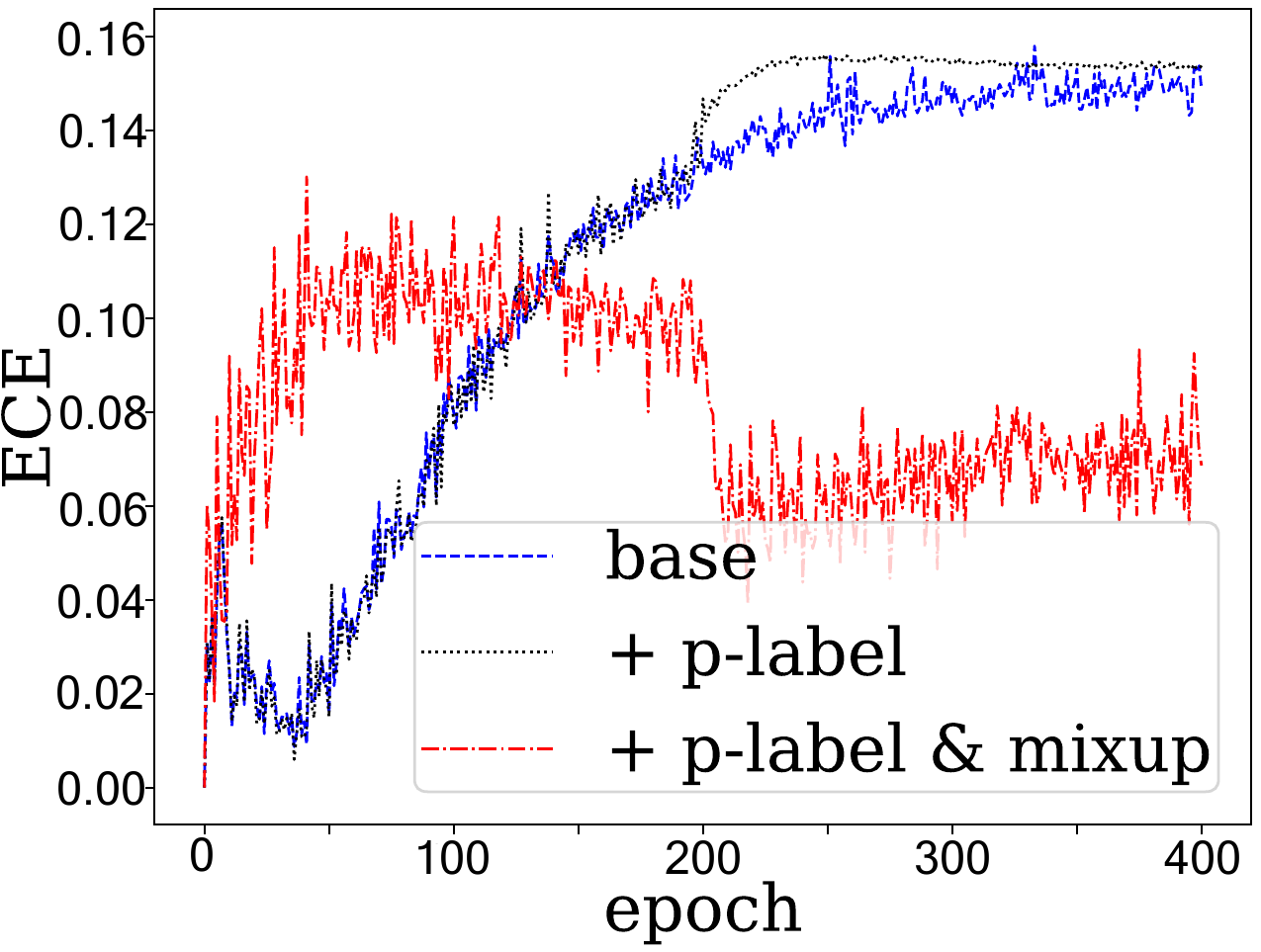}
 \caption{CIFAR-10}
 \label{subf:cifar10}
\end{subfigure}
\hskip 15mm
\begin{subfigure}[b]{0.7\columnwidth}
  \includegraphics[width=\textwidth, height=0.7\textwidth]{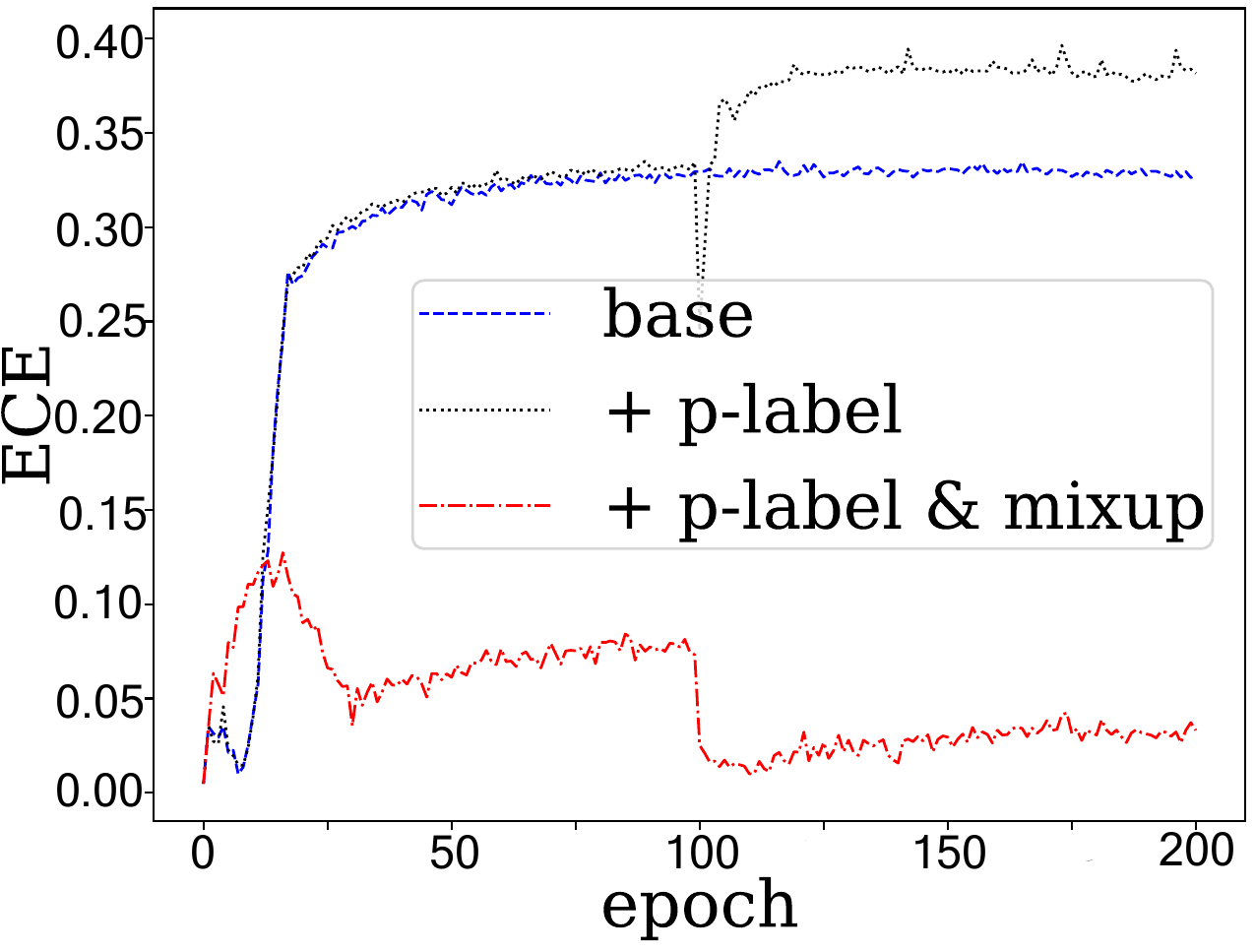}
 \caption{CIFAR-100}
 \label{subf:cifar100}
\end{subfigure}
\caption{ECE calculated for ResNets on varieties of data sets. In (a) and (b), using only pseudo-label algorithm improves calibration, while in (c) and (d), using only pseudo-label algorithm hurts calibration. Further applying Mixup in the last step of pseudo-label algorithm promotes calibration in both cases. The pseudo-labels (or p-labels in short) are inserted into training at the midpoint of the entire training: i.e., at epoch = 100 for (a), (b) and (d) and epoch = 200 for (c).} 
\label{fig:7} 
%\vspace{-0.3cm}
\end{figure*}
%======================================
\section{Mixup Improves Calibration in Semi-supervised Learning} \label{sec:2}
Data augmentation by incorporating cheap unlabeled data from multiple domains is a powerful way to improve prediction accuracy especially when there is limited labeled data. One of the commonly used semi-supervised learning algorithms is the pseudo-labeling algorithm \citep{chapelle2009semi}, which first trains an initial classifier $\hat{\mathcal{C}}_{init}$ on the labeled data, then assigns pseudo-labels to the unlabeled data using the $\hat{\mathcal{C}}_{init}$. Lastly, using the combined labeled and pseudo-labeled data to perform supervised learning and obtain a final classifier $\hat{\mathcal{C}}_{final}$. Previous work has shown that the pseudo-labeling algorithm has many benefits such as improving prediction accuracy and robustness against adversarial attacks \citep{carmon2019unlabeled}. However, as we observe from Figure \ref{subf:cifar10} and \ref{subf:cifar100}, incorporating unlabeled data via the pseudo-labeling algorithm does not always improve calibration; sometimes pseudo-labeling even hurts calibration. We find that further applying Mixup at the last step of pseudo-labeling algorithm mitigates this issue and improves calibration as shown in Figure \ref{fig:7}. The details of the experimental setup are included in the Appendix

\begin{algorithm}[b!]
   \caption{The pseudo-labeling algorithm }
   \label{alg:example}
\textbf{Step 1:} Obtain an initial classifier 
$$\hat{\mathcal{C}}_{init}(x)=sgn(\hat\theta_{init}^\top x),$$ 
where $\hat\theta_{init}=\sum_{i=1}^{n_l} x_iy_i/n_l.
$
\vspace{0.1cm}

\textbf{Step 2:} Apply $\hat{\mathcal{C}}_{init}$ on the unlabeled data set $\{ x^u_i\}_{i=1}^{n_u}$, and obtain pseudo-labels $y^u_i=\hat{\mathcal{C}}_{init}(x^u_i)$ for $i\in[n_u]$.
\vspace{0.1cm}

\textbf{Step 3:} Obtain the final classifier $\hat{\mathcal{C}}_{final}(x)=sgn(\hat\theta_{final}^\top x)$, where
$$
\hat\theta_{final}=\frac{1}{n_l+n_{u}}\left(\sum_{i=1}^{n_l} x_i y_i+\sum_{i=1}^{n_{u}} x^u_i y^u_i\right)$$
\end{algorithm}

We justify the empirical findings above by theoretically analyzing the calibration in the semi-supervised learning setting. %, where there are abundance of unlabeled data available. 
Specifically, we assume we have $n_l$ labeled data points $\{x_i, y_i\}_{i=1}^{n_l}$ and $n_u$ unlabeled data points $\{x^u_i\}_{i=1}^{n_u}$  i.i.d. sampled from the $(\theta^*,\sigma)$-Gaussian model in Definition~\ref{model:class}. The pseudo-labeling algorithm is the same as the one considered in \citet{carmon2019unlabeled}, which is shown in Algorithm \ref{alg:example}.

We then present two theorems. The first theorem demonstrates that when the labeled data is not sufficient, then under some mild conditions, the unlabeled data will help the calibration.
The second theorem characterizes settings where the standard pseudo-labeling algorithm (Algorithm \ref{alg:example}) 
makes calibration worse and increases ECE. 

%More specifically, our first theorem demonstrates that when the labeled data is not sufficient, then under some mild conditions, the unlabeled data will help the calibration. 

  \begin{theorem}\label{thm:semi1}
Suppose $C_1\sqrt{p/n_l}\le\|\theta\|\le C_2\sqrt{p/n_l}$ for some universal constant $C_1<1/2$ and $C_2>2$, when $p/n_l$, $\|\theta\|$, $n_u$ are sufficiently large, we have with high probability, $$
ECE(\hat{\mathcal{C}}_{final})<ECE(\hat{\mathcal{C}}_{init}).
$$
\end{theorem}
%\james{why need $c$ in all these statements?}

Meanwhile, in some cases, for instance, when the labeled data is sufficient, the pseudo-labeling algorithm may hurt the calibration, as shown in the following theorem.

\begin{theorem}\label{thm:semi2}
If $C_1\|\theta\|<C_2$, $p<C_3$ for some constants $C_1,C_2, C_3>0$. Let $n_l$ and $n_{u}\to\infty$, then with high probability, $$
ECE(\hat{\mathcal{C}}_{init})<ECE(\hat{\mathcal{C}}_{final}).
$$
\end{theorem}

The above two theorems suggest that the pseudo-labeling algorithm is not able to robustly guarantee improvement in calibration. In the following,
we show that we can mitigate this issue by applying Mixup to the last step in Algorithm \ref{alg:example}. Specifically, we consider the following classifier with Mixup: 
$$\hat{\mathcal{C}}_{mix,final}(x)=sgn(\hat\theta_{mix,final}^\top x),$$
where
$$
\hat\theta_{final,mix}(\lambda)=\bE_{\lambda\sim\cD_\lambda}[\frac{1}{n_l+n_{u}}\sum_{i=1}^{n_l+n_{u}} x^{l,u}_{i,j}(\lambda) y^{l,u}_{i,j}(\lambda)].$$
 Here $\{x^{l,u}_{i,j}(\lambda),y^{l,u}_{i,j}(\lambda)\}_{i,j=1}^{n_l+n_u}$ is the data set obtained by applying Mixup to the pooled data set by combining $\{x_i,y_i\}_{i=1}^{n_l}$ and $\{x^u_i,y^u_i\}_{i=1}^{n_u}$. 
We then have the following result showing Mixup helps the calibration in the semi-supervised setting. 
\begin{theorem}\label{thm:helpcalibrationsemi}
Under the setup described above, and denote the ECE of $\hat{\mathcal{C}}_{final}$ and $\hat{\mathcal{C}}_{mix,final}$ by $ECE(\hat{\mathcal{C}}_{final})$ and  $ECE(\hat{\mathcal{C}}_{mix,final})$ respectively. If $C_1<\|\theta\|<C_2$ for some universal constants $C_1,C_2$ (not depending on $n$ and $p$), then for sufficiently large $p$ and $n_l, n_{u}$, there exists $\alpha,\beta>0$, such that when the Mixup distribution $\lambda\sim Beta(\alpha,\beta)$, with high probability, we have
 $$
ECE(\hat{\mathcal{C}}_{mix,final})<ECE(\hat{\mathcal{C}}_{final}).
$$
\end{theorem}
{From Theorem \ref{thm:helpcalibrationsemi}, we can see that even though incorporating unlabeled data can sometimes make the model less calibrated, adding Mixup training consistently (i.e., under the same conditions of either Theorem~\ref{thm:semi1} or Theorem~\ref{thm:semi2}) mitigates this issue and provably improves calibration.}

\section{Extension to Maximum Calibration Error}
\label{sec:MCE}
 Here we further investigate how Mixup helps calibration under maximum calibration error. Similar conclusions can be reached for MCE as those for ECE in Section \ref{subsec:mixuphelps}, demonstrating that the effects of Mixup can be found across common calibration metrics.

\begin{figure}[!t]
	\centering
	\begin{subfigure}[b]{0.45\columnwidth}
		\includegraphics[width=\textwidth, height=0.7\textwidth]{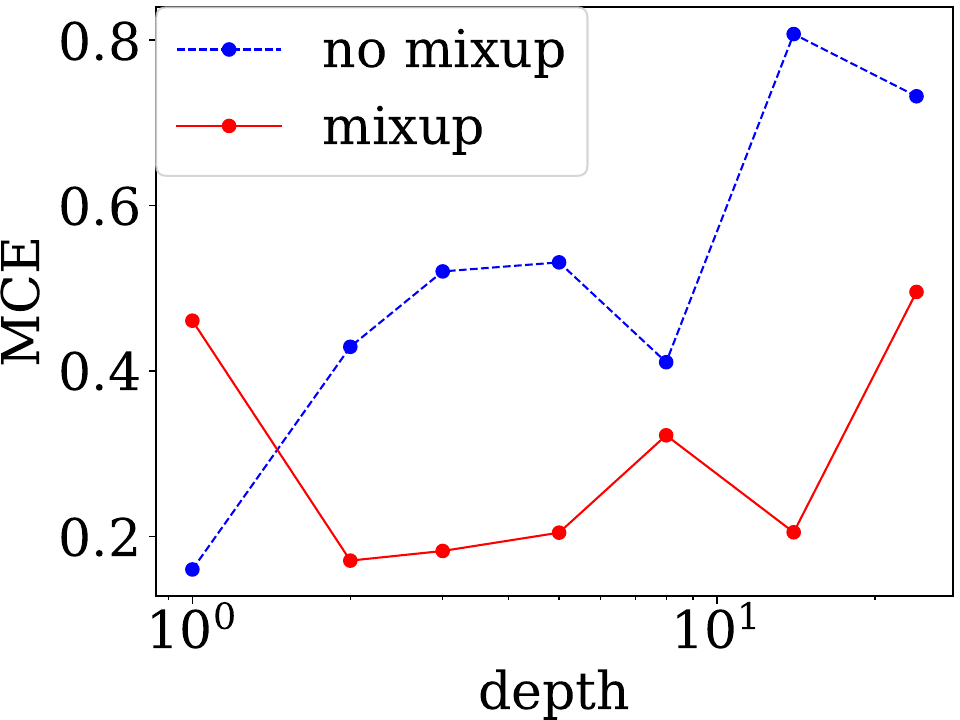}
		\caption{CIFAR-10}
	\end{subfigure}
\hskip 5mm
	\begin{subfigure}[b]{0.45\columnwidth}
		\includegraphics[width=\textwidth, height=0.7\textwidth]{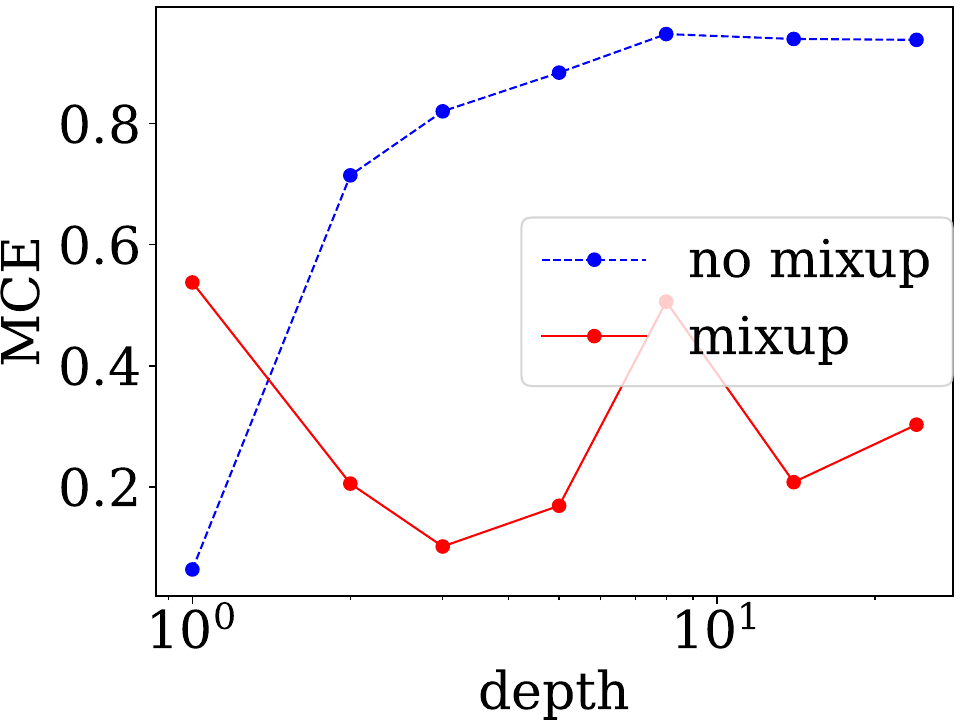}
		\caption{CIFAR-100}
	\end{subfigure}
	\caption{Maximum Calibration Error (MCE) calculated with varying network depth. Mixup augmentation can reduce MCE especially for larger capacity models (deeper networks) compared to these models trained without Mixup.} 
	\label{fig:mce:1} 
\end{figure}

 \begin{theorem}\label{thm:helpcalibrationMCE}
Under the settings described in Theorem \ref{thm:helpcalibration}, there exists $c_2>c_1>0$, when $p/n\in(c_1,c_2)$ and $\|\theta\|_2<C$ for some universal constants $ C>0$ (not depending on $n$ and $p$), then for sufficiently large $p$ and $n$, there exist $\alpha,\beta>0$, such that when the distribution $\cD_\lambda$ is chosen as $Beta(\alpha,\beta)$, with high probability,
 $$
MCE(\hat{\mathcal{C}}^{mix})<MCE(\hat{\mathcal{C}}).
$$
\end{theorem}
%\begin{theorem}\label{thm:helpcalibration}
%Under the settings described above, for sufficiently large $p$ and $n$, there exist $c_2>c_1>0$ and $\alpha,\beta>0$ such that, if $p/n\in(c_1,c_2)$, $\|\theta\|_2<C$ for some universal constants $ C>0$ (not depending on $n$ and $p$), and $\cD_\lambda$ is chosen as $Beta(\alpha,\beta)$, then with high probability,
% $$
%ECE(\hat C^{mix})<ECE(\hat C).
%$$
%\end{theorem}

From the above theorem, we can see Mixup can also help decrease the maximum calibration error. Comparing with ECE, from Figure \ref{fig:mce:1}, we can similarly observe that when the model capacity is small, Mixup does not really help. We here provide the following theorem to further illustrate that point.

%\begin{theorem}\label{thm:lowdim}
%There exists a threshold $\tau=o(1)$, when $p/n\le\tau$ and $\|\theta\|_2<C$ for some universal constant $C>0$, given any constants $\alpha,\beta>0$ (not depending on $n$ and $p$), when $n$ is sufficiently large, we have, with high probability, $$
%ECE(\hat C)<ECE(\hat C^{mix}).
%$$
%\end{theorem}
\begin{theorem}\label{thm:lowdimMCE}
There exists a threshold $\tau=o(1)$ such that if $p/n\le\tau$ and $\|\theta\|_2<C$ for some universal constant $C>0$, given any constants $\alpha,\beta>0$ (not depending on $n$ and $p$), when $n$ is sufficiently large, we have, with high probability, $$
MCE(\hat{\mathcal{C}})<MCE(\hat{\mathcal{C}}^{mix}).
$$
\end{theorem}
%\james{Why need $c$ instead of just saying that $ECE(\hat C)< ECE(\hat C_{mix})$? }
Lastly, we provide a similar theorem as Theorem \ref{thm:capacity} to further illustrate that Mixup helps in the high-dimensional (overparametrized) regime.

\begin{theorem}\label{thm:capacityMCE}
For any constant $c_{\max}>0$, $p/n\to c_{ratio}\in(0,c^{\max})$, when  $\theta$ is sufficiently large (still of a constant level), we have for any $\beta>0$, with high probability, the change of ECE by using Mixup, characterized by
$$
\frac{d}{d\alpha}MCE(\hat{\mathcal{C}}^{mix}_{\alpha,\beta})\mid_{\alpha\to0+}
$$ is negative, and monotonically decreasing with respect to $c_{ratio}$. 
%\zhun{fill in later, describe monotonicity} {\red hard to describe if we didn't define use $\hat\theta_{mix}$ instead of $\hat\theta(\lambda)$}
\end{theorem}
\section{Conclusion and Discussion} \label{sec:3}

Mixup is a popular data augmentation scheme and it has been empirically shown to improve calibration in machine learning. In this paper, we provide a theoretical point of view on how and when Mixup helps the calibration, by studying data generative models. We identify that the calibration improvement induced by Mixup is a high-dimensional phenomenon, and that such reduction in ECE becomes more substantial when the dimension is compared to the number of samples. This suggests that Mixup can be especially helpful for calibration in low sample regime where post-hoc calibration approaches like Platt-scaling are not commonly used. 
We further study the relationship between Mixup and calibration in a semi-supervised setting when there is an abundance of unlabeled data. Using unlabeled data alone can hurt calibration in some settings, while combining Mixup with pseudo-labeling can mitigate this issue.

Our work points to a few promising further directions. Since there are many variants of Mixup \citep{berthelot2019mixmatch,verma2019manifold,roady2020improved,kim2020puzzle}, it would be interesting to study how these extensions of Mixup affect calibration. Another interesting direction is to use the analysis and framework developed in this paper to study the semi-supervised setting where the unlabeled data come from a different domain than the target one. It would be interesting to study how the calibration will change by leveraging the out-of-domain unlabeled data.

\section*{Acknowledgements}
The research of Linjun Zhang is partially supported by  NSF DMS-2015378. The research of Zhun Deng is supported by the Sloan Foundation grants, the NSF grant 1763665, and the Simons Foundation Collaboration on the Theory of Algorithmic Fairness. James Zou is supported by funding from NSF CAREER and the Sloan Fellowship.


\bibliography{example_paper,zhun}
%plainnat
%\bibliographystyle{icml2021}
\bibliographystyle{icml2022}
\newpage

\appendix 
\onecolumn
\noindent\textbf{\Large Appendix}

\section{Technical Details}

\subsection{Proof of Theorem~\ref{thm:helpcalibration}}\label{sec:start}
\begin{theorem}[Restatement of Theorem \ref{thm:helpcalibration}]
Under the settings described in the main paper, if $p/n\to c$  and $\|\theta\|_2<C$ for some universal constants $c, C>0$ (not depending on $n$ and $p$), then for sufficiently large $p$ and $n$, there exist $\alpha,\beta>0$, such that when the distribution $\cD_\lambda$ is chosen as $Beta(\alpha,\beta)$, with high probability,
 $$
ECE(\hat \cC^{mix})<ECE(\hat \cC).
$$
\end{theorem}

\begin{proof}

For the clarity of technical proofs, let us write the true parameter $\theta^*$ as $\theta$, and denote $\hat\theta(0)=\frac{1}{n}\sum_{i=1}^nx_iy_i$. Additionally, since we assume $\sigma$ is known in the main paper, without loss of generality (otherwise we can consider the data as $x_i/\sigma$), we let $\sigma=1$ throughout the proof.

For the mixup estimator, we have %\zhun{confusing here}
\begin{align*}
\hat\theta(\lambda)=&\frac{1}{n^2}\sum_{i, j=1}^n (\lambda x_i+(1-\lambda)x_j)(y_i+(1-\lambda)y_j)\\
=&\frac{1}{n^2}\sum_{i, j=1}^n(\lambda^2 x_i y_i+(1-\lambda)^2 x_j y_j+\lambda(1-\lambda)x_i y_j+\lambda(1-\lambda)x_j y_i)\\
=&[1-2\lambda(1-\lambda)]\frac{1}{n}\sum_{i=1}^nx_iy_i+2\lambda(1-\lambda)\frac{1}{n}\sum_{i=1}^nx_i\cdot\frac{1}{n}\sum_{i=1}^ny_i.
\end{align*}

Then when $\lambda\sim Beta(\alpha,\beta)$, we have $$
\hat\theta^{mix}=\E_\lambda[\hat\theta(\lambda)]=\frac{(\alpha^2+\beta^2)(\alpha+\beta+1)+2\alpha\beta}{(\alpha+\beta)^2(\alpha+\beta+1)}\hat\theta(0)+\frac{2\alpha\beta(\alpha+\beta)}{(\alpha+\beta)^2(\alpha+\beta+1)}\frac{1}{2n}\sum_{i=1}^nx_i\cdot\frac{1}{n}\sum_{i=1}^ny_i%:=(1-t)\hat\theta(0)+t\epsilon.
$$
% \begin{align*}
% &\frac{1}{n}\sum_{i=1}^nx_i\cdot\frac{1}{n}\sum_{i=1}^ny_i\\
%     =&(\frac{1}{n}\sum_{i=1}^nx_i-(2\pi-1)\hat\mu+(2\pi-1)\hat\mu)\cdot(\frac{1}{n}\sum_{i=1}^ny_i-(2\pi-1)+(2\pi-1))\\
%     =&(\frac{1}{n}\sum_{i=1}^nx_i-(2\pi-1)\hat\mu)\cdot(\frac{1}{n}\sum_{i=1}^ny_i-(2\pi-1))+(2\pi-1)^2\hat\mu\\
%     &+(2\pi-1)(\frac{1}{n}\sum_{i=1}^nx_i-(2\pi-1)\hat\mu)+(2\pi-1)\hat\mu\cdot (\frac{1}{n}\sum_{i=1}^ny_i-(2\pi-1)).
% \end{align*}

% As a result, we have $$\hat\mu(\lambda)=[1-2(1-(2\pi-1)^2)\lambda(1-\lambda)]\frac{1}{n}\sum_{i=1}^nx_iy_i+\epsilon,$$
% where $\epsilon=2\lambda(1-\lambda)\cdot((\frac{1}{n}\sum_{i=1}^nx_i-(2\pi-1)\hat\mu)\cdot(\frac{1}{n}\sum_{i=1}^ny_i-(2\pi-1))+(2\pi-1)(\frac{1}{n}\sum_{i=1}^nx_i-(2\pi-1)\hat\mu)+(2\pi-1)\hat\mu\cdot (\frac{1}{n}\sum_{i=1}^ny_i-(2\pi-1)))$.

%Let us denote $t=2(1-(2\pi-1)^2)\lambda(1-\lambda)$, then we write $$
%\hat\mu(\lambda)=(1-t)\hat\mu+\epsilon
%$$

For the ease of presentation, we write $$
\hat\theta(t)=(1-t)\hat\theta(0)+t\epsilon,
$$
where $\epsilon=\frac{1}{n}\sum_{i=1}^nx_i\cdot\frac{1}{n}\sum_{i=1}^ny_i$.

It is easy to see $t\in(0,1]$ when $\alpha\in[0,\infty)$ and $\beta\in(0,\infty)$.

Under our model assumption, we have $\|\frac{1}{n}\sum_{i=1}^nx_i\|=O_p(\sqrt\frac{{p}}{n})$, $|\frac{1}{n}\sum_{i=1}^ny_i|=O_p(\sqrt\frac{{1}}{n})$, and therefore $\|\epsilon\|=O_p(\frac{\sqrt{p}}{n})$.

 Now let us consider the expected calibration error
 $$
ECE=\E_{v=(\hat\theta^{}(t))^\top X} | \E[Y=1\mid \hat f(X)=\frac{1}{e^{-2v}+1}]-\frac{1}{e^{-2v}+1}|.
 $$
 We further expand this quantity as
  \begin{align*}
 \E[Y=1\mid \hat f(X)=\frac{1}{e^{-2v}+1}]=&\E[Y=1\mid \hat\theta(t)^\top X=v]\\
 =&\frac{\Prob(\hat\theta(t)^\top X=v\mid Y=1)}{\Prob(\hat\theta^\top X=v\mid Y=1)+\Prob(\hat\theta(t)^\top X=v\mid Y=-1)}\\
 =&\frac{e^{-\frac{(v-\hat\theta(t)^\top\theta)^2}{2\|\hat\theta(t)\|^2}}}{e^{-\frac{(v-\hat\theta(t)^\top\theta)^2}{2\|\hat\theta(t)\|^2}}+e^{-\frac{(v+\hat\theta(t)^\top\theta)^2}{2\|\hat\theta(t)\|^2}}} \\
 =&\frac{1}{e^{-\frac{2\hat\theta(t)^\top\theta}{\|\hat\theta(t)\|^2}\cdot v}+1}.
%  =&\frac{1}{e^{-\frac{2\hat\theta(0)^\top\bmu}{(1-t)\|\hat\theta(0)\|^2}\cdot v}+1}+O_p(\frac{\sqrt{p}}{n}).
  \end{align*}
  
     Since $\frac{\hat\theta(t)^\top\theta}{\|\hat\theta(t)\|^2}=\frac{((1-t)\hat\theta(0)+t\epsilon)^\top\theta}{\|(1-t)\hat\theta(0)+t\epsilon)\|_2^2}=\frac{1}{1-t}\frac{\hat\theta(0)^\top\theta}{\|\hat\theta(0)\|^2}+O_P(\frac{\sqrt p}{n})$ and $\hat\theta(t)^\top X=(1-t)\hat\theta(0)^\top X+O_P(\frac{\sqrt p}{n})$, we then have \begin{align*}
ECE=&\E_{v=(\hat\theta^{}(t))^\top X} | \E[Y=1\mid \hat f(X)=\frac{1}{e^{-2v}+1}]-\frac{1}{e^{-2v}+1}|\\
=&\E_{v=(\hat\theta^{}(t))^\top X}[|\frac{1}{e^{-\frac{2\hat\theta(t)^\top\theta}{\|\hat\theta(t)\|^2}\cdot v}+1}-\frac{1}{e^{-2v}+1}|]\\
=&\E_{v=(1-t)(\hat\theta^{}(0))^\top X}[|\frac{1}{e^{-\frac{2\hat\theta(0)^\top\theta}{(1-t)\|\hat\theta(0)\|^2}\cdot v}+1}-\frac{1}{e^{-2v}+1}|]+O_P(\frac{\sqrt p}{n})\\
=&\E_{v=(\hat\theta^{}(0))^\top X}[|\frac{1}{e^{-\frac{2\hat\theta(0)^\top\theta}{\|\hat\theta(0)\|^2}\cdot v}+1}-\frac{1}{e^{-2(1-t)v}+1}|]+O_P(\frac{\sqrt p}{n}).
     \end{align*}

Now let us consider the quantity $\frac{\hat\theta(0)^\top\theta}{\|\hat\theta(0)\|^2}$.

Since we have $\hat\theta(0)=\theta+\epsilon_n$ with $\epsilon_n=\frac{1}{n}\sum_{i=1}^n x_iy_i-\theta$, this implies \begin{align*}
    \frac{\hat\theta(0)^\top\theta}{\|\hat\theta(0)\|^2}=\frac{\|\theta\|^2+\epsilon_n^\top\theta}{\|\theta\|^2+\|\epsilon_n\|^2+2\epsilon_n^\top\theta}\sim \frac{\|\theta\|^2+\frac{1}{\sqrt n}\|\theta\|}{\|\theta\|^2+\frac{p}{n}+O_P(\frac{\sqrt p}{n})+\frac{1}{\sqrt n}\|\theta\|},
\end{align*}
where the last equality uses the fact that $\|\epsilon_n\|^2\stackrel{d}{=}\frac{\chi_p^2}{n}=\frac{p+O_P(\sqrt{p})}{n}=\frac{p}{n}+O_P(\frac{\sqrt p}{n})$.

By our assumption, we have $p/n\in(c_1,c_2)$ and $\|\theta\|<C$, implying that there exists a constant $c_0\in(0,1)$, such that with high probability $$
\frac{\hat\theta(0)^\top\theta}{\|\hat\theta(0)\|^2}\le c_0.
$$
Then on the event $\mathcal E=\{\frac{\hat\theta(0)^\top\theta}{\|\hat\theta(0)\|^2}\le c_0\}$, if we choose $t=1-c_0$, we will then have for any $v\in\R$, $$
|\frac{1}{e^{-\frac{2\hat\theta(0)^\top\theta}{\|\hat\theta(0)\|^2}\cdot v}+1}-\frac{1}{e^{-2(1-t)v}+1}|<|\frac{1}{e^{-\frac{2\hat\theta(0)^\top\theta}{\|\hat\theta(0)\|^2}\cdot v}+1}-\frac{1}{e^{-2v}+1}|,
$$
and moreover, the difference is lower bounded by $|\frac{1}{e^{-2v}+1}-\frac{1}{e^{-2c_0v}+1}|$. Use the fact that (since $v$ and $c_0$ does not depend on $n$), $$\E_{v=(\hat\theta^{}(t))^\top X}|\frac{1}{e^{-2v}+1}-\frac{1}{e^{-2c_0v}+1}|=\Omega(1),$$
 we then have the desired result $$
ECE(\hat \cC^{mix})<ECE(\hat \cC).
$$
\end{proof}
\subsection{Proof of Theorem \ref{thm:lowdim}}

\begin{theorem}[Restatement of Theorem~\ref{thm:lowdim}]
In the case where $p/n\to0$ and $\|\theta\|_2<C$ for some universal constant $C>0$, given any constants $\alpha,\beta>0$ (not depending on $n$ and $p$), we have, with high probability, $$
ECE(\hat \cC)<ECE(\hat \cC^{mix}).
$$
\end{theorem}
\begin{proof}

According to the proof in Theorem \ref{thm:helpcalibration}, we have $$
ECE(\hat \cC)=\E_{v=(\hat\theta^{}(0))^\top X}[|\frac{1}{e^{-\frac{2\hat\theta(0)^\top\theta}{\|\hat\theta(0)\|^2}\cdot v}+1}-\frac{1}{e^{-2v}+1}|],
$$
and 
$$
ECE(\hat \cC^{mix})=\E_{v=(\hat\theta^{}(0))^\top X}[|\frac{1}{e^{-\frac{2\hat\theta(0)^\top\theta}{\|\hat\theta(0)\|^2}\cdot v}+1}-\frac{1}{e^{-2(1-t)v}+1}|]+O_P(\frac{\sqrt p}{n}),
$$
where $t\in(0,1)$ is a fixed constant when $\alpha,\beta>0$ are some fixed constants.

When $p/n\to0$, then \begin{align*}
    \frac{\hat\theta(0)^\top\theta}{\|\hat\theta(0)\|^2}\sim \frac{\|\theta\|^2+\frac{1}{\sqrt n}\|\theta\|}{\|\theta\|^2+\frac{p}{n}+O_P(\frac{\sqrt p}{n})+\frac{1}{\sqrt n}\|\theta\|}= 1+O_P(\frac{p}{n})=1+o_P(1).
\end{align*}
Therefore, we have \begin{align*}
ECE(\hat \cC)=&\E_{v=(\hat\theta^{}(0))^\top X}[|\frac{1}{e^{-\frac{2\hat\theta(0)^\top\theta}{\|\hat\theta(0)\|^2}\cdot v}+1}-\frac{1}{e^{-2v}+1}|]\\
=&\E_{v=(\hat\theta^{}(0))^\top X}[|\frac{1}{e^{-{2v}+1}}-\frac{1}{e^{-2v}+1}|]+O_P(\frac{p}{n})\\=&O_P(\frac{p}{n})=o_P(1)
\end{align*}
and
\begin{align*}
ECE(\hat \cC^{mix})=&\E_{v=(\hat\theta^{}(0))^\top X}[|\frac{1}{e^{-\frac{2\hat\theta(0)^\top\theta}{\|\hat\theta(0)\|^2}\cdot v}+1}-\frac{1}{e^{-2(1-t)v}+1}|]+O_P(\frac{\sqrt p}{n})\\
=&\E_{v=(\hat\theta^{}(0))^\top X}[|\frac{1}{e^{-{2v}+1}}-\frac{1}{e^{-2(1-t)v}+1}|]+O_P(\frac{p}{n})
\end{align*}

Since $\E_{v=(\hat\theta^{}(0))^\top X}[|\frac{1}{e^{-{2v}+1}}-\frac{1}{e^{-2(1-t)v}+1}|]=\Omega(1)$ when $t\in(0,1)$ is a fixed constant, we then have the desired result that $$
ECE(\hat \cC)<ECE(\hat \cC^{mix}).
$$

\end{proof}

\subsection{Proof of Theorem~\ref{thm:capacity}}
\begin{theorem}[Restatement of Theorem~\ref{thm:capacity}]
For any constant $c_{\max}>0$, $p/n\to c_{ratio}\in(0,c^{\max})$, when  $\theta$ is sufficiently large (still of a constant level), we have for any $\beta>0$, with high probability, the change of ECE by using Mixup, characterized by
$$
\frac{d}{d\alpha}ECE(\hat \cC^{mix}_{\alpha,\beta})\mid_{\alpha\to0+}
$$ is negative, and monotonically decreasing with respect to $c_{ratio}$. 
\end{theorem}

\begin{proof} Recall that 
$$
ECE(\hat \cC^{mix})=\E_{v=(\hat\theta^{}(0))^\top X}[|\frac{1}{e^{-\frac{2\hat\theta(0)^\top\theta}{\|\hat\theta(0)\|^2}\cdot v}+1}-\frac{1}{e^{-2(1-t)v}+1}|]+O_P(\frac{\sqrt p}{n}).
$$
The case $\alpha=0$ corresponds to the case where $t=0$. Since $|\frac{1}{e^{-\frac{2\hat\theta(0)^\top\theta}{\|\hat\theta(0)\|^2}\cdot v}+1}-\frac{1}{e^{-2(1-t)v}+1}|$ as a function of $v$ is symmetric around $0$, we have that when $t$ is sufficiently small (such that $1-t>\frac{\hat\theta(0)^\top\theta}{\|\hat\theta(0)\|^2}$ with high probability)
\begin{align*}
\E_{v=(\hat\theta^{}(0))^\top X}[|\frac{1}{e^{-\frac{2\hat\theta(0)^\top\theta}{\|\hat\theta(0)\|^2}\cdot v}+1}-\frac{1}{e^{-2(1-t)v}+1}|]=&\E_{v=|(\hat\theta^{}(0))^\top X|}[|\frac{1}{e^{-\frac{2\hat\theta(0)^\top\theta}{\|\hat\theta(0)\|^2}\cdot v}+1}-\frac{1}{e^{-2(1-t)v}+1}|]\\
=&\E_{v=|(\hat\theta^{}(0))^\top X|}[\frac{1}{e^{-\frac{2\hat\theta(0)^\top\theta}{\|\hat\theta(0)\|^2}\cdot v}+1}-\frac{1}{e^{-2(1-t)v}+1}].
\end{align*}
Then let us take the derivative with respect to $t$, we get \begin{align*}
\E_{v=|(\hat\theta^{}(0))^\top X|}[-\frac{e^{-2(1-t)v}\cdot 2v}{(e^{-2(1-t)v}+1)^2}].
\end{align*}
Therefore, the derivative evaluated at $t=0$ equals to 
\begin{align*}
\E_{v=|(\hat\theta^{}(0))^\top X|}[-\frac{e^{-2v}\cdot 2v}{(e^{-2v}+1)^2}],
\end{align*}
which is negative.

We then only need to show it is monotonically decreasing in the rest of this proof.

Again, by symmetry, we only need to consider the distribution of $X$ as $N(\theta, I)$. We have $\hat\theta^{}(0)^\top X\sim N(\hat\theta^{}(0)^\top\theta,\|\hat\theta^{}(0)^\top\|^2).$
Since we have $\hat\theta(0)=\theta+\epsilon_n$ with $\epsilon_n=\frac{1}{n}\sum_{i=1}^n x_iy_i-\theta$, we then have $\hat\theta^{}(0)^\top\theta=\|\theta\|^2+O_P(\frac{1}{\sqrt n})$, and $\|\hat\theta^{}(0)\|^2=\|\theta\|^2+p/n+O(\frac{1 }{\sqrt n})$. In order to show $\E_{v=|(\hat\theta^{}(0))^\top X|}[-\frac{e^{-2v}\cdot 2v}{(e^{-2v}+1)^2}]$ is monotonically decreasing, it's sufficient to show that $$
\E_{Z\sim N(0,1)}\frac{\mu+\sigma Z}{e^{-2(\mu+\sigma Z)}+e^{2(\mu+\sigma Z)}+2}
$$
is monotonically increasing in $\sigma\in(0,c_{\max})$ when $\mu$ is sufficiently large. Let us then take derivative with respect $\sigma$, we have \begin{align*}
\E_{Z\sim N(0,1)}\frac{Z(e^{-2(\mu+\sigma Z)}+e^{2(\mu+\sigma Z)}+2)-(\mu+\sigma Z)\cdot 2Z\cdot (e^{2(\mu+\sigma Z)}-e^{-2(\mu+\sigma Z)})}{(e^{-2(\mu+\sigma Z)}+e^{2(\mu+\sigma Z)}+2)^2}.
\end{align*}

It suffices to show when $\mu$ is sufficiently large, this term is positive.

In fact, when $\mu$ is sufficiently large, it suffices to look at dominating term $
\E_{Z\sim N(0,1)}[\frac{-2\mu\cdot Z\cdot e^{2(\mu+\sigma Z)}}{e^{4(\mu+\sigma Z)}}],
$
for which we have
$$
\E_{Z\sim N(0,1)}[\frac{-2\mu\cdot Z\cdot e^{2(\mu+\sigma Z)}}{e^{4(\mu+\sigma Z)}}]=-2\mu\cdot \E_{Z\sim N(0,1)}[Z\cdot e^{-2(\mu+\sigma Z)}]>0.
$$
We complete the proof.

\end{proof}
\subsection{Proof of Theorem~\ref{thm:out}}
\begin{theorem}[Restatement of Theorem \ref{thm:out}]
Let us consider the ECE evaluated on the out-of-domain Gaussian model with mean parameter $\theta'$, that is, $\bP(y=1)=\bP(y=-1)=1/2$, and
$$ x\mid y\sim \cN(y\cdot\theta',\sigma^2 I), \text{ for } i=1,2,...,n, %+\frac{1}{2}N(-\mu,I).
$$If we have $(\theta'-\theta^*)^\top\theta^*\le p/(2n)$, then when $p$ and $n$ are sufficiently large,  there exist $\alpha,\beta>0$, such that when the distribution $\cD_\lambda$ is chosen as $Beta(\alpha,\beta)$, with high probability,
$$ECE(\hat \cC^{mix};\theta',\sigma)<ECE(\hat \cC;\theta',\sigma),$$
%\zhun{fill in later}
\end{theorem}
\begin{proof} When the distribution has mean $\theta'$, using the same analysis before, we obtain
  \begin{align*}
 \E[Y=1\mid \hat f(X)=\frac{1}{e^{-2v}+1}]=&\E[Y=1\mid \hat\theta(t)^\top X=v]\\
 =&\frac{\Prob(\hat\theta(t)^\top X=v\mid Y=1)}{\Prob(\hat\theta^\top X=v\mid Y=1)+\Prob(\hat\theta(t)^\top X=v\mid Y=-1)}\\
 =&\frac{1}{e^{-\frac{2\hat\theta(t)^\top\theta'}{\|\hat\theta(t)\|^2}\cdot v}+1}.
%  =&\frac{1}{e^{-\frac{2\hat\theta(0)^\top\bmu}{(1-t)\|\hat\theta(0)\|^2}\cdot v}+1}+O_p(\frac{\sqrt{p}}{n}).
  \end{align*}
  
 Again, following the same analysis above, it's suffices to show that with high probability, $$
  \frac{\hat\theta(0)^\top\theta'}{\|\hat\theta(0)\|^2}<1.
  $$
  Again, we  use $\hat\theta(0)=\theta+\epsilon_n$ with $\epsilon_n=\frac{1}{n}\sum_{i=1}^n x_iy_i-\theta$, we then have 
  $$
  \frac{\hat\theta(0)^\top\theta'}{\|\hat\theta(0)\|^2}=\frac{\theta^\top\theta'}{\|\theta\|^2+p/n}+O_P(\frac{1}{\sqrt n})=\frac{\theta^\top(\theta'-\theta)+\|\theta\|^2}{\|\theta\|^2+p/n}+O_P(\frac{1}{\sqrt n}).
  $$
  If we have $(\theta'-\theta^*)^\top\theta^*\le p/(2n)$, then when $p$ and $n$ are sufficiently large,  there exist $\alpha,\beta>0$, such that when the distribution $\cD_\lambda$ is chosen as $Beta(\alpha,\beta)$, with high probability, $$
  \frac{\hat\theta(0)^\top\theta'}{\|\hat\theta(0)\|^2}<1.
  $$
  and therefore
$$ECE(\hat \cC^{mix};\theta',\sigma)<ECE(\hat \cC;\theta',\sigma),$$
\end{proof}
  \subsection{Proof of Theorem~\ref{thm:helpcalibration2}}
\begin{theorem}[Restatement of Theorem \ref{thm:helpcalibration2}]
Under the settings described above with Assumption~\ref{assump}, if $p/n\to c$, $g, \hat h$ is $L$-Lipschitz, and $\|\theta\|_2<C$ for some universal constants $c, L, C>0$ (not depending on $n$ and $p$), then for sufficiently large $p$ and $n$, there exist $\alpha,\beta>0$ for the Mixup distribution $\cD_\lambda = Beta(\alpha,\beta)$, such that, with high probability,
 $$
\ECE(\hat\cC^{mix})<\ECE(\hat \cC).
$$
\end{theorem}
\begin{proof}

Let us first recall Assumption~\ref{assump}:
\begin{assumption}[Assumption~\ref{assump} in the main text]
For any given $v\in\R^p$, $k\in\{-1,1\}$, there exists a $\theta^*\in\R^p$, such that given $y=k$, the probability density function of $R_1=v^\top \hat h(x)$ and $R_2=v^\top h(x)=v^\top z\sim N(k\cdot b^\top\theta^*, \|v\|^2)$ satisfies that $p_{R_1}(u)=p_{R_2}(u)\cdot (1+\delta_u)$ for all $u\in\R$ where $\delta_u$ satisfies $\E_{R_1}[|\delta_u|^2]=o(1)$  when $n\to\infty$. 
\end{assumption}

Using the similar analysis from above, we have
  \begin{align*}
 \E[Y=1\mid \hat f(X)=\frac{1}{e^{-2v}+1}]=&\E[Y=1\mid \hat\theta(t)^\top \hat h(X)=v]\\
 =&\frac{\Prob(\hat\theta(t)^\top \hat h(X)=v\mid Y=1)}{\Prob(\hat\theta^\top \hat h(X)=v\mid Y=1)+\Prob(\hat\theta(t)^\top \hat h(X)=v\mid Y=-1)}\\
  =&\frac{e^{-\frac{(v-\hat\theta(t)^\top\theta)^2}{2\|\hat\theta(t)\|^2}}}{e^{-\frac{(v-\hat\theta(t)^\top\theta)^2}{2\|\hat\theta(t)\|^2}}+e^{-\frac{(v+\hat\theta(t)^\top\theta)^2}{2\|\hat\theta(t)\|^2}}} (1+\delta_v)\\
 =&\frac{1}{e^{-\frac{2\hat\theta(t)^\top\theta}{\|\hat\theta(t)\|^2}\cdot v}+1}(1+\delta_v).
%  =&\frac{1}{e^{-\frac{2\hat\theta(0)^\top\bmu}{(1-t)\|\hat\theta(0)\|^2}\cdot v}+1}+O_p(\frac{\sqrt{p}}{n}).
  \end{align*}
  Then by Assumption~\ref{assump} and use the fact that $|\frac{1}{e^{-\frac{2\hat\theta(t)^\top\theta}{\|\hat\theta(t)\|^2}\cdot v}+1}|\le 1$, we have the expected calibration error as 
  \begin{align*}
      ECE=&\E_{v=(\hat\theta^{}(t))^\top X} | \E[Y=1\mid \hat f(X)=\frac{1}{e^{-2v}+1}]-\frac{1}{e^{-2v}+1}|\\
=&\E_{v=(\hat\theta^{}(t))^\top X}[|\frac{1}{e^{-\frac{2\hat\theta(t)^\top\theta}{\|\hat\theta(t)\|^2}\cdot v}+1}-\frac{1}{e^{-2v}+1}|]+o(1)
  \end{align*}
  
 Then, by Assumption~\ref{assump}, we have $\E[|v^\top (\hat h(x)-z)|]\to 0$, where $z\sim\frac{1}{2}N(-\theta,I)+\frac{1}{2}N(\theta,I)$, which implies $\|\frac{1}{n}\sum_{i=1}^n\hat h(x_i)-\frac{1}{n}\sum_{i=1}^n z_i\|=o(\sqrt{p})$. Since $\|\frac{1}{n}\sum_{i=1}^n z_i\|=O_P(\frac{\sqrt p}{n})$, and $|\frac{1}{n}\sum_{i=1}^ny_i|=O_p(\sqrt\frac{{1}}{n})$, we have 
 $$
\| \frac{1}{n}\sum_{i=1}^n\hat h(x_i)\cdot\frac{1}{n}\sum_{i=1}^ny_i\|= o_p(\sqrt\frac{{p}}{n})=o_p(1).
 $$
 As a result, using the same analysis as those in Section~\ref{sec:start}, we have 
$$\hat{\theta}^{mix}=\sum_{i,j=1}^n\bE_{\lambda\sim\cD_\lambda} (\lambda \hat h(x_i)+(1-\lambda)\hat h(x_j))\cdot \tilde{y}_{i,j}(\lambda)/n^2=(1-t)\hat\theta(0)+o_p(1),
$$  
 and therefore $\frac{\hat\theta(t)^\top\theta}{\|\hat\theta(t)\|^2}=\frac{((1-t)\hat\theta(0)+t\epsilon)^\top\theta}{\|(1-t)\hat\theta(0)+t\epsilon)\|_2^2}=\frac{1}{1-t}\frac{\hat\theta(0)^\top\theta}{\|\hat\theta(0)\|^2}+o_P(1)$ and $\hat\theta(t)^\top X=(1-t)\hat\theta(0)^\top X+o_P(1)$, we then have \begin{align*}
ECE=&\E_{v=(\hat\theta^{}(t))^\top X} | \E[Y=1\mid \hat f(X)=\frac{1}{e^{-2v}+1}]-\frac{1}{e^{-2v}+1}|\\
=&\E_{v=(\hat\theta^{}(t))^\top X}[|\frac{1}{e^{-\frac{2\hat\theta(t)^\top\theta}{\|\hat\theta(t)\|^2}\cdot v}+1}-\frac{1}{e^{-2v}+1}|]\\
=&\E_{v=(1-t)(\hat\theta^{}(0))^\top X}[|\frac{1}{e^{-\frac{2\hat\theta(0)^\top\theta}{(1-t)\|\hat\theta(0)\|^2}\cdot v}+1}-\frac{1}{e^{-2v}+1}|]+o_P(1)\\
=&\E_{v=(\hat\theta^{}(0))^\top X}[|\frac{1}{e^{-\frac{2\hat\theta(0)^\top\theta}{\|\hat\theta(0)\|^2}\cdot v}+1}-\frac{1}{e^{-2(1-t)v}+1}|]+o_P(1).
     \end{align*}
     Again, it boils down to studying the quantity $\frac{\hat\theta(0)^\top\theta}{\|\hat\theta(0)\|^2}$, and it suffices to show this quantity is smaller than 1 with high probability. 
     
     Using the same analysis above, recall that we have $\E[|v^\top (\hat h(x)-z)|]\to 0$ for any $\bm v$ with $\|\bm v\|<C$. Let $\tilde\theta=\E_x[\hat h(x)]$ and plugging in $v=\tilde\theta$, we then obtain $|\tilde\theta^\top(\tilde\theta-\theta)|=o(1)$. Also, plugging in $v=\theta$, we obtain $|\theta^\top(\tilde\theta-\theta)|=o(1)$. Combining these two pieces, we obtain $$\|\tilde\theta-\theta\|=o(1).$$
     
     As a result, we have $$
     \|\hat\theta(0)\|^2=\|\frac{1}{2n}\sum_{i=1}^n\hat y_ih(x_i)\|\ge\|\frac{1}{2n}\sum_{i=1}^ny_i\hat h(x_i)-\E[\hat h(x)]\|+\|\E[\hat h(x)]\|=\Omega_P(\sqrt{\frac{p}{n}})+\|\theta\|+o(1),
     $$
     where the term $\Omega_P(\sqrt{\frac{p}{n}})$ is derived as follows.
     
     First of all, we write
     $$
     \|\frac{1}{2n}\sum_{i=1}^n\hat y_ih(x_i)-\E[\hat h(x)]\|^2=\sum_{j=1}^p(\frac{1}{2n}\sum_{i=1}^n\hat y_ih_j(x_i)-\E[\hat h_j(x)])^2.
     $$
     For each coordinate, we have $Var(z_j)=1$ and \begin{equation}\label{eq:var}
     |Var(\hat h_j(x_i))-1|=o(1).
     \end{equation}
     %Additionally, since $\hat h$ and $g$ are all Lipshitz, by the fact that a Lipchitz transformation of a Gaussian random variable is subgaussian, we have 
     Additionally, since $\hat h(x)$ is sub-gaussian, combining with the inequality~\eqref{eq:var}, we have that $\hat h_j(x_i)$ has subgaussian norm lower bounded by some constant, which implies $$
      \|\frac{1}{n}\sum_{i=1}^n\hat y_ih(x_i)-\E[\hat h(x)]\|^2=\sum_{j=1}^p(\frac{1}{n}\sum_{i=1}^n\hat y_ih_j(x_i)-\E[\hat h_j(x)])^2=\Omega_P(\frac{p}{n}).
     $$
     
     Additionally, we have $$
     \|\hat\theta(0)^\top\theta\|\le \|\theta\|+o(1).
     $$
     Therefore, we have with high probability, $$\frac{\hat\theta(0)^\top\theta}{\|\hat\theta(0)\|^2}<1.$$
     
     \paragraph{Verification of the unknown $\sigma$ case} When $\sigma$ is unknown, we estimate $\sigma$ by $\hat\sigma=\sqrt{\|\sum_{i=1}^n (x_i-y_i\hat\theta)\|^2/pn}$. It's easy to see $|\hat\sigma-\sigma|=O_P(1/\sqrt{n})$. We then let $\hat h(x)=x/\hat\sigma$, and verify for any $v\in\R^p$ with $\|v\|\le C$, 
      $R_1=v^\top \hat h(x)=v^\top x/\hat\sigma$ and $R_2=v^\top x/\sigma$ satisfies that $p_{R_1}(u)=p_{R_2}(u)\cdot (1+\delta_u)$ for all $u\in\R$ where $\delta_u$ satisfies $\E_{R_1}[|\delta_u|^2]=o(1)$  when $n\to\infty$. We have $\hat h$ and $g$ are all Lipschitz with constant $2\sigma$.
      
      When $y$=1, we have $$
p_{R_1}(u)=\frac{1}{\sqrt{2\pi}\sigma/\hat\sigma}\exp\{-\frac{(u-v^\top\theta)^2}{2\sigma^2/\hat\sigma^2}\},      p_{R_2}(u)=\frac{1}{\sqrt{2\pi}}\exp\{-\frac{(u-v^\top\theta)^2}{2}\}
      $$
      Denote $g(a)=\frac{1}{\sqrt{2\pi}a}\exp\{-\frac{(u-v^\top\theta)^2}{2a^2}\}$, we then have $g'(a)=-\frac{1}{\sqrt{2\pi}a^2}\exp\{-\frac{(u-v^\top\theta)^2}{2a^2}\}+\frac{1}{\sqrt{2\pi}a}\exp\{-\frac{(u-v^\top\theta)^2}{2a^2}\}\cdot\frac{(u-v^\top\theta)^2}{a^3}$ and therefore $g'(1)=-\frac{1}{\sqrt{2\pi}}\exp\{-\frac{(u-v^\top\theta)^2}{2}\}+\frac{1}{\sqrt{2\pi}}\exp\{-\frac{(u-v^\top\theta)^2}{2}\}\cdot{(u-v^\top\theta)^2}$. 
      
      We then have $$
      \delta_u=\frac{p_{R_1}(u)-p_{R_2}(u)}{p_{R_2}(u)}=[(u-v^\top\theta)^2-1]\cdot (\frac{\sigma^2}{\hat\sigma^2}-1)=O_P(\frac{(u-v^\top\theta)^2-1}{\sqrt n}).
      $$
      Since $\E_{u\sim R_1}[((u-v^\top\theta)^2-1)^2]=O(1)$, we have $\E_{R_1}[|\delta_u|^2]=o(1)$ when $n\to\infty.$
  \end{proof}

  \subsection{Proof of Theorem~\ref{thm:semi1}}
  
  \begin{theorem}[Restatement of Theorem \ref{thm:semi1}]
Suppose $C_1\sqrt{p/n_l}\le\|\theta\|\le C_2\sqrt{p/n_l}$ for some universal constant $C_1<1/2$ and $C_2>2$, when $p/n_l$, $\|\theta\|$, $n_u$ are sufficiently large, we have with high probability, $$
ECE(\hat \cC_{final})<ECE(\hat \cC_{init}).
$$
\end{theorem}
\begin{proof}
According to the proof in the above section, we have $$
ECE(\hat C_{final})=\E_{v=(\hat\theta_{final})^\top X}[|\frac{1}{e^{-\frac{2\hat\theta_{final}^\top\theta}{\|\hat\theta_{final}\|^2}\cdot v}+1}-\frac{1}{e^{-2v}+1}|],
$$
and 
$$
ECE(\hat C_{init})=\E_{v=(\hat\theta_{init})^\top X}[|\frac{1}{e^{-\frac{2\hat\theta_{init}^\top\theta}{\|\hat\theta_{init}\|^2}\cdot v}+1}-\frac{1}{e^{-2v}+1}|].
$$

 For the initial estimator, we  use $\hat\theta_{init}=\hat\theta(0)=\theta+\epsilon_n$ with $\epsilon_n=\frac{1}{n}\sum_{i=1}^n x_iy_i-\theta$, we then have 
  $$
  \frac{\hat\theta(0)^\top\theta}{\|\hat\theta(0)\|^2}=\frac{\|\theta\|^2}{\|\theta\|^2+p/n_l}+O_P(\frac{1}{\sqrt n_l}).
  $$
  When $C_1\sqrt{p/n_l}\le\|\theta\|\le C_2\sqrt{p/n_l}$, we have  $$
   \frac{\hat\theta_{init}^\top\theta}{\|\hat\theta_{init}\|^2}=\frac{\hat\theta(0)^\top\theta}{\|\hat\theta(0)\|^2}\le\frac{C_2^2}{C_2^2+1}.
  $$
  
In the case where we combine the unlabeled data, we follow the similar analysis of \citet{carmon2019unlabeled} to study the property of $y_i^u$. Let $b_i$ be the indicator that the $i$-th pseudo-label is incorrect, so that $x_i^u\sim N((1-2b_i)y_i^u\theta,I):=(1-2b_i)y_i^u\theta+\epsilon_i^u$. Then we can write $$\hat\theta_{final}=\gamma\theta+\tilde\delta,$$
where $\gamma=\frac{1}{n_u}\sum_{i=1}^{n_u}(1-2b_i)$ and $\tilde\delta=\frac{1}{n_u}\sum_{i=1}^{n_u} \epsilon_i^uy_i^u$.

We then derive concentration bounds for $\|\tilde\delta\|^2$ and $\theta^\top\tilde\delta$. Recall that $y^u_i=sgn(\hat\theta_{init}^\top x_i^u)$, we choose a coordinate system such that the first coordinate is in the direction of $\hat\theta_{init}$,
and let $v^{(i)}$ denote the $i$-th entry of vector $v$ in this coordinate system. Then $y_i^u=sgn(x_i^{u(1)})=sgn(\theta^{\top}\hat\theta_{init}+\epsilon_i^{u(1)})$. 

Under this coordinate system, for $j\ge 2$, we have $\epsilon_i^{u(j)}$ are independent with $y_i^{u}$ and therefore $\epsilon_i^{u(j)}y_i^u\sim N(0,1)$ for all $j\ge 2$. For the first coordinate, since $\theta^{\top}\hat\theta_{init}=\Omega_P(1)$, we have $|\E[\epsilon_i^{u(1)}y_i^u]|=|\E[\epsilon_i^{u(1)}sgn(\theta^{\top}\hat\theta_{init}+\epsilon_i^{u(1)})]|=\Omega_P(1)$.

Then we have $$
\sum_{j=1}^p (\frac{1}{n_u}\sum_{i=1}^{n_u}\epsilon_i^{u(j)}y_i^u)^2=(\frac{1}{n_u}\sum_{i=1}^{n_u}\epsilon_i^{u(1)}y_i^u)^2+\sum_{j=2}^p (\frac{1}{n_u}\sum_{i=1}^{n_u}\epsilon_i^{u(j)}y_i^u)^2=\Omega_P(1)+\frac{p+O_p(\sqrt p)}{n_u}.
$$
The same analysis also yields $$
|\tilde\delta^\top\theta|=|\frac{1}{n_u}\sum_{i=1}^{n_u}\epsilon_i^{u(1)}y_i^u\theta_1+\sum_{j=2}^p\frac{1}{n_u}\sum_{i=1}^{n_u}\epsilon_i^{u(j)}y_i^u\theta_j|= \Omega_P(|\theta_1|)+O_P(\|\theta\|_2/\sqrt{n}).
$$

Moreover, the proportion of misclassified samples converge the misclassification error produced by $\hat C_{init}$: $$\gamma=\frac{1}{n_u}\sum_{i=1}^{n_u}(1-2b_i)\to 1-\exp(-c\|\theta\|^2)+O_P(\frac{1}{n}).$$

This implies \begin{align*}
    \frac{\hat\theta_{final}^\top\theta}{\|\hat\theta_{final}\|^2}=\frac{\gamma\|\theta\|^2+\tilde\delta^\top\theta}{\gamma^2\|\theta\|^2+\|\tilde\delta\|^2+2\tilde\delta^\top\theta}\sim \frac{\gamma\|\theta\|^2+\Omega_P(|\theta_1|)}{\gamma^2\|\theta\|^2+\Omega_P(\frac{n_u+p}{n_u})+\Omega_P(|\theta_1|)}.
\end{align*}

  When $\|\theta\|$ is sufficiently large (which implies $p/n_l$ is sufficiently large), we have $$
\frac{C_2^2}{C_2^2+1}<   \frac{\hat\theta_{final}^\top\theta}{\|\hat\theta_{final}\|^2}<1,
  $$
  implying $$
   \frac{\hat\theta_{init}^\top\theta}{\|\hat\theta_{init}\|^2}<\frac{\hat\theta_{final}^\top\theta}{\|\hat\theta_{final}\|^2}<1,
  $$
  and therefore $$
  ECE(\hat \cC_{final})<ECE(\hat \cC_{init}).
$$
\end{proof}
\subsection{Proof of Theorem~\ref{thm:semi2}}

\begin{theorem}[Restatement of Theorem \ref{thm:semi2} ]
If $\|\theta\|<C$ for some constant $C>2$, given fixed $p$ and let $n_l$ and $n_{u}\to\infty$ with $p$ fixed, then with high probability probability, $$
ECE(\hat \cC_{init})<ECE(\hat \cC_{final}).
$$
\end{theorem}
\begin{proof}
Using the same analysis as in the above section, we have $$
ECE(\hat \cC_{final})=\E_{v=(\hat\theta_{final})^\top X}[|\frac{1}{e^{-\frac{2\hat\theta_{final}^\top\theta}{\|\hat\theta_{final}\|^2}\cdot v}+1}-\frac{1}{e^{-2v}+1}|],
$$
and 
$$
ECE(\hat \cC_{init})=\E_{v=(\hat\theta_{init})^\top X}[|\frac{1}{e^{-\frac{2\hat\theta_{init}^\top\theta}{\|\hat\theta_{init}\|^2}\cdot v}+1}-\frac{1}{e^{-2v}+1}|].
$$

 For the initial estimator, we  use $\hat\theta_{init}=\hat\theta(0)=\theta+\epsilon_n$ with $\epsilon_n=\frac{1}{n}\sum_{i=1}^n x_iy_i-\theta$, we then have 
  $$
  \frac{\hat\theta(0)^\top\theta}{\|\hat\theta(0)\|^2}=\frac{\|\theta\|^2}{\|\theta\|^2+p/n_l}+O_P(\frac{1}{\sqrt n_l}).
  $$
  When  $\|\theta\|<C$ for some constant $C>2$, given  $p$ fixed and $n_l
  \to\infty$, we have  $$
   \frac{\hat\theta_{init}^\top\theta}{\|\hat\theta_{init}\|^2}=\frac{\hat\theta(0)^\top\theta}{\|\hat\theta(0)\|^2}\to 1.
  $$
  For the semi-supervised classifier, when $\|\theta\|<C$ and $n_u\to\infty$, we have 
  \begin{align*}
    \frac{\hat\theta_{final}^\top\theta}{\|\hat\theta_{final}\|^2}=\frac{\gamma\|\theta\|^2+\tilde\delta^\top\theta}{\gamma^2\|\theta\|^2+\|\tilde\delta\|^2+2\tilde\delta^\top\theta}\sim \frac{\gamma\|\theta\|^2+\Omega_P(|\theta_1|)}{\gamma^2\|\theta\|^2+\Omega_P(\frac{n_u+p}{n_u})+\Omega_P(|\theta_1|)}<1.
\end{align*}
As a result, $$
   \frac{\hat\theta_{final}^\top\theta}{\|\hat\theta_{final}\|^2}<\frac{\hat\theta_{init}^\top\theta}{\|\hat\theta_{init}\|^2}\le 1,
  $$
  and therefore $$
ECE(\hat \cC_{init})<  ECE(\hat \cC_{final}).
$$

\end{proof}
\subsection{Proof of Theorem~\ref{thm:helpcalibrationsemi}}

\begin{theorem}[Restatement of Theorem \ref{thm:helpcalibrationsemi} ]
Under the setup described above, and denote the ECE of $\hat \cC_{final}$ and $\hat \cC_{mix,final}$ by $ECE(\hat \cC_{final})$ and  $ECE(\hat \cC_{mix,final})$ respectively. If $C_1<\|\theta\|_2<C_2$ for some universal constants $C_1,C_2$ (not depending on $n$ and $p$), then for sufficiently large $p$ and $n_l, n_{u}$, there exists $\alpha,\beta>0$, such that when the Mixup distribution $\lambda\sim Beta(\alpha,\beta)$, with high probability, we have
 $$
ECE(\hat \cC_{mix,final})<ECE(\hat \cC_{final}).
$$
\end{theorem}
\begin{proof}
Using the same analysis as those in Section~\ref{sec:start}, we have 

Using the similar analysis from above, we have
  \begin{align*}
 \E[Y=1\mid \hat f(X)=\frac{1}{e^{-2v}+1}]=&\E[Y=1\mid \hat\theta^\top X=v]\\
 =&\frac{\Prob(\hat\theta^\top X=v\mid Y=1)}{\Prob(\hat\theta^\top X=v\mid Y=1)+\Prob(\hat\theta^\top X=v\mid Y=-1)}\\
%  =&\frac{e^{-\frac{(v-\hat\theta(t)^\top\theta)^2}{2\|\hat\theta(t)\|^2}}}{e^{-\frac{(v-\hat\theta(t)^\top\theta)^2}{2\|\hat\theta(t)\|^2}}+e^{-\frac{(v+\hat\theta(t)^\top\theta)^2}{2\|\hat\theta(t)\|^2}}} (1+\delta_v)\\
 =&\frac{1}{e^{-\frac{2\hat\theta^\top\theta}{\|\hat\theta\|^2}\cdot v}+1}.
%  =&\frac{1}{e^{-\frac{2\hat\theta(0)^\top\bmu}{(1-t)\|\hat\theta(0)\|^2}\cdot v}+1}+O_p(\frac{\sqrt{p}}{n}).
  \end{align*}
  Then  we have the expected calibration error as 
  \begin{align*}
      ECE=&\E_{v=(\hat\theta^{}(t))^\top X} | \E[Y=1\mid \hat f(X)=\frac{1}{e^{-2v}+1}]-\frac{1}{e^{-2v}+1}|\\
=&\E_{v=(\hat\theta^{}(t))^\top X}[|\frac{1}{e^{-\frac{2\hat\theta(t)^\top\theta}{\|\hat\theta(t)\|^2}\cdot v}+1}-\frac{1}{e^{-2v}+1}|]+o(1)
  \end{align*}
  
 Then, since $\|\frac{1}{n_u}\sum_{i=1}^{n_u} x_i^u\|=O_P(\frac{\sqrt p}{n_u})$, and $|\frac{1}{n}\sum_{i=1}^ny^u_i|=O_p(\sqrt\frac{{1}}{n_u})$, we have 
 $$
\| \frac{1}{n_u}\sum_{i=1}^{n_u} x_i^u\cdot\frac{1}{n_u}\sum_{i=1}^{n_u}y_i^u\|= O_p(\frac{\sqrt{p}}{n_u})=o_p(1).
 $$
 As a result, using the same analysis as those in Section~\ref{sec:start}, and denote $\hat\theta(0)=\hat\theta_{final}$ we have 
$$\hat{\theta}_{final, mix}=\sum_{i,j=1}^n\bE_{\lambda\sim\cD_\lambda} (\lambda \hat h(x_i)+(1-\lambda)\hat h(x_j))\cdot \tilde{y}_{i,j}(\lambda)/n^2=(1-t)\hat\theta(0)+o_p(1),
$$  
 and therefore $\frac{\hat\theta(t)^\top\theta}{\|\hat\theta(t)\|^2}=\frac{1}{1-t}\frac{\hat\theta(0)^\top\theta}{\|\hat\theta(0)\|^2}+o_P(1)$ and $\hat\theta(t)^\top X=(1-t)\hat\theta(0)^\top X+o_P(1)$, we then have \begin{align*}
ECE=&\E_{v=(\hat\theta^{}(t))^\top X} | \E[Y=1\mid \hat f(X)=\frac{1}{e^{-2v}+1}]-\frac{1}{e^{-2v}+1}|\\
=&\E_{v=(\hat\theta^{}(t))^\top X}[|\frac{1}{e^{-\frac{2\hat\theta(t)^\top\theta}{\|\hat\theta(t)\|^2}\cdot v}+1}-\frac{1}{e^{-2v}+1}|]\\
=&\E_{v=(1-t)(\hat\theta^{}(0))^\top X}[|\frac{1}{e^{-\frac{2\hat\theta(0)^\top\theta}{(1-t)\|\hat\theta(0)\|^2}\cdot v}+1}-\frac{1}{e^{-2v}+1}|]+o_P(1)\\
=&\E_{v=(\hat\theta^{}(0))^\top X}[|\frac{1}{e^{-\frac{2\hat\theta(0)^\top\theta}{\|\hat\theta(0)\|^2}\cdot v}+1}-\frac{1}{e^{-2(1-t)v}+1}|]+o_P(1).
     \end{align*}
     Again, it boils down to studying the quantity $\frac{\hat\theta(0)^\top\theta}{\|\hat\theta(0)\|^2}$, and it suffices to show this quantity is smaller than 1 with high probability. 
     
     To see this, when $C_1<\|\theta\|_2<C_2$ and $n_u,p$ sufficiently large, we have with high probability, \begin{align*}
    \frac{\hat\theta_{final}^\top\theta}{\|\hat\theta_{final}\|^2}=\frac{\gamma\|\theta\|^2+\tilde\delta^\top\theta}{\gamma^2\|\theta\|^2+\|\tilde\delta\|^2+2\tilde\delta^\top\theta}\sim \frac{\gamma\|\theta\|^2+\Omega_P(|\theta_1|)}{\gamma^2\|\theta\|^2+\Omega_P(\frac{n_u+p}{n_u})+\Omega_P(|\theta_1|)}<1.
\end{align*}
Therefore, there exists $\alpha,\beta>0$, such that when the Mixup distribution $\lambda\sim Beta(\alpha,\beta)$, with high probability, we have
 $$
ECE(\hat \cC_{mix,final})<ECE(\hat \cC_{final}).
$$
\end{proof}

\subsection{Proof of Theorem~\ref{thm:helpcalibrationMCE}}
%===============================================
 \begin{theorem}[Restatement of Theorem \ref{thm:helpcalibrationMCE}]
	Under the settings described in Theorem \ref{thm:helpcalibration}, there exists $c_2>c_1>0$, when $p/n\in(c_1,c_2)$ and $\|\theta\|_2<C$ for some universal constants $ C>0$ (not depending on $n$ and $p$), then for sufficiently large $p$ and $n$, there exist $\alpha,\beta>0$, such that when the distribution $\cD_\lambda$ is chosen as $Beta(\alpha,\beta)$, with high probability,
	$$
	MCE(\hat{\mathcal{C}}^{mix})<MCE(\hat{\mathcal{C}}).
	$$
\end{theorem}

\begin{proof}

% \begin{align*}
	% &\frac{1}{n}\sum_{i=1}^nx_i\cdot\frac{1}{n}\sum_{i=1}^ny_i\\
	%     =&(\frac{1}{n}\sum_{i=1}^nx_i-(2\pi-1)\hat\mu+(2\pi-1)\hat\mu)\cdot(\frac{1}{n}\sum_{i=1}^ny_i-(2\pi-1)+(2\pi-1))\\
	%     =&(\frac{1}{n}\sum_{i=1}^nx_i-(2\pi-1)\hat\mu)\cdot(\frac{1}{n}\sum_{i=1}^ny_i-(2\pi-1))+(2\pi-1)^2\hat\mu\\
	%     &+(2\pi-1)(\frac{1}{n}\sum_{i=1}^nx_i-(2\pi-1)\hat\mu)+(2\pi-1)\hat\mu\cdot (\frac{1}{n}\sum_{i=1}^ny_i-(2\pi-1)).
	% \end{align*}

% As a result, we have $$\hat\mu(\lambda)=[1-2(1-(2\pi-1)^2)\lambda(1-\lambda)]\frac{1}{n}\sum_{i=1}^nx_iy_i+\epsilon,$$
% where $\epsilon=2\lambda(1-\lambda)\cdot((\frac{1}{n}\sum_{i=1}^nx_i-(2\pi-1)\hat\mu)\cdot(\frac{1}{n}\sum_{i=1}^ny_i-(2\pi-1))+(2\pi-1)(\frac{1}{n}\sum_{i=1}^nx_i-(2\pi-1)\hat\mu)+(2\pi-1)\hat\mu\cdot (\frac{1}{n}\sum_{i=1}^ny_i-(2\pi-1)))$.

%Let us denote $t=2(1-(2\pi-1)^2)\lambda(1-\lambda)$, then we write $$
%\hat\mu(\lambda)=(1-t)\hat\mu+\epsilon
%$$

Following the calculation of Theorem \ref{thm:helpcalibration}, we consider the maximum calibration error,
$$
MCE=\max_{v} | \E[Y=1\mid \hat f(X)=\frac{1}{e^{-2v}+1}]-\frac{1}{e^{-2v}+1}|.
$$
In addition, we have 
\begin{align*}
	\E[Y=1\mid \hat f(X)=\frac{1}{e^{-2v}+1}]
	=\frac{1}{e^{-\frac{2\hat\theta(t)^\top\theta}{\|\hat\theta(t)\|^2}\cdot v}+1},
	%  =&\frac{1}{e^{-\frac{2\hat\theta(0)^\top\bmu}{(1-t)\|\hat\theta(0)\|^2}\cdot v}+1}+O_p(\frac{\sqrt{p}}{n}).
\end{align*}
where $$t=\frac{2\alpha\beta(\alpha+\beta)}{(\alpha+\beta)^2(\alpha+\beta+1)}.$$
Let us denote $\rho = \frac{\hat\theta(t)^\top\theta}{\|\hat\theta(t)\|^2}$. Since $\frac{\hat\theta(t)^\top\theta}{\|\hat\theta(t)\|^2}=\frac{((1-t)\hat\theta(0)+t\epsilon)^\top\theta}{\|(1-t)\hat\theta(0)+t\epsilon)\|_2^2}=\frac{1}{1-t}\frac{\hat\theta(0)^\top\theta}{\|\hat\theta(0)\|^2}+O_P(\frac{\sqrt p}{n})$. As a result, 

\begin{align*}
	MCE=&\max_v| \E[Y=1\mid \hat f(X)=\frac{1}{e^{-2v}+1}]-\frac{1}{e^{-2v}+1}|\\
	=&\max_v|\frac{1}{e^{-\frac{2\hat\theta(t)^\top\theta}{\|\hat\theta(t)\|^2}\cdot v}+1}-\frac{1}{e^{-2v}+1}|\\
	=&\max_v|\frac{1}{e^{-\frac{2\hat\theta(0)^\top\theta}{(1-t)\|\hat\theta(0)\|^2}\cdot v}+1}-\frac{1}{e^{-2v}+1}|+O_P(\frac{\sqrt p}{n})\\
	=&\max_v|\frac{1}{e^{-\frac{2\hat\theta(0)^\top\theta}{\|\hat\theta(0)\|^2}\cdot v}+1}-\frac{1}{e^{-2(1-t)v}+1}|+O_P(\frac{\sqrt p}{n}).
\end{align*}

Now let us consider the quantity $\frac{\hat\theta(0)^\top\theta}{\|\hat\theta(0)\|^2}$.

Since we have $\hat\theta(0)=\theta+\epsilon_n$ with $\epsilon_n=\frac{1}{n}\sum_{i=1}^n x_iy_i-\theta$, this implies \begin{align*}
	\frac{\hat\theta(0)^\top\theta}{\|\hat\theta(0)\|^2}=\frac{\|\theta\|^2+\epsilon_n^\top\theta}{\|\theta\|^2+\|\epsilon_n\|^2+2\epsilon_n^\top\theta}\sim \frac{\|\theta\|^2+\frac{1}{\sqrt n}\|\theta\|}{\|\theta\|^2+\frac{p}{n}+O_P(\frac{\sqrt p}{n})+\frac{1}{\sqrt n}\|\theta\|},
\end{align*}
where the last equality uses the fact that $\|\epsilon_n\|^2\stackrel{d}{=}\frac{\chi_p^2}{n}=\frac{p+O_P(\sqrt{p})}{n}=\frac{p}{n}+O_P(\frac{\sqrt p}{n})$.

By our assumption, we have $p/n\in(c_1,c_2)$ and $\|\theta\|<C$, implying that there exists a constant $c_0\in(0,1)$, such that with high probability $$
\frac{\hat\theta(0)^\top\theta}{\|\hat\theta(0)\|^2}\le c_0.
$$
Then on the event $\mathcal E=\{\frac{\hat\theta(0)^\top\theta}{\|\hat\theta(0)\|^2}\le c_0\}$, if we choose $t=1-c_0$, we will then have for any $v\in\R$, $$
|\frac{1}{e^{-\frac{2\hat\theta(0)^\top\theta}{\|\hat\theta(0)\|^2}\cdot v}+1}-\frac{1}{e^{-2(1-t)v}+1}|<|\frac{1}{e^{-\frac{2\hat\theta(0)^\top\theta}{\|\hat\theta(0)\|^2}\cdot v}+1}-\frac{1}{e^{-2v}+1}|,
$$
and moreover, the difference between LHS and RHS is lower bounded by $|\frac{1}{e^{-2v}+1}-\frac{1}{e^{-2c_0v}+1}|$. Thus, we have 
$$ MCE(\hat \cC^{mix})<MCE(\hat \cC).
$$

\end{proof}

\subsection{Proof of Theorem~ \ref{thm:lowdimMCE}}

\begin{theorem}[Restatement of Theorem \ref{thm:lowdimMCE}]
	There exists a threshold $\tau=o(1)$ such that if $p/n\le\tau$ and $\|\theta\|_2<C$ for some universal constant $C>0$, given any constants $\alpha,\beta>0$ (not depending on $n$ and $p$), when $n$ is sufficiently large, we have, with high probability, $$
	MCE(\hat{\mathcal{C}})<MCE(\hat{\mathcal{C}}^{mix}).
	$$
\end{theorem}

\begin{proof}
	According to the proof in Theorem \ref{thm:helpcalibrationMCE}, we have $$
	MCE(\hat \cC)=\max_v|\frac{1}{e^{-\frac{2\hat\theta(0)^\top\theta}{\|\hat\theta(0)\|^2}\cdot v}+1}-\frac{1}{e^{-2v}+1}|,
	$$
	and 
	$$
	MCE(\hat \cC^{mix})=\max_v|\frac{1}{e^{-\frac{2\hat\theta(0)^\top\theta}{\|\hat\theta(0)\|^2}\cdot v}+1}-\frac{1}{e^{-2(1-t)v}+1}|+O_P(\frac{\sqrt p}{n}),
	$$
	where $t\in(0,1)$ is a fixed constant when $\alpha,\beta>0$ are some fixed constants.
	
	When $p/n\to0$, then \begin{align*}
		\frac{\hat\theta(0)^\top\theta}{\|\hat\theta(0)\|^2}\sim \frac{\|\theta\|^2+\frac{1}{\sqrt n}\|\theta\|}{\|\theta\|^2+\frac{p}{n}+O_P(\frac{\sqrt p}{n})+\frac{1}{\sqrt n}\|\theta\|}= 1+O_P(\frac{p}{n})=1+o_P(1).
	\end{align*}
	Therefore, we have \begin{align*}
		MCE(\hat \cC)=&\max_v[|\frac{1}{e^{-\frac{2\hat\theta(0)^\top\theta}{\|\hat\theta(0)\|^2}\cdot v}+1}-\frac{1}{e^{-2v}+1}|]\\
		=&\max_v[|\frac{1}{e^{-{2v}+1}}-\frac{1}{e^{-2v}+1}|]+O_P(\frac{p}{n})\\=&O_P(\frac{p}{n})=o_P(1)
	\end{align*}
	and
	\begin{align*}
	MCE(\hat \cC^{mix})=&\max_v|\frac{1}{e^{-\frac{2\hat\theta(0)^\top\theta}{\|\hat\theta(0)\|^2}\cdot v}+1}-\frac{1}{e^{-2(1-t)v}+1}|+O_P(\frac{\sqrt p}{n})\\
		=&|\frac{1}{e^{-{2v}+1}}-\frac{1}{e^{-2(1-t)v}+1}|+O_P(\frac{p}{n})
	\end{align*}
	
	Since $\max_v|\frac{1}{e^{-{2v}+1}}-\frac{1}{e^{-2(1-t)v}+1}|=\Omega(1)$ when $t\in(0,1)$ is a fixed constant, we then have the desired result that $$
	MCE(\hat \cC)<MCE(\hat \cC^{mix}).
	$$

\end{proof}

\subsection{Proof of Theorem~ \ref{thm:capacityMCE}}

\begin{theorem}[Restatement of Theorem \ref{thm:capacityMCE}]
	For any constant $c_{\max}>0$, $p/n\to c_{ratio}\in(0,c^{\max})$, when  $\theta$ is sufficiently large (still of a constant level), we have for any $\beta>0$, with high probability, the change of ECE by using Mixup, characterized by
	$$
	\frac{d}{d\alpha}MCE(\hat{\mathcal{C}}^{mix}_{\alpha,\beta})\mid_{\alpha\to0+}
	$$ is negative, and monotonically decreasing with respect to $c_{ratio}$. 
	%\zhun{fill in later, describe monotonicity} {\red hard to describe if we didn't define use $\hat\theta_{mix}$ instead of $\hat\theta(\lambda)$}
\end{theorem}

\begin{proof}
Recall that 
$$
MCE(\hat \cC^{mix})=\max_v|\frac{1}{e^{-\frac{2\hat\theta(0)^\top\theta}{\|\hat\theta(0)\|^2}\cdot v}+1}-\frac{1}{e^{-2(1-t)v}+1}|+O_P(\frac{\sqrt p}{n}).
$$
The case $\alpha=0$ corresponds to the case where $t=0$. Since $|\frac{1}{e^{-\frac{2\hat\theta(0)^\top\theta}{\|\hat\theta(0)\|^2}\cdot v}+1}-\frac{1}{e^{-2(1-t)v}+1}|$ as a function of $v$ is symmetric around $0$, we have that when $t$ is sufficiently small (such that $1-t>\frac{\hat\theta(0)^\top\theta}{\|\hat\theta(0)\|^2}$ with high probability)
\begin{align*}
\max_v|\frac{1}{e^{-\frac{2\hat\theta(0)^\top\theta}{\|\hat\theta(0)\|^2}\cdot v}+1}-\frac{1}{e^{-2(1-t)v}+1}|=&\max_{v>0}|\frac{1}{e^{-\frac{2\hat\theta(0)^\top\theta}{\|\hat\theta(0)\|^2}\cdot v}+1}-\frac{1}{e^{-2(1-t)v}+1}|\\
	=&\max_{v>0}[\frac{1}{e^{-\frac{2\hat\theta(0)^\top\theta}{\|\hat\theta(0)\|^2}\cdot v}+1}-\frac{1}{e^{-2(1-t)v}+1}].
\end{align*}

Let us denote 
$$v^*=\argmax_{v>0}[\frac{1}{e^{-\frac{2\hat\theta(0)^\top\theta}{\|\hat\theta(0)\|^2}\cdot v}+1}-\frac{1}{e^{-2(1-t)v}+1}].$$

For the term 
$$\xi(t)=\frac{1}{e^{-\frac{2\hat\theta(0)^\top\theta}{\|\hat\theta(0)\|^2}\cdot v^*}+1}-\frac{1}{e^{-2(1-t)v^*}+1},$$
let us take the derivative with respect to $t$, we get 

$$\frac{d}{dt}\xi(t)=-\frac{e^{-2(1-t)v^*}\cdot 2v^*}{(e^{-2(1-t)v^*}+1)^2}.$$

Therefore, the derivative evaluated at $t=0$ equals to 

$$	-\frac{e^{-2v^*}\cdot 2v^*}{(e^{-2v*}+1)^2},$$

for $v^*>0$, which is negative.

Recall that we have $\hat\theta(0)=\theta+\epsilon_n$ with $\epsilon_n=\frac{1}{n}\sum_{i=1}^n x_iy_i-\theta$, this implies \begin{align*}
	\frac{\hat\theta(0)^\top\theta}{\|\hat\theta(0)\|^2}=\frac{\|\theta\|^2+\epsilon_n^\top\theta}{\|\theta\|^2+\|\epsilon_n\|^2+2\epsilon_n^\top\theta}\sim \frac{\|\theta\|^2+\frac{1}{\sqrt n}\|\theta\|}{\|\theta\|^2+\frac{p}{n}+O_P(\frac{\sqrt p}{n})+\frac{1}{\sqrt n}\|\theta\|},
\end{align*}
where the last equality uses the fact that $\|\epsilon_n\|^2\stackrel{d}{=}\frac{\chi_p^2}{n}=\frac{p+O_P(\sqrt{p})}{n}=\frac{p}{n}+O_P(\frac{\sqrt p}{n})$.

Consider the case when $c_{ratio}=c_1$ and $c_{ratio}=c_2$, where $0<c_1<c_2$, we want to prove that,  
$$	-\frac{e^{-2v^*(c_1)}\cdot 2v^*(c_1)}{(e^{-2v^*(c_1)}+1)^2}>	-\frac{e^{-2v^*(c_2)}\cdot 2v^*(c_2)}{(e^{-2v^*(c_2)}+1)^2},$$
where $v^*(c_1)>0$ and $v^*(c_2)>0$ are the maximizers of $\frac{1}{e^{-\frac{2\hat\theta(0)^\top\theta}{\|\hat\theta(0)\|^2}\cdot v}+1}-\frac{1}{e^{-2(1-t)v}+1}$ when $c_{ratio}=c_1$ and $c_{ratio}=c_2$ respectively.

Since 
$$-\frac{e^{-2v^*}\cdot 2v^*}{(e^{-2v*}+1)^2}$$
is a decreasing function of $v^*$, and with high probability 
$$ \frac{\hat\theta(0)^\top\theta}{\|\hat\theta(0)\|^2}\Big|_{c_{ratio}=c_1}>\frac{\hat\theta(0)^\top\theta}{\|\hat\theta(0)\|^2}\Big|_{c_{ratio}=c_2}.$$

Thus, we only need to show that the maximizer $v^*(\rho)$ defined by
$$v^*(\rho)=\argmax_{v>0}[\frac{1}{e^{-2\rho\cdot v}+1}-\frac{1}{e^{-2(1-t)v}+1}] \Big |_{t\rightarrow 0^+}$$

is an decreasing function of $\rho$ for $\rho\in[0,1)$ (since $\frac{\hat\theta(0)^\top\theta}{\|\hat\theta(0)\|^2}\in[0,1)$).
	
As we know that $v^*(\rho)$ is the solution of the following equation:
$$\frac{2\rho v^*(\rho)}{e^{2\rho v^*(\rho)}+e^{-2\rho v^*(\rho)}+2}=\frac{2 v^*(\rho)}{e^{2 v^*(\rho)}+e^{-2 v^*(\rho)}+2}.$$

From Figure \ref{fig:illu}, we can directly see that $v^*(\rho)$ increases as $\rho$ decreases. We complete the proof.

\begin{figure}[t]
	\begin{center}
	\includegraphics[width=8cm]{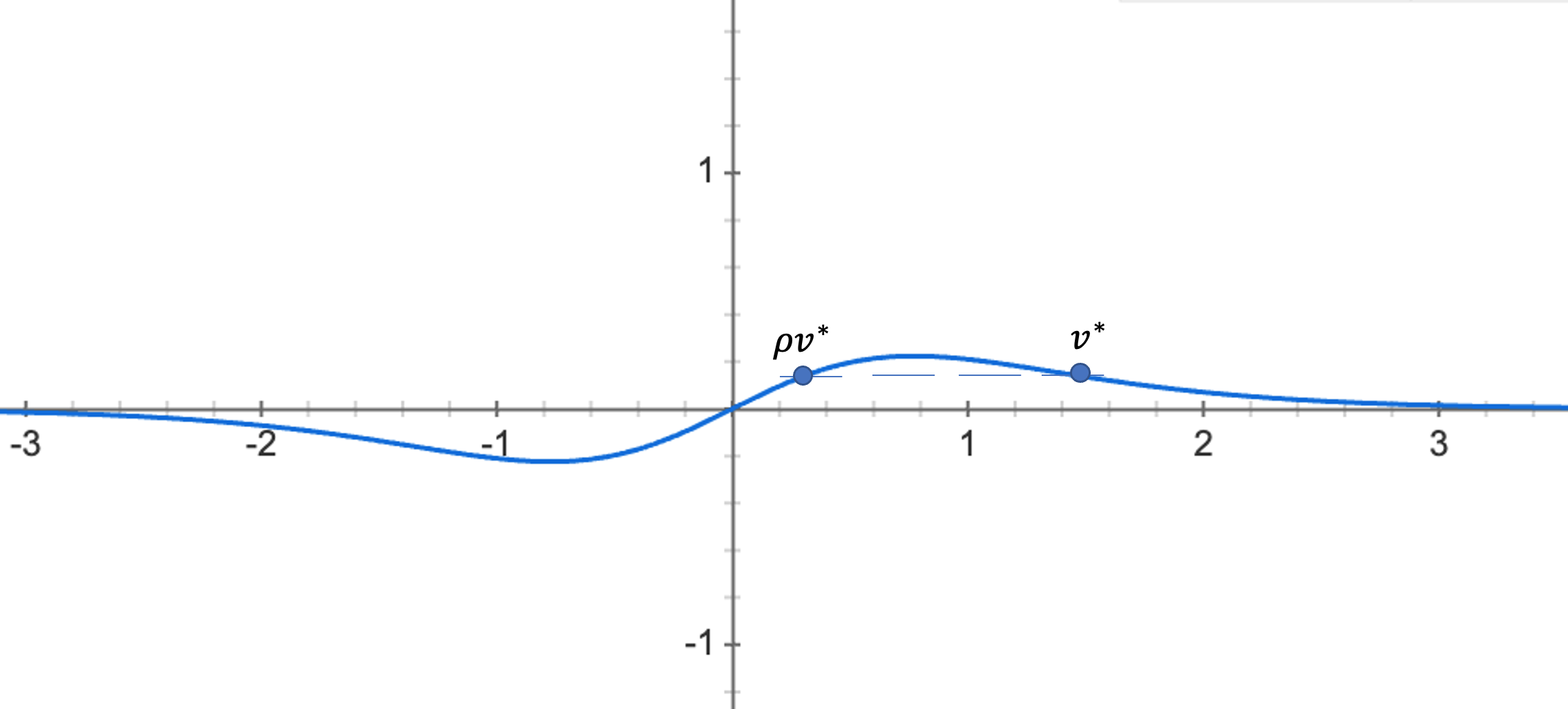}
	\caption{Illustration plot for the function $2x/(e^{2}+e^{-2x}+2)$.}
	\label{fig:illu}
	\end{center}
\end{figure}

\end{proof}

\section{Experimental setup and additional numerical results} \label{app:exp:addition}
To complement our theory, we further provide more experimental evidence on popular image classification data sets with neural networks. In Figures \ref{fig:1} and \ref{fig:4}, we used fully-connected neural networks and ResNets with various values of the width (\emph{i.e.,} the number of neurons per hidden layer) and the depth (\emph{i.e.}, the number of hidden layers). For the experiments on the effect of the width, we fixed the depth  to be 8 and varied the width from 10 to 3000. For the experiments on the effect of the depth, the depth was varied from 1 to 24 (\emph{i.e.}, from 3 to 26 layers including input/output layers) by fixing the width to be 400 with data-augmentation and 80 without data-augmentation. We used the following standard data-augmentation operations using \texttt{torchvision.transforms} for both data sets: random crop (via \texttt{RandomCrop(32, padding=4})) and random horizontal flip (via \texttt{RandomHorizontalFlip}) for each image. We used the standard data sets --- CIFAR-10 and CIFAR-100  \citep{krizhevsky2009learning}. We employed SGD with mini-batch size of 64.  We set the learning rate to be 0.01 and momentum coefficient to be 0.9.  We used the Beta distribution $Beta(\alpha,\alpha)$ with $\alpha=1.0$ for Mixup.  

In Figure \ref{fig:7}, we adopted the standard data sets, Kuzushiji-MNIST  \citep{clanuwat2019deep},  Fashion-MNIST \citep{xiao2017fashion}, and CIFAR-10 and CIFAR-100  \citep{krizhevsky2009learning}. We used SGD with mini-batch size of 64 and the learning rate of 0.01. The Beta distribution $Beta(\alpha,\alpha)$ with $\alpha=1.0$ was used for Mixup. We used the standard pre-activation ResNet with $18$ layers and ReLU activations \citep{he2016identity}. For each data set, we randomly divided  each training data (100\%)\ into a labeled training data (50\%)\ and a unlabeled training data (50\%) with the 50-50 split. Following the theoretical analysis, we first trained the ResNet with labeled data until the half of the last epoch in each figure. Then, the pseudo-labels were generated by the ResNet and used for the final half of the training. 

We run experiments with a machine with 10-Core 3.30 GHz Intel Core i9-9820X and four NVIDIA RTX 2080 Ti GPUs with 11 GB GPU memory.

Figure \ref{fig:new:new:1} shows that Mixup also tend to reduce test loss for larger capacity models. The experimental setting of Figure \ref{fig:new:new:1} is the exactly same as that of Figures \ref{fig:1} and \ref{fig:4}. Here, relative test loss of a particular case is defined by $\frac{\text{test loss of a particular case}}{\text{test loss of no mixup base case}}$.

Figures \ref{fig:new:1}--\ref{fig:new:2} show that Mixup can  reduce ECE$_2$ particularly for larger capacity models. 

Figure \ref{fig:mce:1} uses the same setting as that of Figures \ref{fig:1} and \ref{fig:4}. Similarly to Figure \ref{fig:mce:1}, Figure \ref{fig:mce:2} below shows that Mixup can reduce MCE particularly for larger capacity models with varying degrees of depth and width. The setting of Figure \ref{fig:mce:2} is the same as that of Figures \ref{fig:1} and \ref{fig:4} with the data-augmentation.

\begin{figure}[h!]
\centering
\begin{subfigure}{0.2\columnwidth}
  \includegraphics[width=\textwidth, height=0.7\textwidth]{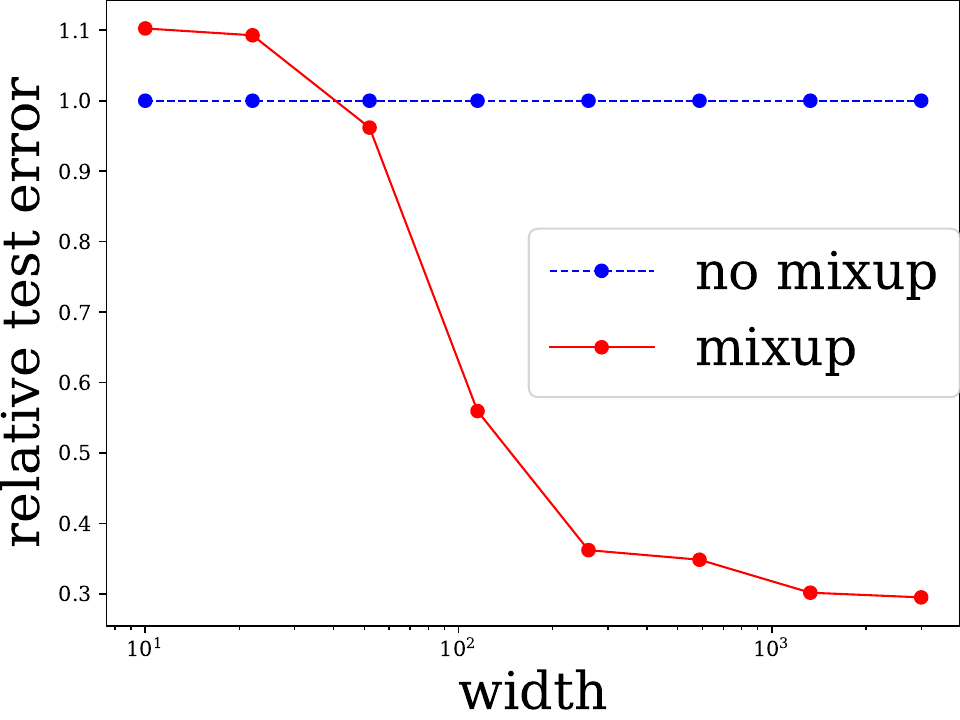}
 \caption{Width}
\end{subfigure}
\hspace{0.1in}
\begin{subfigure}{0.2\columnwidth}
  \includegraphics[width=\textwidth, height=0.7\textwidth]{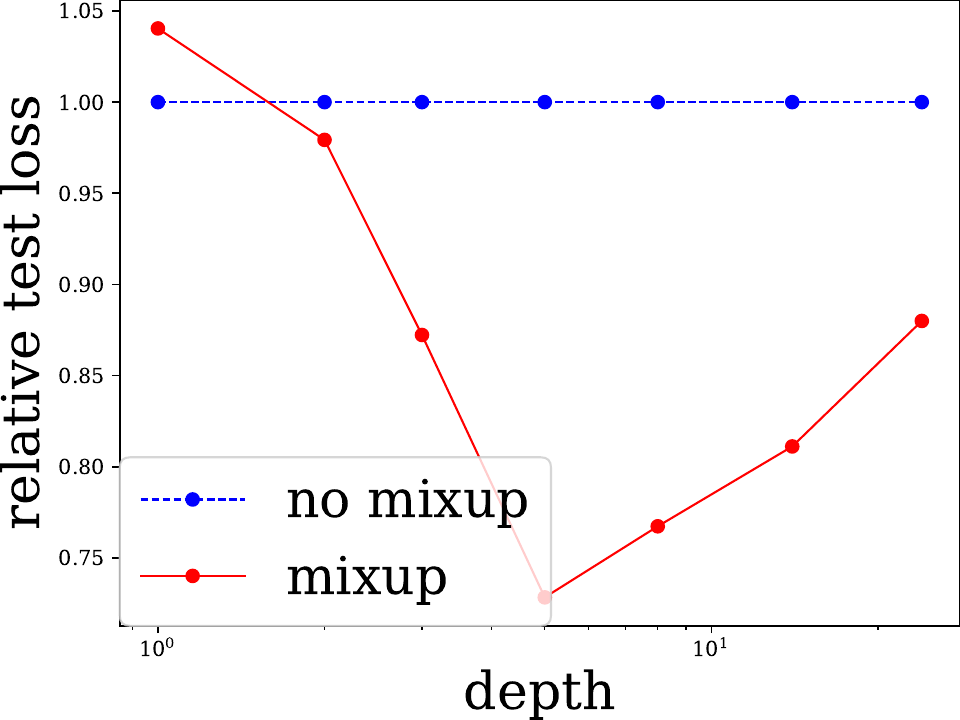}
 \caption{Depth}
\end{subfigure} 
\hspace{0.1in}
\begin{subfigure}{0.2\columnwidth}
  \includegraphics[width=\textwidth, height=0.7\textwidth]{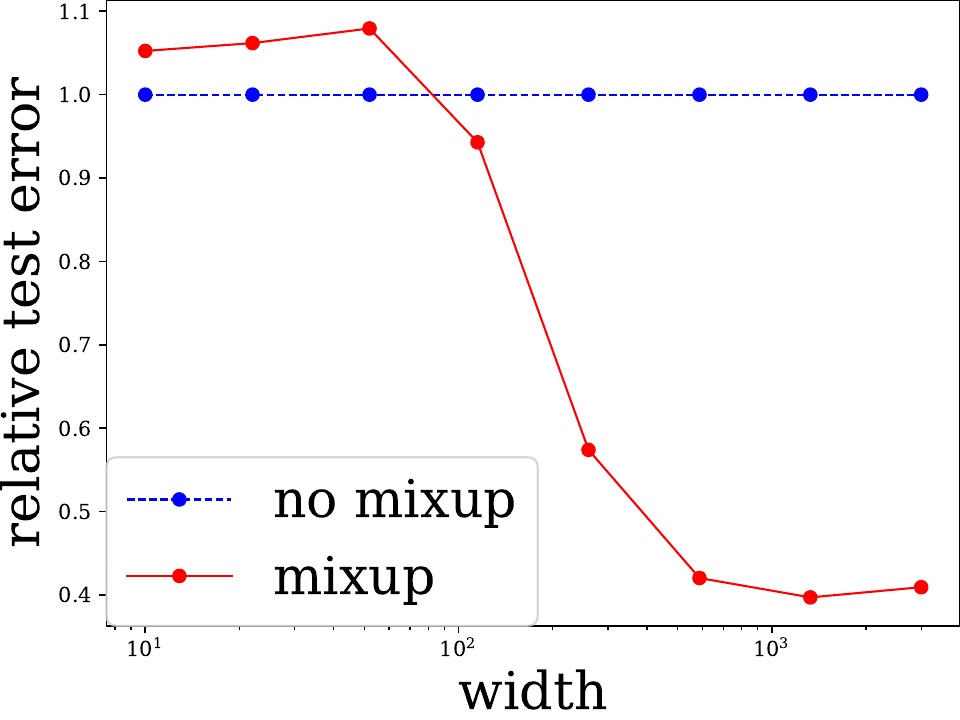}
 \caption{Width}
\end{subfigure}
\hspace{0.1in}
\begin{subfigure}{0.2\columnwidth}
  \includegraphics[width=\textwidth, height=0.7\textwidth]{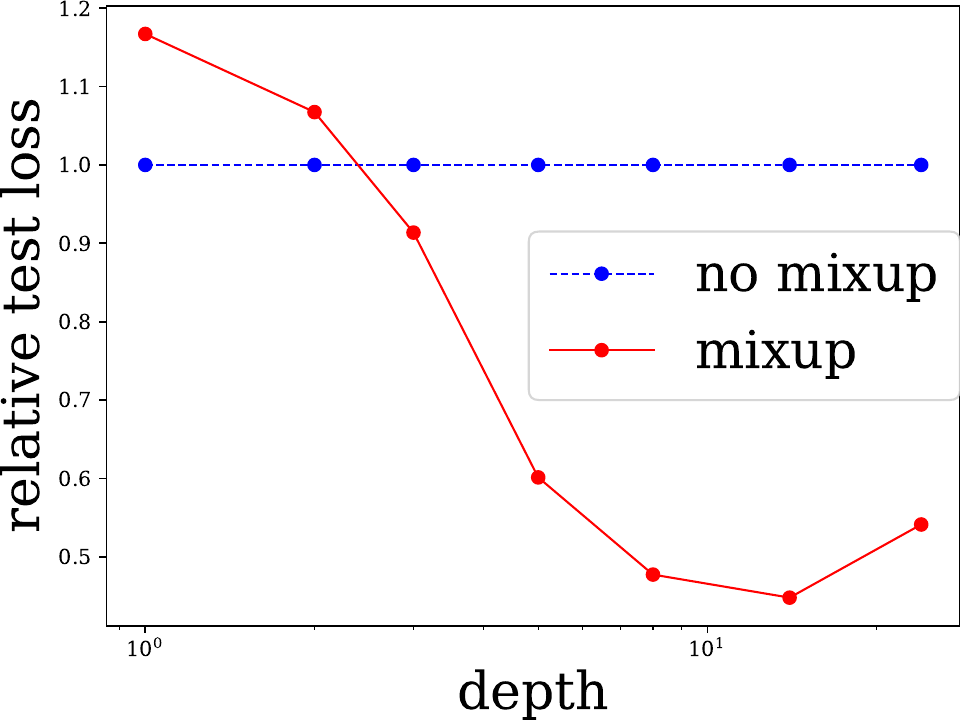}
 \caption{Depth}
\end{subfigure}
\hspace{0.1in}
\begin{subfigure}{0.2\columnwidth}
  \includegraphics[width=\textwidth, height=0.7\textwidth]{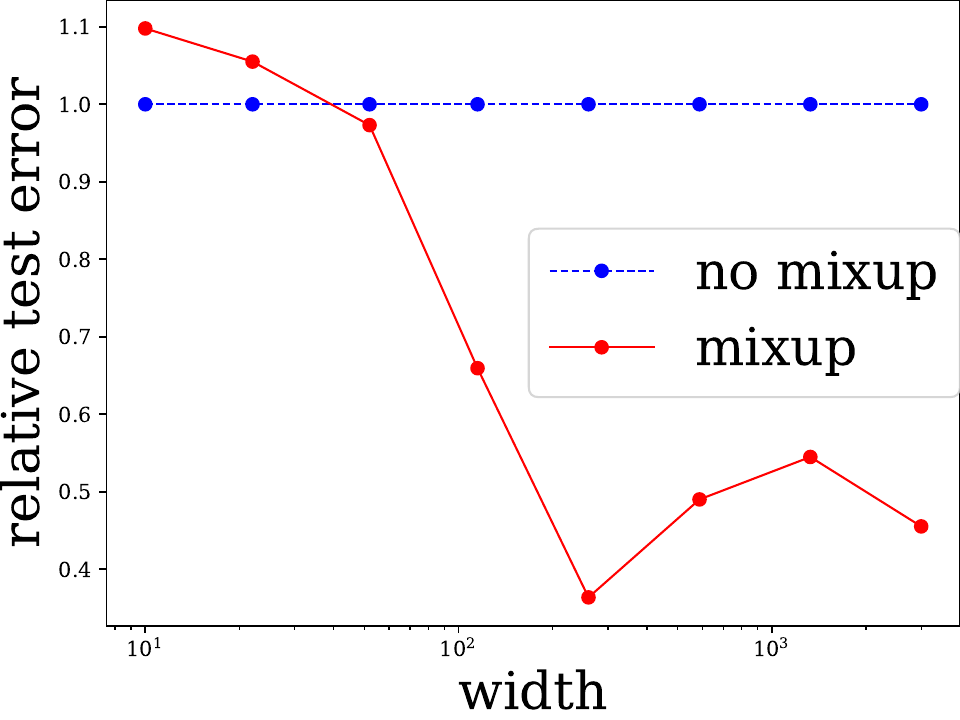}
 \caption{Width}
\end{subfigure}
\hspace{0.1in}
\begin{subfigure}{0.2\columnwidth}
  \includegraphics[width=\textwidth, height=0.7\textwidth]{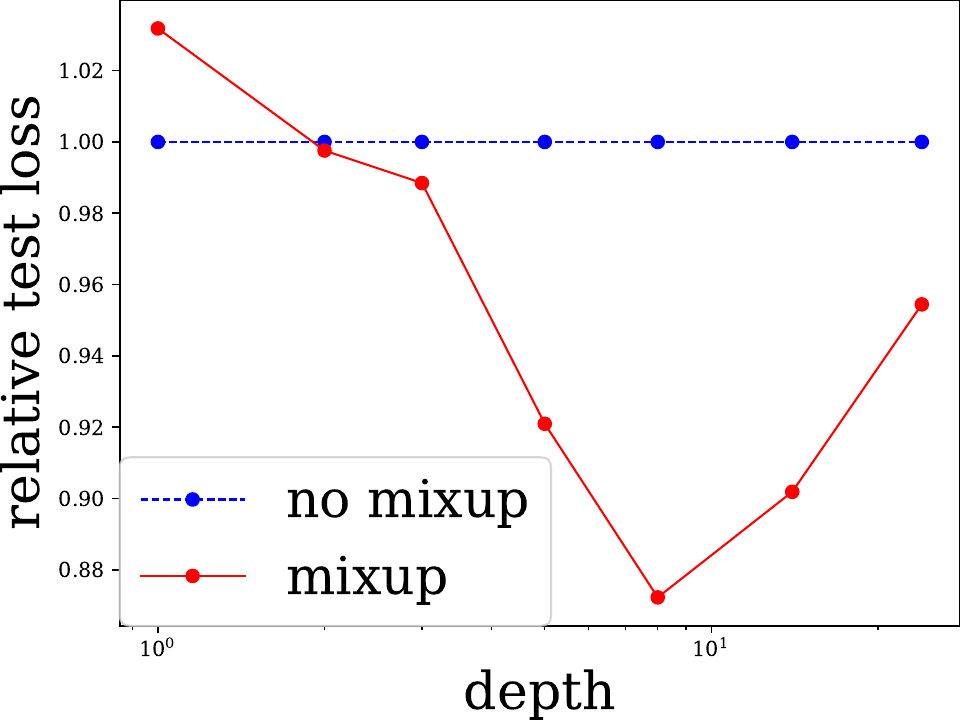}
 \caption{Depth}
\end{subfigure} 
\hspace{0.1in}
\begin{subfigure}{0.2\columnwidth}
  \includegraphics[width=\textwidth, height=0.7\textwidth]{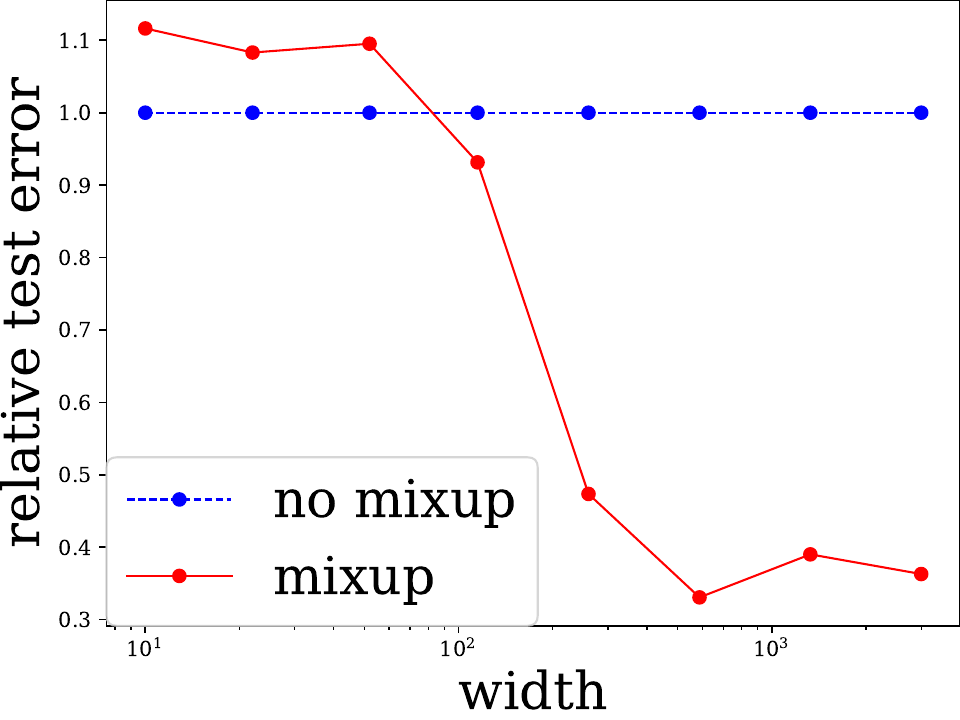}
 \caption{Width}
\end{subfigure}
\hspace{0.1in}
\begin{subfigure}{0.2\columnwidth}
  \includegraphics[width=\textwidth, height=0.7\textwidth]{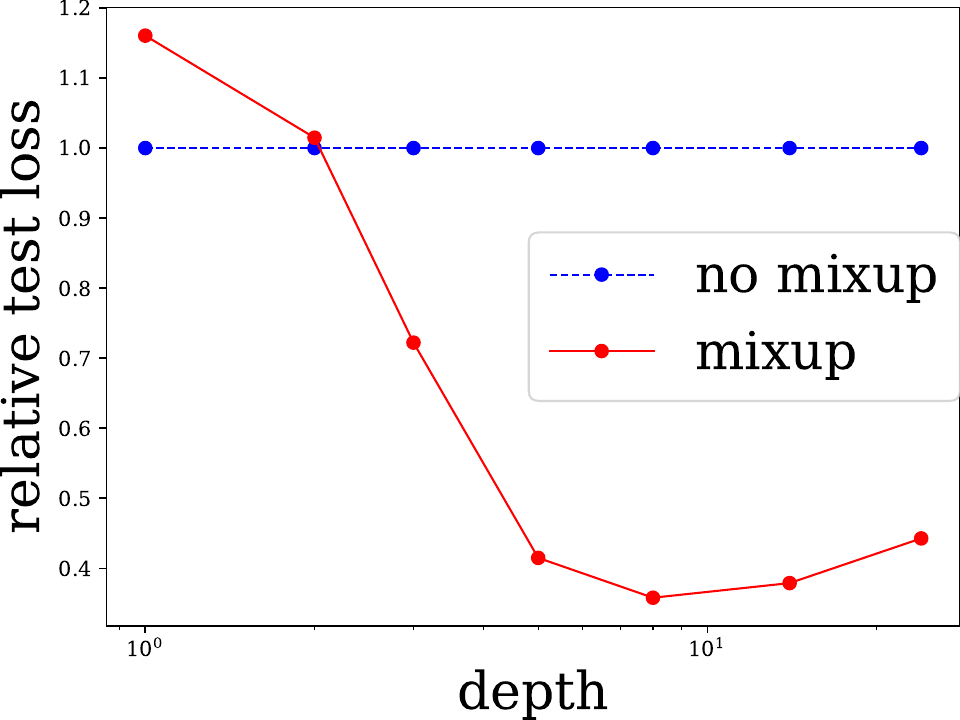}
 \caption{Depth}
\end{subfigure}
\caption{Relative test loss: (a), (b): CIFAR-10 without data augmentation; (c), (d): CIFAR-10 with data augmentation; (e), (f): CIFAR-100 without data augmentation; (g), (h): CIFAR-100 with data augmentation.} 
\label{fig:new:new:1} 
\end{figure}

\begin{figure}[h!]
\centering
\begin{subfigure}[b]{0.24\textwidth}
  \includegraphics[width=\textwidth, height=0.7\textwidth]{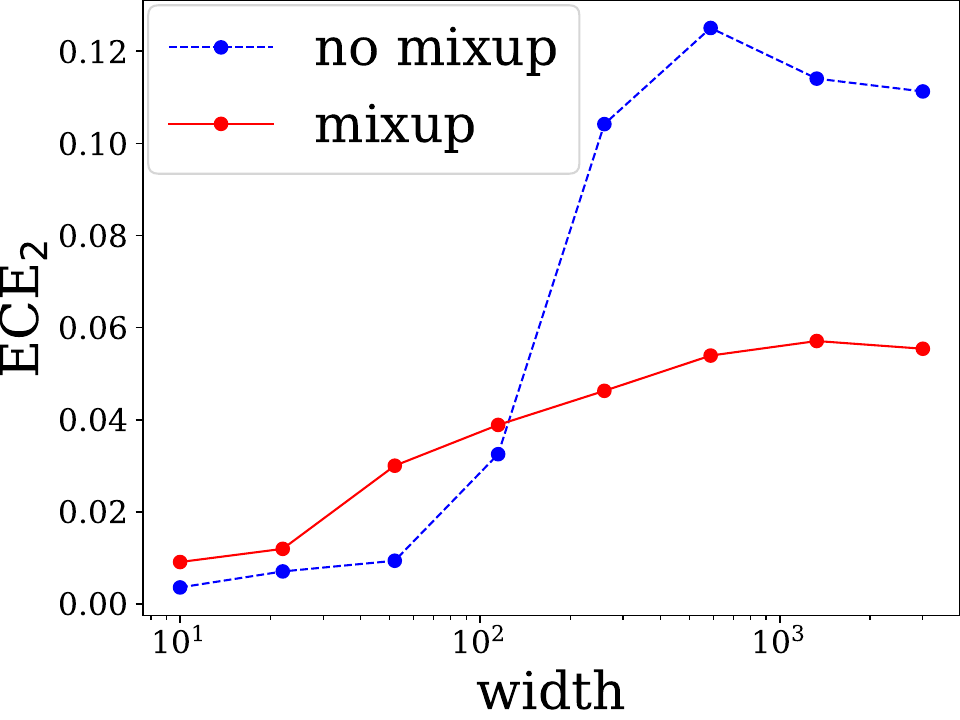}
 \caption{CIFAR-10: width}
\end{subfigure}
\begin{subfigure}[b]{0.24\textwidth}
  \includegraphics[width=\textwidth, height=0.7\textwidth]{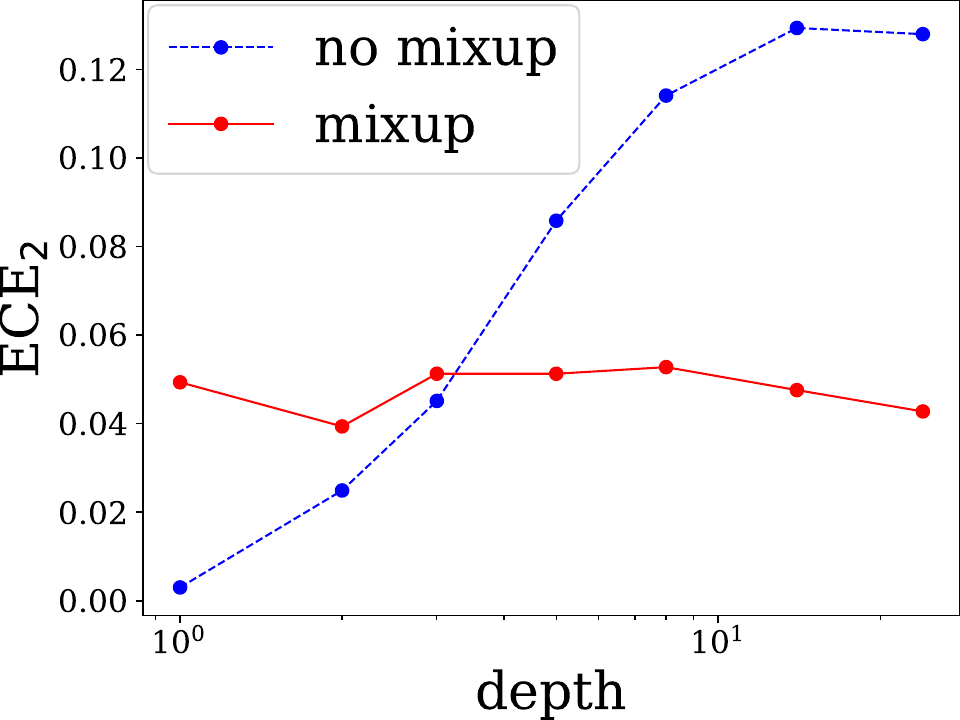}
 \caption{CIFAR-10: depth}
\end{subfigure}
\begin{subfigure}[b]{0.24\textwidth}
  \includegraphics[width=\textwidth, height=0.7\textwidth]{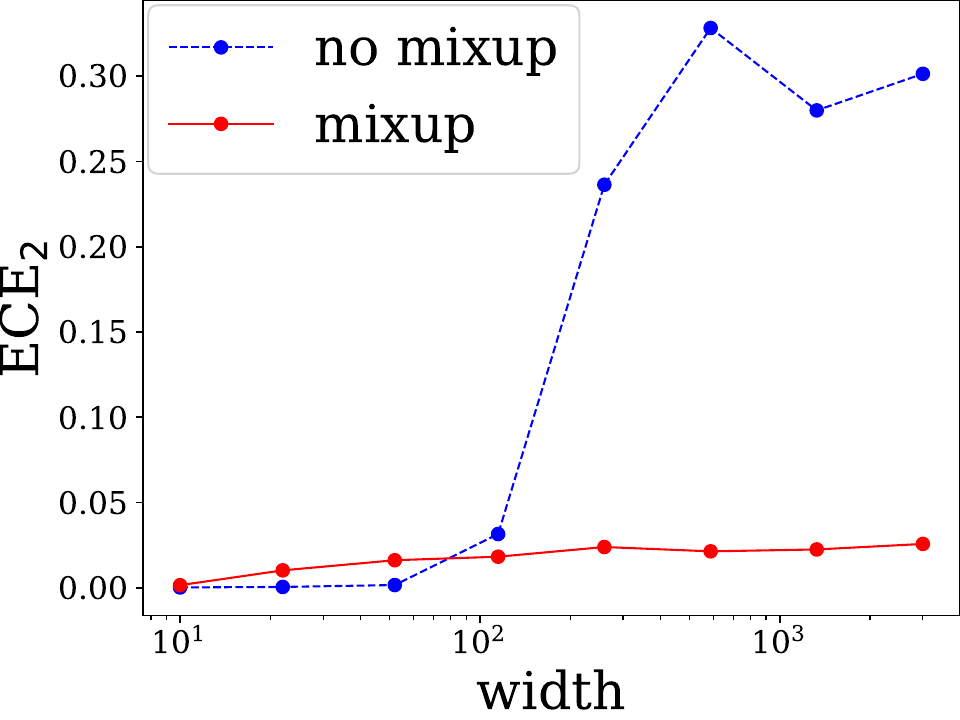}
 \caption{CIFAR-100: width}
\end{subfigure}
\begin{subfigure}[b]{0.24\textwidth}
  \includegraphics[width=\textwidth, height=0.7\textwidth]{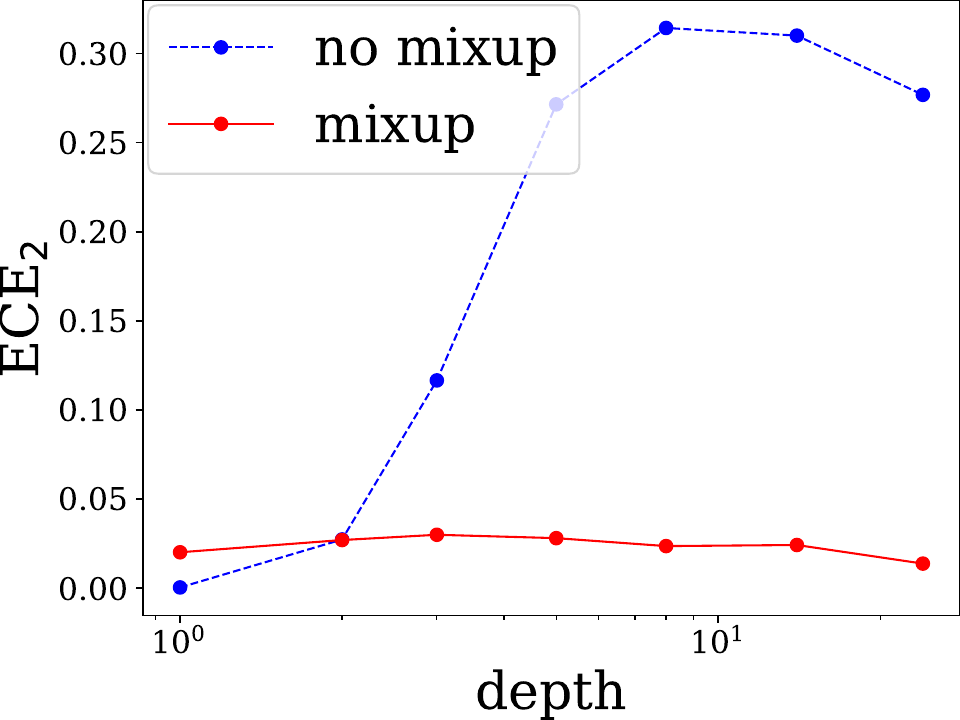}
 \caption{CIFAR-100: depth}
\end{subfigure}
\caption{ECE$_2$  with  data-augmentation} 
\label{fig:new:1} 
\end{figure}

\begin{figure}[H]
\centering
\begin{subfigure}[b]{0.24\textwidth}
  \includegraphics[width=\textwidth, height=0.7\textwidth]{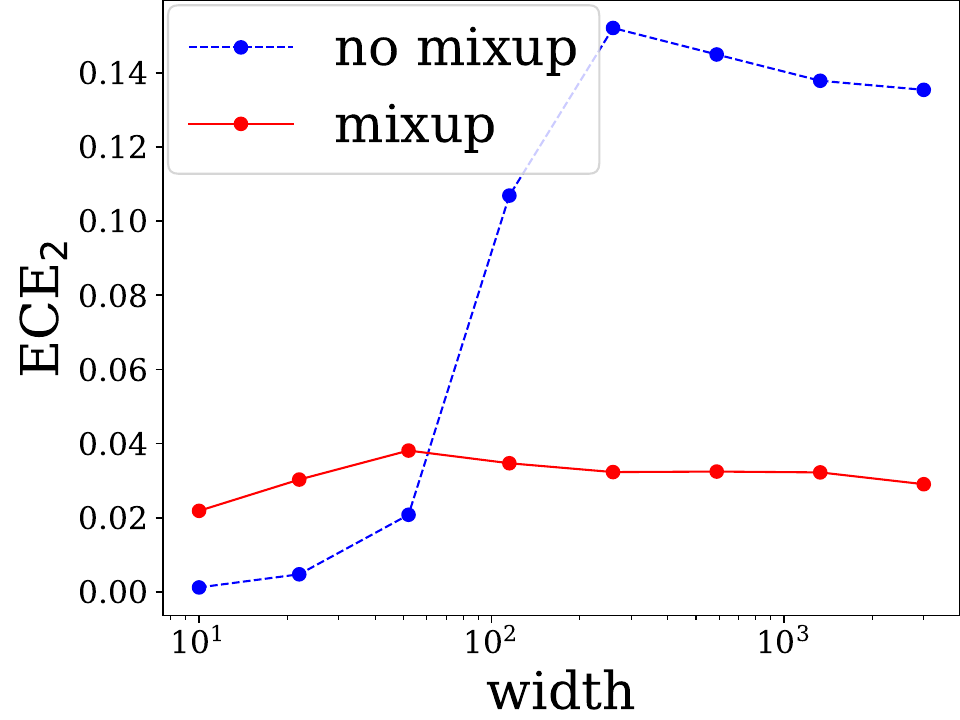}
 \caption{CIFAR-10: width}
\end{subfigure}
\begin{subfigure}[b]{0.24\textwidth}
  \includegraphics[width=\textwidth, height=0.7\textwidth]{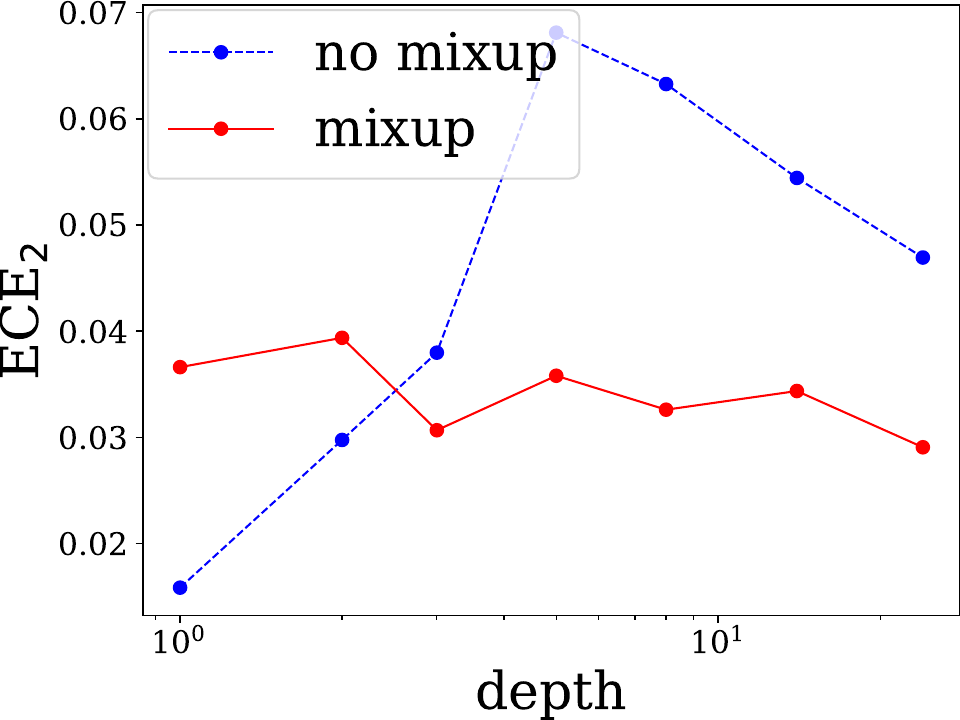}
 \caption{CIFAR-10: depth}
\end{subfigure}
\begin{subfigure}[b]{0.24\textwidth}
  \includegraphics[width=\textwidth, height=0.7\textwidth]{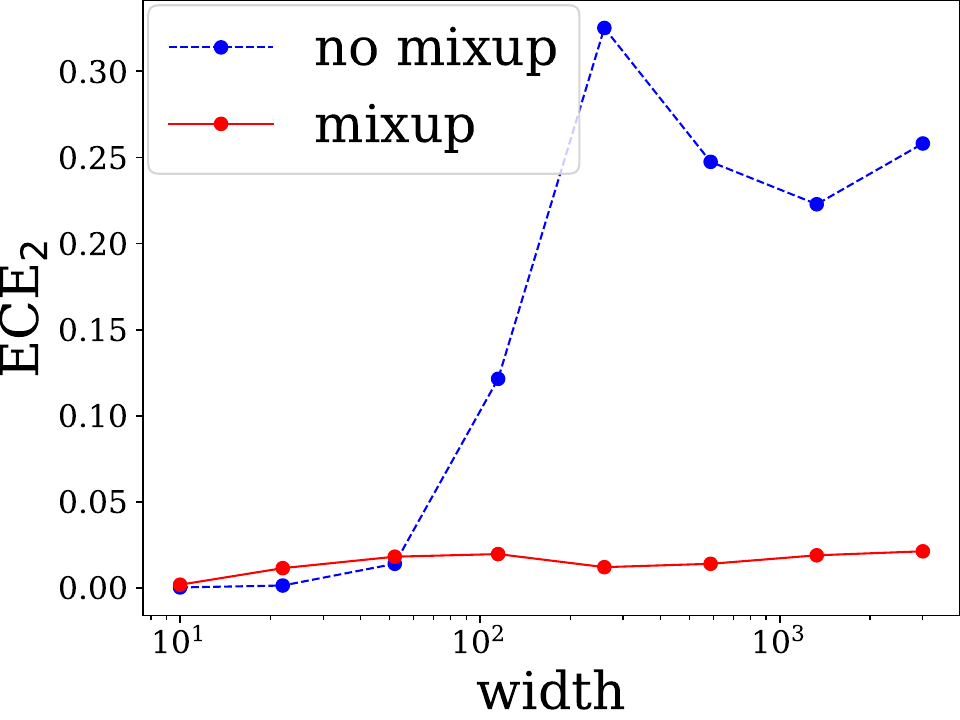}
 \caption{CIFAR-100: width}
\end{subfigure}
\begin{subfigure}[b]{0.24\textwidth}
  \includegraphics[width=\textwidth, height=0.7\textwidth]{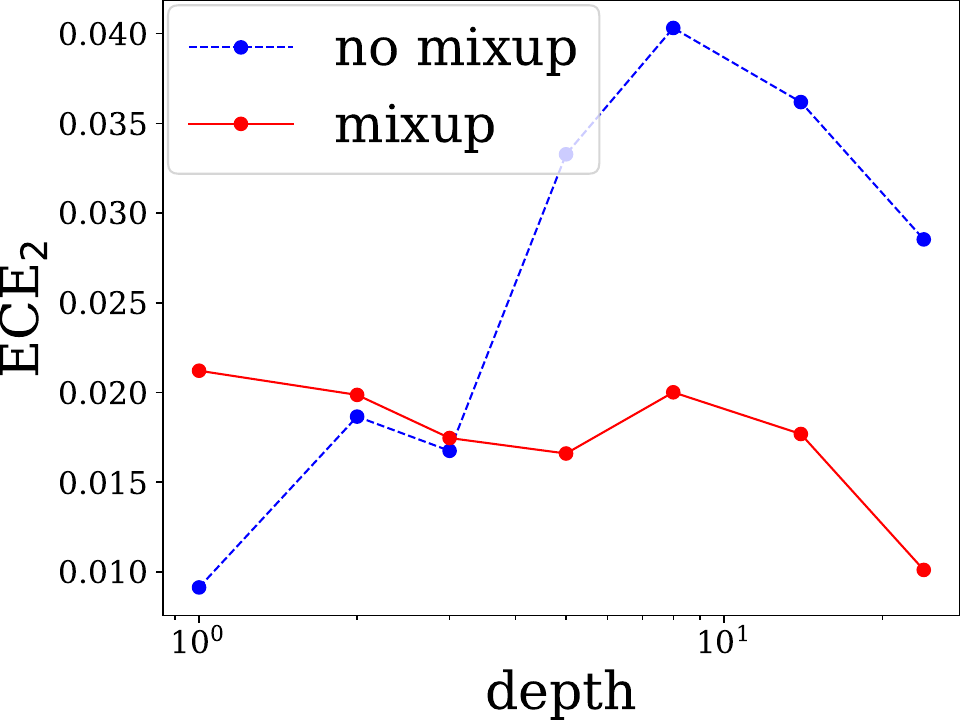}
 \caption{CIFAR-100: depth}
\end{subfigure}
\caption{ECE$_2$  without data-augmentation} 
\label{fig:new:2} 
\end{figure}

\begin{figure}[H]
\centering
\begin{subfigure}[b]{0.24\textwidth}
  \includegraphics[width=\textwidth, height=0.7\textwidth]{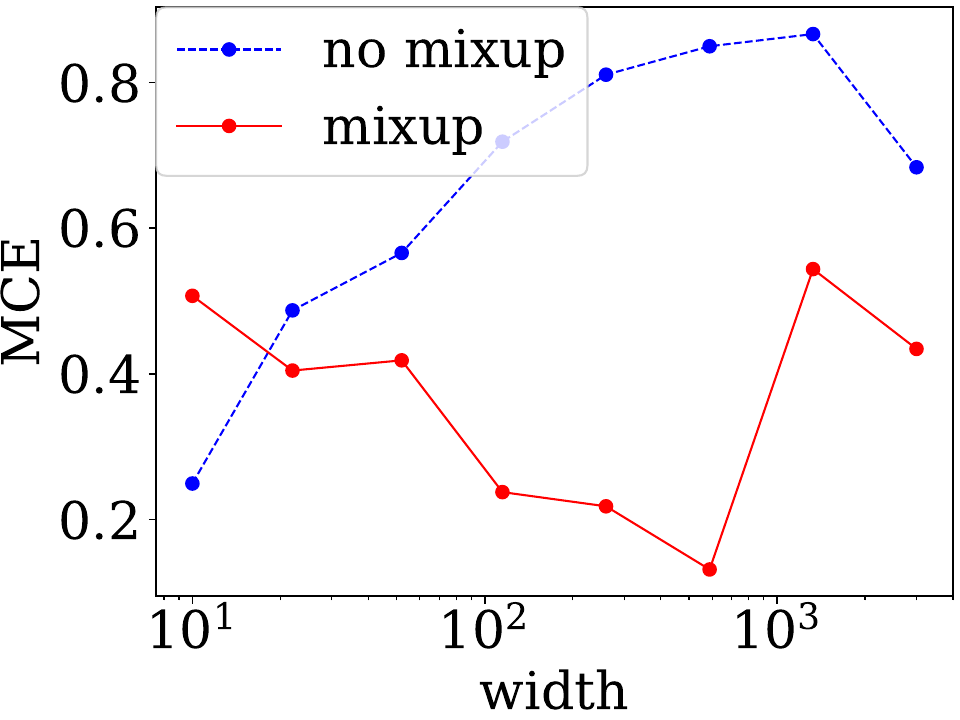}
 \caption{CIFAR-10: width}
\end{subfigure}
\begin{subfigure}[b]{0.24\textwidth}
  \includegraphics[width=\textwidth, height=0.7\textwidth]{fig/mce/cifar10aa0depth.pdf}
 \caption{CIFAR-10: depth}
\end{subfigure}
\begin{subfigure}[b]{0.24\textwidth}
  \includegraphics[width=\textwidth, height=0.7\textwidth]{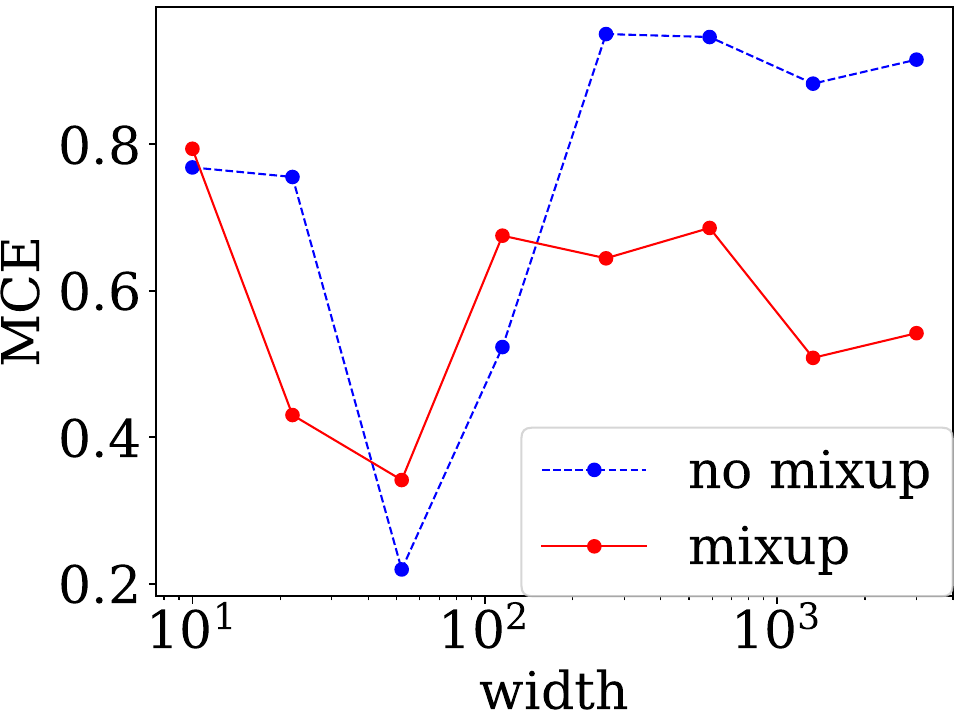}
 \caption{CIFAR-100: width}
\end{subfigure}
\begin{subfigure}[b]{0.24\textwidth}
  \includegraphics[width=\textwidth, height=0.7\textwidth]{fig/mce/cifar100aa1depth.pdf}
 \caption{CIFAR-100: depth}
\end{subfigure}
\caption{Maximum Calibration Error (MCE)} 
\label{fig:mce:2} 
\end{figure}

\end{document}